\documentclass[10pt]{article}

\usepackage[utf8]{inputenc} % allow utf-8 input
\usepackage[T1]{fontenc}    % use 8-bit T1 fonts

\usepackage{microtype}
\usepackage{graphicx}
\usepackage[dvipsnames]{xcolor}
	
\usepackage{color, colortbl}

\usepackage{subcaption}
\usepackage{booktabs} 
\usepackage{natbib}
\usepackage{hyperref}

\usepackage{amsmath, amssymb, amsfonts, amsthm}

\usepackage{mathtools}

\usepackage{minitoc}
\usepackage[capitalize,noabbrev]{cleveref}
\usepackage[paper=a4paper]{geometry}
\usepackage[outline]{contour}
\contourlength{0.1pt}
\theoremstyle{plain}
\newtheorem{theorem}{Theorem}[section]
\newtheorem{proposition}[theorem]{Proposition}
\newtheorem{lemma}[theorem]{Lemma}
\newtheorem{corollary}[theorem]{Corollary}

\newtheorem{claim}[theorem]{Claim}
\newtheorem{observation}[theorem]{Observation}

\theoremstyle{definition}
\newtheorem{definition}[theorem]{Definition}
\newtheorem{assumption}[theorem]{Assumption}

\theoremstyle{remark}
\newtheorem{remark}[theorem]{Remark}
\newtheorem{example}{Example}
\newtheorem{condition}{Condition}

\usepackage[textsize=tiny]{todonotes}

\usepackage[utf8]{inputenc} % allow utf-8 input
\usepackage[T1]{fontenc}    % use 8-bit T1 fonts
\usepackage{url}            % simple URL typesetting
\usepackage{booktabs}       % professional-quality tables
\usepackage{nicefrac}       % compact symbols for 1/2, etc.
\usepackage{microtype}      % micro typography

\usepackage{lipsum}
\usepackage{graphicx}
\graphicspath{ {./images/} }
\usepackage{mleftright}
\usepackage{array}
\usepackage{commath}
\usepackage[noadjust]{cite}
\usepackage{mathtools}
\usepackage{color}
\usepackage{bbm}
\usepackage[flushleft]{threeparttable}
\usepackage{multirow}
\usepackage{algorithm}
\usepackage{algorithmic}
\usepackage{soul}
\usepackage{pifont}
\setstcolor{red}
\usepackage{multicol}
\usepackage{xspace}
\usepackage{wrapfig}
\usepackage{enumitem}

\newcommand{\cmark}{\ding{51}}%

\newcommand{\mypara}[1]{\vspace{3pt} \noindent \textbf{#1} \hspace{0.02in}}

\newcolumntype{Z}{>{\centering\arraybackslash}p{1.42cm}}
\newcolumntype{Y}{>{\raggedright\arraybackslash}p{1.3cm}}

\newcommand{\diag}{\mathop{\mbox{\rm diag}}}

\newcommand{\Av}{\mathbf{A}}
\newcommand{\Bv}{\mathbf{B}}

\newcommand{\Dv}{\mathbf{D}}
\newcommand{\Ev}{\mathbf{E}}

\newcommand{\Hv}{\mathbf{H}}
\newcommand{\Iv}{\mathbf{I}}

\newcommand{\Mv}{\mathbf{M}}

\newcommand{\Pv}{\mathbf{P}}
\newcommand{\Qv}{\mathbf{Q}}
\newcommand{\Rv}{\mathbf{R}}

\newcommand{\Uv}{\mathbf{U}}
\newcommand{\Vv}{\mathbf{V}}
\newcommand{\Wv}{\mathbf{W}}
\newcommand{\Xv}{\mathbf{X}}

\newcommand{\Ac}{\mathcal{A}}

\newcommand{\Cc}{\mathcal{C}}
\newcommand{\Dc}{\mathcal{D}}
\newcommand{\Ec}{\mathcal{E}}
\newcommand{\Fc}{\mathcal{F}}

\newcommand{\Ic}{\mathcal{I}}

\newcommand{\Lc}{\mathcal{L}}

\newcommand{\Nc}{\mathcal{N}}

\newcommand{\Rc}{\mathcal{R}}

\newcommand{\Tc}{\mathcal{T}}

\newcommand{\Eb}{\mathbb{E}}

\newcommand{\Rb}{\mathbb{R}}

\newcommand{\Thetav}{\mathbf{\Theta}}
\newcommand{\Lambdav}{\mathbf{\Lambda}}

\newcommand{\Phiv}{\mathbf{\Phi}}

\def\res{{\text{res}}}

\def\dh{{\overline{d}}}
\def\dl{{\underline{d}}}

\DeclareMathOperator*{\argmin}{arg\,min}

\def \alg{\texttt{FLUTE}\xspace}
\def \algstar{\texttt{FLUTE}*\xspace}

\definecolor{clr}{rgb}{0.88,1,1}
\title{Federated Representation Learning in the Under-Parameterized Regime}

\author{
 Renpu Liu\thanks{School of EECS, The Pennsylvania State University, University Park, PA, USA. Correspondence to: Jing Yang <yangjing@psu.edu>.},~~~
 Cong Shen\thanks{Department of Electrical and Computer Engineering, University of Virginia, Charlottesville, VA, USA.},~~~
 Jing Yang\footnotemark[1] 
}
\date{}

\begin{document}

\maketitle

%--------------------------------------------

\begingroup
\renewcommand\thefootnote{}
\footnotetext{This work has been accepted to the 41st International Conference on Machine Learning (ICML 2024).}
\endgroup

%--------------------------------------------

\begin{abstract}
Federated representation learning (FRL) is a popular personalized federated learning (FL) framework where clients work together to train a common representation while retaining their personalized heads. Existing studies, however, largely focus on the over-parameterized regime. In this paper, we make the initial efforts to investigate FRL in the under-parameterized regime, where the FL model is insufficient to express the variations in all ground-truth models. 
We propose a novel FRL algorithm \alg, and theoretically characterize its sample complexity and convergence rate for linear models in the under-parameterized regime. To the best of our knowledge, this is the first FRL algorithm with provable performance guarantees in this regime. \alg features a data-independent random initialization and a carefully designed objective function that aids the distillation of subspace spanned by the global optimal representation from the misaligned local representations. On the technical side, we bridge low-rank matrix approximation techniques with the FL analysis, which may be of broad interest. We also extend \alg beyond linear representations. Experimental results demonstrate that \alg outperforms state-of-the-art FRL solutions in both synthetic and real-world tasks. 
\end{abstract}

\section{Introduction}

In the development of machine learning (ML), the role of representation learning has become increasingly essential. It transforms raw data into meaningful features, reveals hidden patterns and insights in data, and facilitates efficient learning of various ML tasks such as meta-learning \citep{tripuraneni2021provable}, multi-task learning \citep{wang2016distributed}, and few-shot learning \citep{du2020few}. 

Recently, representation learning has been introduced to the federated learning (FL) framework to cope with the heterogeneous local datasets at participating clients~\citep{liang2020think}. In the FL setting, it often assumes that all clients share a common representation, which works in conjunction with personalized local heads to realize personalized prediction while harnessing the collective training power~\citep{arivazhagan2019federated, collins2021exploiting, zhong2022feddar, shen2023share}.

Existing theoretical analysis of representation learning usually assumes the adopted model is over-parameterized to almost fit the ground-truth model~\citep{tripuraneni2021provable,wang2016distributed}. While this may be valid for expressive models like Deep Neural Networks~\citep{he2016deep,liu2017survey} or Large Language Models~\citep{openai2023gpt4, touvron2023llama}, it may be too restrictive for FL on resource-constrained devices, as adopting over-parameterized models in such a framework faces several significant challenges, as elaborated below.
\vspace{-0.05in}

\begin{itemize}[topsep=0pt, leftmargin=10pt] 
    \item \textbf{Computation limitation.} In FL, edge devices like smartphones and Internet of Things (IoT) devices often have limited memory and lack computational power, which are not capable of either storing or training over-parameterized models~\citep{8664630,he2020group, kairouz2021advances}\footnote{For example, two of the widely adopted neural network models suitable for IoT or embedded devices, MobileNetV3 \citep{howard2019searching} and EfficientNet-B0 \citep{tan2019efficientnet}, only have a few million parameters and, as an example, typically process at most a few GFLOPS in a Raspberry Pi 4 \citep{ju2023efficient}. }. 
\vspace{-0.1in}

    \item \textbf{Communication overhead.} In FL, the clients need to communicate updated model information with the server frequently. It thus becomes prohibitive to transmit a huge number of model updates for devices operating with limited communication energy and bandwidth.
\vspace{-0.1in}

    \item \textbf{Privacy concern.} Existing works show that excessively expressive models may ``memorize'' relevant information from local datasets, increasing the model's susceptibility to reconstruction attacks ~\citep{hitaj2017deep, 8835269, wang2018inferring, li2020federated} or membership inference~\citep{tan2022parameters}.

\end{itemize}

Motivated by those concerns, in this work, we focus on federated representation learning (FRL) in the \textit{under-parameterized} regime, where {the parameterized model class is not rich enough to realize the ground-truth models across all clients}. This is arguably a more realistic setting for edge devices supporting FL. Meanwhile, due to the inherent limitation of the expressiveness of the under-parameterized models, the algorithm design and theoretical guarantees in the over-parameterized regime do not naturally translate to this setting. 
We summarize our main contributions as follows.
\begin{itemize}[topsep=0pt, leftmargin=10pt] 
    \item \textbf{Algorithm design.} A major challenge for FRL in the under-parameterized regime is the fact that the locally optimal representation may not be globally optimal. As a result, simply averaging the local representations may not converge to the global optimal solution. To cope with this challenge, we propose \alg, a novel FRL framework tailored for the under-parameterized setting.  To the best of our knowledge, this is the first FRL framework that focuses on the under-parameterized regime. Our algorithm design features two primary innovations. 
    First, we develop a new regularization term that generalizes the existing formulations in a non-trivial way. In particular, this new regularization term is designed to provably enhance the performance of FRL in the under-parameterized setting. 
Second, our algorithm contains a new and critical step of server-side updating by simultaneously optimizing both the representation layer and all local head layers. This represents a significant departure from existing approaches in FRL, particularly in over-parameterized settings where local heads are optimized solely on the client side. By leveraging information across these local heads, our approach could learn the ground-truth model more effectively.

\vspace{-0.1in}
\item \textbf{Theoretical guarantees.} In terms of theoretical performance, we specialize \alg to the linear setting and analyze the sample complexity required for \alg to recover a near-optimal model, as well as characterizing its convergence rate. \alg achieves a sample complexity that scales in $\tilde{\mathcal{O}}\left(\frac{\max\{d, M\}}{M\epsilon^2}\right)$ for recovering an $\epsilon$-optimal model, where $d$ is the dimension of the input data and $M$ is the number of clients. This result indicates a linear sample complexity speedup in terms of $M$ in the high dimensional setting (i.e., $d\geq M$) compared with its single-agent counterpart~\citep{hsu2014random}. Besides, it {outperforms} the sample complexity in the noiseless over-parameterized FRL setting~\citep{collins2021exploiting} in terms of both $M$ and $d$. 
Moreover, we show that \alg converges to the optimal model exponentially fast when the number of samples is sufficiently large.

\vspace{-0.1in}
\item \textbf{Technical contributions.}

In the under-parameterized regime, we must analyze the convergence of both the representation and personalized heads toward their optimal estimations. This is in sharp contrast to the over-parameterized regime, where we only need to study the convergence of the representation column space to the ground truth~\citep{collins2021exploiting, zhong2022feddar}. Towards this end, we adopt a low-rank matrix approximation framework~\citep{chen2023fast}  of the ground-truth model. However, in contrast to conventional low-rank matrix approximation, in FRL, the global model is not accessible a priori but must be learned from distributed local datasets. Thus, the technical analysis needs to bound the unavoidable gradient discrepancy in the under-parameterized regime, as well as ensure that neither gradient discrepancy nor noise-induced errors accumulate over iterations. To address these technical challenges, we first provide new concentration results to ensure that the norm of the gradient discrepancy can be bounded when local datasets are sufficiently large. We then develop iteration-dependent upper bounds for sample complexity, which guarantee that the improvement in the estimation, i.e., the 'distance' between our estimated model and the optimal low-rank model, can mitigate potential disturbances caused by gradient discrepancy and noise.

\vspace{-0.1in}
\item \textbf{Empirical evaluation.\footnote{Main experiments can be reproduced with the code provided under the following link: \url{https://github.com/RenpuLiu/flute}}} 
We conduct a series of experiments utilizing both synthetic datasets for linear \alg and real-world datasets, specifically CIFAR-10 and CIFAR-100 \citep{krizhevsky2009learning}, for general \alg. The empirical results demonstrate the advantages of \alg, as evidenced by its superior performance over baselines, particularly in the scenarios where the level of under-parameterization is significant.

\end{itemize}

\section{Related Work}
 
\textbf{Representation learning.} 
Representation learning focuses on acquiring a representation across diverse tasks to effectively extract feature information{~\citep{lecun2015deep, tripuraneni2021provable,wang2016distributed,finn2017model}}. In the linear multi-task learning setting, \citet{du2020few} characterize the optimal solution of the empirical risk minimization (ERM) problem, demonstrating that the gap between the solution and the ground-truth representation is upper bounded by ${\mathcal{O}}\big(\sqrt{\frac{M+d}{MN}}\big)$, where $d$ is the dimension of data, $M$ is the number of clients and $N$ is the number of samples per task. \citet{tripuraneni2021provable} give an upper bound 
$\mathcal{O}\big(\sqrt{\frac{d}{MN}}\big)$ using the Method-of-Moment estimator. \citet{thekumparampil2021sample} also show the $\mathcal{O}\big(\sqrt{\frac{d}{MN}}\big)$ upper bound in their work. \citet{duchi2022subspace} consider data-dependent noise and show that the sample complexity required to recover the shared subspace of the linear models scales in $O\big(\log^3(Nd)\sqrt{\frac{d}{MN}}\big)$.
These works, however, only focus on the over-parameterized regime in a centralized setting.  

\textbf{Federated representation learning.}  
Recently, representation learning has been introduced to FL 
\citep{arivazhagan2019federated,liang2020think,collins2021exploiting,yu2020salvaging}. \citet{liang2020think} propose an FRL framework named Fed-LG, where the distinct representations are stored locally and the common prediction head is forwarded to the server for aggregation. In contrast, \citet{arivazhagan2019federated} propose FedPer, where a common representation is shared among clients, with personalized local heads kept at the client side. A similar setting is adopted by FedRep \citep{collins2021exploiting}, where exponential convergence to the optimal representation in the linear setting is proved.
These works focus on the over-parameterized regime, while the under-parameterized regime has largely been overlooked.

\if{0}
\begin{table*}[t]
\centering
\vspace{-0.2cm}
\caption{Comparison of the setup with related works on the low-rank matrix factorization problem.}
\centering 
\begin{tabular}{c|c|c|c|c}
\midrule
\midrule
&  Under-Param. & Algorithm-Based & Data-Based  & Federated\\
\midrule
\citet{du2020few}                             &                     &                & \cmark     &              \\
\citet{ye2021global}                          &                     & \cmark         &            &              \\
\citet{tripuraneni2021provable}               &                     &                & \cmark     &              \\
\citet{collins2021exploiting}                 &                     & \cmark         & \cmark     & \cmark       \\
\citet{pitaval2015convergence}                & \cmark              & \cmark         &            &              \\
{\citet{ge2017spurious}}                      & \cmark              &          &            &              \\
\citet{zhu2021global}                         & \cmark              & \cmark         &            &              \\
\citet{chen2023fast}                          & \cmark              & \cmark         &            &              \\
This work                                     & \cmark              & \cmark         & \cmark     & \cmark       \\
\midrule
\midrule
\end{tabular}
\label{tab:related-low-rank}
\scriptsize
\\
% In the table:
% The 'Methodology' entry gives the methods or algorithms to solve the problem.
The 'Under-Param.' entry denotes whether the under-parameterized setting is considered.
The 'Algorithm-based' entry specifies whether the objective function optimization relies on an iterative approach.
The 'Data-Based' entry highlights if the objective function is a learning problem over the observed data or an optimization problem directly leveraging the ground-truth model.
The 'Federated' entry indicates whether the algorithm is designed for a federated learning setting.

\end{table*}
\fi

%\jingc{Low-rank matrix factorization. Does asymmetric matrix approximation belong to this as well?}
%\rpc{Yes asymmetric matrix approximation belongs to Low-rank matrix approximation.}

\textbf{Low-rank matrix factorization.} Under-parameterized representation learning problem considered in this work is closely related to low-rank matrix factorization, where the objective is to find two low-rank matrices whose product is closest to a given matrix $\Phiv$.  
\citet{pitaval2015convergence} prove the global convergence of gradient search with infinitesimal step size for this problem. \citet{ge2017spurious} demonstrate that no spurious minima exists in such a problem and all saddle points are strict. Based on a revised robust strict saddle property, \citet{zhu2021global} show that the local search method such as {gradient descent} leads to a linear convergence rate with good initialization with a regularity condition on $\Phiv$. 
\citet{chen2023fast} extend the analysis in \citet{zhu2021global} to general $\Phiv$, and show that with a moderate random initialization, the gradient descent method will converge globally at a linear rate.
In the over-parameterized regime, \citet{ye2021global} proves that the gradient descent method will converge to a global minimum at a polynomial rate with random initialization. We note, however, that these works assume the perfect knowledge of $\Phiv$, which is different from the data-based {\it representation learning} problem considered in this work.

\section{Problem Formulation}
\textbf{Notations.} 
We use $\diag (x_1,\cdots,x_d)$ to denote a $d$-dimension diagonal matrix with diagonal entries $x_1,\cdots,x_d$. 
$\langle x,y\rangle$ denotes the inner product of $x$ and $y$, and $\|x\|$ denotes the Euclidean norm of vector $x$. 
We use $f \circ \psi$ to denote the composition of functions $f:\Rb^{k}\rightarrow \Rb^{m}$ and $\psi:\Rb^{d}\rightarrow \Rb^k$, i.e., $(f \circ \psi)(x)=f(\psi(x))$. 
$\mathbf{I}_d$ represents a $d\times d$ identify matrix, and $\mathbf{0}$ is a $d$-dimensional all-zero vector.

\textbf{FL with common representation.}
We consider an FL system consisting of $M$ clients and one server. Client $i$ has a local dataset $\mathcal{D}_i$ that consists of $n_i$ training samples $(x,y)$ where $x\in\Rb^d$ and $y\in\Rb^m$. For simplicity, we assume $n_i=N$ for all client $i\in[M]$. 
For $(x_{i,j}, y_{i,j})\in\Dc_i$, we assume $y_{i,j} = g_i(x_{i,j})+\xi_{i,j}$, where $x_{i,j}$ is randomly drawn according to a sub-Gaussian distribution $P_{X}$ with mean $\mathbf{0}$ and covariance matrix $\mathbf{I}_d$, $g_i: \Rb^{d}\rightarrow \Rb^m$ is a deterministic function, and $\xi_{i,j}\in\Rb^{m}$ is an independent and identically distributed (IID) centered sub-Gaussian noise vector with covariance matrix $\sigma^2\mathbf{I}_d$.
 
\if{0}
\begin{align}\label{def:FLerm}
    \min_{g_i}\frac{1}{M}\sum_{i\in[M]}\frac{1}{|\Dc_i|}\sum_{(x,y)\in\Dc_i}\ell(g_i(x),y).  
\end{align}
\fi

Federated representation learning (FRL) aims at learning both a common {representation} that suits all clients and an individual head that only fits client $i$. 
{An FL framework adopting this principle was proposed by \citet{arivazhagan2019federated}, and we follow the same framework in this paper.}
More specifically, we assume that the local model of client $i$ can be decomposed into two parts: a common representation $\psi_\Bv:\Rb^{d}\rightarrow\Rb^{k}$ shared by all clients and a local head $f_{w_i}:\Rb^{k}\rightarrow \Rb^m$, where $\Bv$ and $w_i$ are the parameters of the corresponding functions. 
Then, the ERM problem considered in this FRL framework can be formulated as: 
\begin{align}\label{def:genral-emp-risk}
    \min_{\Bv,\{w_i\}}\frac{1}{M}\sum_{i\in[M]}\frac{1}{N}\sum_{(x,y)\in\Dc_i}\ell((f_{w_i}\circ\psi_{\Bv})(x),y).  
\end{align} 
This formulation leverages the common representation while accommodating data heterogeneity among clients, facilitating efficient personalized model training~\citep{arivazhagan2019federated, collins2021exploiting}.

In this work, we focus on the under-parameterized setting in FRL, which is formally defined as follows.
\if{0}
\begin{definition}[Under-Parameterization in FRL.]\label{def:under-param}
    Given a local head class $\Fc$ and a representation class $\Psi$, an FRL model (class) is called under-parameterized if it does not contain all ground-truth models, namely, there exists at least one client $i$ such that $\psi\circ f\neq g_i$ for any $\psi\in\Psi$, $f\in\Fc$.
\end{definition}
\fi

\begin{definition}[Under-Parameterization in FRL]\label{def:under-param}
    Given a common representation class $\Psi$ and a collection of local head classes $\{\Fc_i\}_{i=1}^M$, an FRL problem is under-parameterized if there does not exist a representation $\psi\in\Psi$, and a collection of functions $f_1\times f_2\ldots\times f_M\in \Fc_1\times\Fc_2\ldots\times \Fc_M$ such that $f_i\circ \psi = g_i$ for all $i\in [M]$.
\end{definition}

The over-parameterization in FRL can be defined in a symmetric form. 
This definition aligns with the over-parameterized frameworks in matrix approximation, as detailed in \citet{jiang2022algorithmic, ye2021global}, where over-parameterization is characterized by the rank of the representation being no-less than that of the ground-truth model. It also encompasses the definition in central statistical learning~\citep{Belkin_2019, oneto2023we}, where over-parameterization is defined as the predictor's function class being sufficiently rich to approximate the global minimum.

While various algorithms have been developed and analyzed in the over-parameterized setting~\citep{arivazhagan2019federated,liang2020think,collins2021exploiting}, to the best of our knowledge, under-parameterized FRL has not been studied in the literature before. This is, however, arguably a more practical setting in large-scale FRL supported by a massive number of resources-scarce IoT devices, as such IoT devices usually cannot support the storage, computation, and communication of models parameterized by a large number of parameters, while the task heterogeneity across massive devices imposes significant challenges on the model class to reconstruct $M$ different local models perfectly\footnote{Continuing the previous example of MobileNet, which can be adapted for object detection for autonomous driving \citep{chen2021deep}, it is known that a single model may not capture very detailed or complex features of the complete environment, including pedestrians, cyclists, and various road signs~\citep{chen2022mobile}.}.

\mypara{Low-dimensional linear representation.}
We first focus on the \emph{linear} setting in which all local models $g_i$ are linear, i.e., $y_{i,j} = \phi_i^\top x_{i,j} +\xi_{i,j}$ for $(x_{i,j}, y_{i,j})\in\Dc_i$. Denote $\Phiv := [\phi_1,\cdots,\phi_M]\in\Rb^{d\times M}$ and assume its rank is $r$. Then, similar to the works of \citet{collins2021exploiting,arivazhagan2019federated}, we consider a linear prediction model where {$(f_{w_i}\circ \psi_{\Bv})(x)$} can be expressed as {$x^\top\Bv w_i$}. Here, $\Bv\in\Rb^{d\times k}$ is the common linear representation shared across clients, and $w_i\in\Rb^k$ is the local head maintained by client $i$. We denote $\Wv = [w_1,\cdots,w_M]$. Then, if we further consider the $\ell_2$ loss function, the ERM problem becomes 
\begin{align}\label{def:linear-emp-risk}
    \min_{\Bv,\Wv}\frac{1}{M}\sum_{i\in[M]}\frac{1}{N}\sum_{(x,y)\in\Dc_i}\|x^\top\Bv w_i-y\|^2.
\end{align}
\if{0}
The corresponding population risk minimization problem is
\begin{align}\label{def:linear-popu-risk}
    \min_{\Bv,\Wv}\frac{1}{M}\sum_{i\in[M]}\mathop{\mathbb{E}}_{\substack{ x\sim P_x,\\y\sim P_{i,Y|X}}}\|x^\top\Bv w_i-y\|^2.
\end{align}
\fi
We note that the existing literature usually assumes that $r\leq k$, which falls in the {over-parameterized} regime \citep{zhu2021global}. The over-parameterized assumption implies the existence of a pair of $\Bv$ and $\Wv$ that can accurately recover the ground-truth model $\Phiv$, i.e., $\Bv\Wv = \Phiv$. Thus, the learning goal in the over-parameterized regime is to identify such a pair using available training data{~\citep{du2020few, tripuraneni2021provable,collins2021exploiting, shen2023share}}. 

In contrast to the existing works, in the \emph{under-parameterized} regime given in \Cref{def:under-param}, we have $r > k$, i.e., there does not exist matrices $\Bv\in \Rb^{d\times k}$ and $\Wv\in\Rb^{k\times M}$ such that $\Bv\Wv=\Phiv$. Our objective is to learn a common representation and local heads $(\Bv,\Wv)$ in the federated learning framework such that $\|\Bv\Wv-\Phiv\|_F^2$ reaches its minimum, although $\Phiv$ is not explicitly given but embedded in local datasets.

\section{The \alg Algorithm}\label{sec:linear-flute}\vspace{-0.01in}
In this section, we present the \alg algorithm for the linear model. We will first highlight the unique challenges the under-parameterized setting brings, and then introduce our algorithm design.

\subsection{Challenges}
In order to understand the fundamental differences between the over- and under-parameterized regimes, we first assume $\Phiv$ is known beforehand, and consider solving the following optimization problem: 
\begin{align}\label{eqn:ama}
    (\Bv^*,\Wv^*)=\mathop{\arg\min}_{
    \substack{\Bv\in\Rb^{d\times k}, \Wv\in\Rb^{k\times M}}
    }
    \|\Bv\Wv-\Phiv\|_F^2.
\end{align}
Denote the singular value decomposition (SVD) of $\Phiv$ as $\Uv\Lambdav\Vv^\top$, where $\Uv$ and $\Vv$ are two unitary matrices, and $\Lambdav$ is a diagonal matrix. When $k\geq r$, i.e., the model is over-parameterized, $\Bv^*$ and $\Wv^*$ can be explicitly constructed from the  SVD of $\Phiv$, i.e., any $(\Bv,\Wv)$ satisfying $\Bv\Wv = \Uv\Lambdav\Vv^\top$ is an optimizer to \Cref{eqn:ama}. When $k<r$, i.e., in the under-parameterized regime, we can no longer recover the full matrix $\Phiv$ with $\Bv^*$ and $\Wv^*$.  Instead, existing result~\citep{golub2013matrix} states that we can only determine that the solution must satisfy $\Bv^*\Wv^* = \Uv_k\Lambdav_k\Vv^\top_k$, where $\Lambdav_k$ is a $k\times k$ diagonal matrix with the $k$ largest singular values of $\Phiv$ as the diagonal entries. 

Compared with the over-parameterized setting, learning $\Bv^*$ and $\Wv^*$ from decentralized datasets is more challenging in the under-parameterized setting. Let {$\Bv_i^\diamond$} be the locally optimized representation at client $i$, i.e., $(\Bv_i^\diamond,w_i^\diamond)=\argmin \|\Bv_i w_i-\phi_i\|^2$. Then, in the over-parameterized setting, {$\Bv_i^\diamond$} will always stay in the same column space as $\Bv^*$, i.e., $\text{span}(\Bv_i^{\diamond})\subseteq\text{span}(\Bv^*)$, $\forall i\in[M]$. However, for the under-parameterized setting, it is possible that $\text{span}(\Bv_i^\diamond)\not\subset\text{span}(\Bv^*)$, $\exists i\in[M]$. How to aggregate the locally obtained {$\Bv_i^\diamond$} to correctly {span} the column space of $\Bv^*$ thus becomes a unique challenge in the under-parameterized setting and requires novel techniques different from those in the existing over-parameterized literature.

\if{0}
existing result~\citep{golub2013matrix} states that the solution satisfies $\Bv^*\Wv^* = \Uv_k\Lambdav_k\Vv^T_k$, where $\Lambdav_k$ is a $k\times k$ diagonal matrix, its first $k$ diagonal entries equal to the first $k$ largest singular values of $\Phiv$. \congc{say a couple of sentences why we cannot analytically get $\Bv^*$ and $\Wv^*$ in this regime even if we know $\Phiv$.}
\rpc{I don't quite get the point}
\congc{For UP regime, you can get closed-form expressions for $\Bv^*$ and $\Wv^*$, but here we can only say $\Bv^*$ and $\Wv^*$ satisfy $\Bv^*\Wv^* = \Uv_k\Lambdav_k\Vv^T_k$. Explain why we cannot similarly get closed-form expressions for $\Bv^*$ and $\Wv^*$.}
\rpc{like over-param., for any $\Pv$, $\Bv^* = \Uv_r\Lambdav_k^{-\frac{1}{2}}[\Pv\quad \mathbf{0}]$ and $\Wv^* = [\Pv\quad \mathbf{0}]^T\Lambdav_k^{-\frac{1}{2}}\Vv_r^T$ is a pair of optimizer for under-param. problem.} \congc{Then we have the same closed-form sol for over- and under-param? What is the point of having this paragraph?}
\rpc{Here I want to highlight that for over-param, $\Bv^*$ and $\Wv^*$ can recover $\Phiv$: $\Bv^*\Wv^* = \Phiv$, but for under-param case, optimal solution can only recover the top k principal component of $\Phiv$: $\Bv^*\Wv^* = \Uv_k\Lambda_k\Vv^k$}

Compared with the over-parameterized setting, learning $\Bv^*$ and $\Wv^*$ from decentralized datasets is more challenging in the under-parameterized setting, since in the over-parameterized setting, locally optimizing the representation will stay in the same column space as $\Bv^*$, i.e., $\text{span}(\Bv)\subseteq\text{span}(\Bv^*)$, but for the under-parameterized setting, it is possible that locally optimizing the representation will make it deviate from the column space of $\Bv^*$, $\text{span}(\Bv)\not\subset\text{span}(\Bv^*)$.
\fi
\if{0}
These include FedPer~\citep{arivazhagan2019federated}, FedRep~\citep{collins2021exploiting} and \texttt{CENTAUR}~\citep{shen2023share}, in which clients retain their local head and only transmit their updated representation to the server for an aggregation. %, may exhibit suboptimal performance in the under-parameterized regime. 
The intuitive explanation is {the following.} In the over-parameterized setting, the locally updated representation will align with the column space of $\Bv^*$, i.e., $\text{span}(\Bv^*)$, then the server has the capacity to accurately recover $\text{span}(\Bv^*)$ by only receiving representations from local clients. However, the under-parameterized setting {has a unique} challenge that {\bf locally updated representations can deviate significantly from $\text{span}(\Bv^*)$}. In this case, the absence of local heads during the {server aggregation will negatively impact the server's capability to reconstruct $\text{span}(\Bv^*)$.
To illustrate this phenomenon, we present a toy example, which shows the potential limitation of algorithms that only transmit the representation in the under-parameterized context.
\fi
\begin{example}\label{exp:example1}
Consider a scenario {that $\Phiv\in\Rb^{d\times M}$ with $M<d$}. We assume 
    % \begin{aligned}
$        \Phiv = \Uv \diag (\lambda_1,\cdots,\lambda_M),$
    % \end{aligned}
where $\Uv:=[u_1,\ldots,u_M]$ is a unitary matrix and $\lambda_1>\lambda_2>\cdots>\lambda_M>0$. Assume $k=1$. Then, we have $\Bv^*\Wv^*=u_1\lambda_1$. Assume each client $i$ can perfectly recover its local model $\phi_i=u_i\lambda_i$ with $\Bv_i=u_i\lambda_i/w_i$. Then, depending on the value of $w_i$'s, the aggregated representation $\Bv:=\frac{1}{M}\sum_i \Bv_i$ may exhibit different properties. For example, if $w_i=\lambda_i$, we have $\Bv = \frac{1}{M}\sum_i u_i$, which deviates significantly from the column space of $\Bv^*$. On the other hand, if $w_i=\sqrt{\lambda_i/M}$, then $\Bv_i=u_i\sqrt{M\lambda_i}$, while  $\Bv = \sum_i u_i \sqrt{\lambda_i/M}$. Thus, $u_1$ will have a heavier weight in the aggregated representation, which will eventually help recover the column space of $\Bv^*$. Intuitively, to accurately recover the column space of $\Bv^*$, in the under-parameterized setting, it requires a more sophisticated algorithm design not just to estimate the column space of $\Phiv$, but also \textit{distill the most significant components} of it from distributed datasets in each aggregation. 
\end{example}

\subsection{A New Loss Function}\label{subsec:4.2}
Motivated by the observation in \Cref{exp:example1}, instead of considering the original problem in \eqref{def:linear-emp-risk}, we introduce {two new regularization terms and consider} the following ERM problem:
\begin{align}\label{def:AMA-emp}\vspace{-0.2in}
  \min_{\Bv,\{w_i\}_{i=1}^M}
       & \frac{1}{M}\sum_{i\in[M]}\frac{1}{N}\sum_{(x,y)\in\Dc_i}
        \|x^\top\Bv w_i-y\|^2
        \underbrace{-\gamma_1\|\Bv\Wv\|_F^2}_{\text{(I)}}+\underbrace{\gamma_2(\|\Bv^\top\Bv\|_F^2+\|\Wv\Wv^\top\|_F^2)}_{\text{(II)}}.
\end{align}
In \Cref{def:AMA-emp}, we introduce the regularization term (I) into the loss function, with the purpose of preserving the top-$k$ significant components of $\Bv\Wv$. By preserving the significant components in \(\Bv\), the term (I) mitigates local over-fitting induced during local updates. However, minimizing term (I) alone would result in a uniform enlargement of all $k$ singular values of $\Bv\Wv$. To address this, we further incorporate the regularization term (II). This term is specifically formulated to promote the $k$ most significant components and suppress the less significant ones. By doing so, it aids the server in accurately distilling the correct subspace spanned by the optimal representation. We note that when $\gamma_1=2\gamma_2$, (I) and (II) together recover the conventional penalty term $\|\Bv^\top\Bv-\Wv\Wv^\top\|_F^2$, which has been previously adopted for low-rank matrix approximation~\citep{chen2023fast,zhu2021global,wang2016unified} and multi-task learning~\citep{tripuraneni2021provable}.

\if 0
The regularization term $\|\Bv^\top\Bv-\Wv\Wv^\top\|_F^2$ has been previously adopted for low-rank matrix approximation~\citep{chen2023fast,zhu2021global,wang2016unified} and multi-task learning 
\citep{tripuraneni2021provable}. We adopt it here to {\it help the server automatically distill the correct subspace spanned by the optimal representation in the under-parameterized setting}.

Returning to Example \ref{exp:example1}, we observe that when the locally learned optimal representation and head are $\Bv_i =\sqrt{M\lambda_i}u_i$ and $w_i=\sqrt{\lambda_1/M}$, respectively, 
$\|\Bv^\top\Bv-\Wv\Wv^\top\|_F^2$ equals zero. 
Thus, including $\|\Bv^\top\Bv-\Wv\Wv^\top\|_F^2$ facilitates the proper selection of the column space of the representation during aggregation, which is critical in the under-parameterized setting. 
\fi

\begin{algorithm}[t]
\caption{\alg Linear}\label{alg:fedUP}
\begin{algorithmic}[1]
 \STATE {\bf Input:} Learning rates $\eta_{l}$ and $\eta_{r}$, regularization parameters $\gamma_1$ and $\gamma_2$, communication round $T$, constant $\alpha$.
\STATE {\bf Initialization:}  
{All entries of $\Bv^0$ and $\Wv^0$ are independently sampled form $\Nc(0,\alpha^2)$.}
\FOR{$t\in[T]$}
\STATE Server sends $\mathbf{B}^{t-1}$ and ${w}_i^{t-1}$ to client $i$, $\forall i\in [M]$.
\FOR{client $i\in[M]$ {in parallel}} 
\STATE Calculates $\nabla_{{w}_i^{t-1}}L_i\mleft({w}_i^{t-1},\mathbf{B}^{t-1}\mright)$ and $\nabla_{\mathbf{B}^{t-1}}L_i\mleft({w}_i^{t-1},\mathbf{B}^{t-1}\mright)$.
\STATE {Sends gradients to the server.}
\ENDFOR
\STATE {Server updates} according to Equations~\eqref{equ:server-update1} to \eqref{equ:server-update}.
\ENDFOR
\end{algorithmic}
\end{algorithm}

\subsection{\alg for Linear Model}

In order to solve the optimization problem given in \eqref{def:AMA-emp}, we introduce an algorithm named \alg (\textbf{F}erated \textbf{L}earning in \textbf{U}nder-parame\textbf{T}erized R\textbf{E}gime), which is compactly described in \Cref{alg:fedUP}. Specifically, for each epoch, the algorithm consists of three major steps, namely, \emph{server broadcast}, \emph{client update}, and \emph{server update}.

\textbf{Server broadcast.} At the beginning of epoch $t$, the server broadcasts the representation $\Bv^{t-1}$ to all clients, and $w_i^{t-1}$ (i.e., the $i$-th column of $\Wv^{t-1}$) to each individual client $i$.

\textbf{Client update.} 
Denoting the local loss function as \[L_i = \frac{1}{N}\sum_{(x,y)\in\Dc_i} \|x^\top\Bv w_i-y\|^2,\] the client calculates the gradient {of $L_i$} with respect to $w_i^{t-1}$ and $\Bv_i^{t-1}$, respectively, and uploads them to the server.
%as follows:
\if{0}
\begin{align}
    \nabla_{\mathbf{w}_i^{t-1}}L_i\mleft({w}_i^{t-1},\mathbf{B}^{t-1}\mright),\quad
    \nabla_{\mathbf{B}^{t-1}}L_i\mleft({w}_i^{t-1},\mathbf{B}^{t-1}\mright),
\end{align}
{Finally, client $i$ uploads them to the server.}
\fi

\textbf{Server update.}
After receiving $\nabla_{\mathbf{w}_i^{t-1}}L_i$ and $\nabla_{\mathbf{B}^{t-1}}L_i$ from all clients, the server first aggregates them to update the global representation and local heads as follows:
\begin{equation} \label{equ:server-update1}
\begin{aligned}
        \bar{\Bv}^t 
        &=
        \mathbf{B}^{t-1}-{\eta_l}\sum_{i\in[M]}\nabla_{\mathbf{B}^{t-1}}L_i\mleft({w}_i^{t-1},\mathbf{B}^{t-1}\mright),\\
             {w}^{t}_i &= w_i^{t-1}-\eta_{l}\nabla_{w_i^{t-1}}L_i\mleft({w}_i^{t-1},\mathbf{B}^{t-1}\mright), \forall i\in[M],
    \end{aligned}
\end{equation}
after which it constructs matrix $\bar{\Wv}^t$ by setting $\bar{\Wv}^t :=
        [w_1^t,\cdots,w_M^t]$.
        %${w}^{t}_i$ as its $i$-th column. 
It then performs another step of gradient descent with respect to the regularization term in \eqref{def:AMA-emp} to refine the global representation and local heads and obtain $\Bv^t$ and $\Wv^t$:
\begin{equation}
\begin{aligned}
&{\Bv}^t =\bar{\Bv}^t+\gamma_1\eta_r\nabla_{\Bv^{t-1}}\|\Bv^{t-1}\Wv^{t-1}\|_F^2-\gamma_2\eta_r\nabla_{\Bv^{t-1}}(\|(\Bv^{t-1})^\top\Bv^{t-1}\|_F^2+\|\Wv^{t-1}(\Wv^{t-1})^\top\|_F^2), \label{equ:server-update}\\
&{\Wv}^t =\bar{\Wv}^t+\gamma_1\eta_r\nabla_{\Wv^{t-1}}\|\Bv^{t-1}\Wv^{t-1}\|_F^2-\gamma_2\eta_r\nabla_{\Wv^{t-1}}(\|(\Bv^{t-1})^\top\Bv^{t-1}\|_F^2+\|\Wv^{t-1}(\Wv^{t-1})^\top\|_F^2),
\end{aligned}
\end{equation}
\if{0}
\begin{align}
    &{\Bv}^t=\bar{\Bv}^t-\rp{\lambda\eta_r} \nabla_{\Bv^{t-1}}\|(\Bv^{t-1})^T\Bv^{t-1}-\Wv^{t-1}(\Wv^{t-1})^T\|^2_F,\\
    &\Wv^t = \bar{\Wv}^t-\rp{\lambda\eta_r} \nabla_{\Wv^{t-1}}\|(\Bv^{t-1})^T\Bv^{t-1}-\Wv^{t-1}(\Wv^{t-1})^T\|^2_F.
\end{align}
\fi
The procedure repeats until some stop criterion is satisfied.

\if{0}
\begin{align*}
\bar{\Bv}^t &= \mathbf{B}^{t-1}-{\eta_l}\sum_{i\in[M]}\nabla_{\mathbf{B}^{t-1}}L_i\mleft({w}_i^{t-1},\mathbf{B}^{t-1}\mright)\\
{w}^{t}_i &= w_i^{t-1}-\eta_{l}\nabla_{\mathbf{w}_i^{t-1}}L_i\mleft({w}_i^{t-1},\mathbf{B}^{t-1}\mright),\quad \forall i\in[M]\\
(\bar{\Wv}^t)^T &\leftarrow [w_1^t,\cdots,w_M^t];\\
{\Bv}^t & \leftarrow \bar{\Bv}^t-{\lambda\eta_r}\nabla\|(\Bv^{t-1})^T\Bv^{t-1}-\Wv^{t-1}(\Wv^{t-1})^T\|^2_F\\
\Wv^t & \leftarrow \bar{\Wv}^t-{\lambda\eta_r}\nabla\|(\Bv^{t-1})^T\Bv^{t-1}-\Wv^{t-1}(\Wv^{t-1})^T\|^2_F
\end{align*}
\fi
\begin{remark}
    {When $\alpha$ is small,} the initialization of $\Bv^0$ and $\Wv^0$ would ensure that the largest singular value of $\Bv^0(\Wv^0)^T$ is sufficiently small with high probability. As we will show in the next section, such initialization guarantees that \alg converges to the global minimum. 
\end{remark}

The major differences between \alg and existing FRL algorithms such as FedRep~\citep{collins2021exploiting}, FedRod~\citep{chen2021bridging}, and FedCP~\citep{zhang2023fedcp} lie in the server-side model updating. While these existing algorithms typically involve transmitting only the shared representation layers of local models to the server, with local heads being optimized and utilized exclusively at the client side, \alg requires clients to transmit both the shared representation layers and the local heads to the server. {The increased communication cost is fundamentally necessary due to the unique nature of FRL in the under-parameterized regime, as it} allows for server-side optimization, not just aggregation, of the entire model. Furthermore, \alg introduces additional data-free penalty terms to the server-side updates. These terms are designed to guide the shared representation to converge toward the global minimum by leveraging the information in the local heads. This approach represents a significant paradigm shift in federated learning, aiming to enhance the overall global performance of the FRL model.

\section{Theoretical Guarantees}\label{sec:section5}
\if{0}
\begin{assumption}[IID data distribution]\label{assump:sub-g-x}
    For $(x,y)\in\Dc_i$ for client $i$, the samples $x\in\Rb^d$ are IID sub-gaussian vectors with $0$ mean and $\Iv_d$ covariance matrix. 
\end{assumption}

\rp{Also, in this work, we assume the noise term is sub-Gaussin, which is a standard assumption in the relevant literature.
\begin{assumption}[Sub-Gaussian noise]\label{assump:sub-g-noise}
    For $(x,x^\top \phi_i+\xi)\in\Dc_i$ for any client $i$, $\xi$ is a centered sub-Gaussian noise with covariance $\sigma^2$.
\end{assumption}
}
\fi

Before introducing our main theorem, 
we denote $\dl = \min\{d,M\}$ and $\dh = \max\{d,M\} $. We also denote $\lambda_1\geq\lambda_2\geq\cdots\geq\lambda_{\dl}$
as the ordered singular values of $\Phiv$ with $\Delta := 2(\lambda_k-\lambda_{k+1})$. Denote $E = \sum_{i}\lambda_i^2$.  We assume $\Delta>0$ throughout the analysis.

\subsection{Main Results}

\begin{theorem}[Sample complexity]\label{thm:thm1}
    Set $\gamma_1 = \frac{1}{4}$ and $\gamma_2 = \frac{1}{8}$ in \cref{def:AMA-emp}. Let $0<\alpha \lesssim \frac{1}{10d}$, and $\eta:=\eta_l=\eta_r\lesssim \frac{\Delta^2}{228{\lambda_1}^3}$. % Denote $\dh = \max\{d,M\} $. 
    Then, for any $\epsilon>0$ and \(0\leq\delta\leq 1\), under \Cref{alg:fedUP}, there exists positive constants $c$ and $c'$ such that when the number of samples per client satisfies 
    $$
    N\geq c\frac{\lambda^4_1 k(\dh+\log\frac{1}{\delta}+\log \log \frac{1}{\epsilon})(\sqrt{k(\lambda_1)^2 +E}+\sqrt{k}\sigma)^2}{M\eta^2\Delta^6\epsilon^2},
    $$
 and $t\geq \frac{\log(\epsilon\sqrt{M}\eta\Delta^2/{c'\lambda_1^2\sqrt{k}})}{\log(1-\eta\Delta/{16})}$, with probability at least $1-\delta$, 
    \begin{align}
        \frac{1}{M}\sum_{i\in[M]}\|\Bv^t w^t_i-\Bv^* w_i^*\|\leq \epsilon.
    \end{align} 
\end{theorem} 
\if{0}
\begin{remark}[Interpretation of \cref{thm:thm1}]
    As shown in \cref{thm:thm1}, in high-dimensional data settings, we observe an $M$-fold speedup in sample complexity compared to the single-agent case. However, 
    in scenarios where the data dimension $d$ is less than the number of clients $M$, the per-user sample requirement essentially stabilizes to a constant.
    One intuitive explanation is that in a lower-dimensional space ($d<M$), each additional agent might contribute diminishing marginal information. The available information needed to recover the top-$k$ significant components of ground-truth model $\Phiv$ might already be sufficiently captured by the existing agents, making the contribution of additional agents less significant.
\end{remark}

\jingc{when $d>M$, i.e., when the data is high-dimensional, the result indicates $M$-fold speedup compared with the single-agent case. However, when $d<M$, the required number of samples per user becomes a constant. Need to justify why linear speedup cannot be achieved in this regime.}
\jingc{should provide a better interpretation of the results and explain the dependency on all important parameters intuitively.}
\fi

\begin{remark}
\Cref{thm:thm1} indicates that the per-client sample complexity scales in $\tilde{\mathcal{O}}\left(\frac{\max\{d,M\}}{M\epsilon^2}\right)$. Compared with the single-client setting, which is essentially a noisy linear regression problem with sample complexity $\mathcal{O}(\frac{d}{\epsilon^2})$ \citep{hsu2014random}, \alg achieves a linear speedup in terms of $M$ in the high dimensional setting (i.e., $d>M$). When $d<M$, the sample complexity of \alg becomes independent with $M$, which is due to the fact that each client requires a minimum number of samples to have the local optimization problem non-ill-conditioned. Compared with the sample complexity $\mathcal{O}\left(\frac{d}{M}+\log(M)\right)$ of FRL in the noiseless over-parameterized setting \citep{collins2021exploiting}, \alg achieves more favorable dependency on $M$.   
\end{remark}

\begin{remark}
We note that the dependency on $\Delta$ and $\lambda_1$, especially $\Delta$, is unique for the under-parameterized FRL. 
For the special case when $\lambda_{k+1}^*=0$, the problem we consider essentially falls into the over-parameterized regime, and \alg can still be applied.   
\Cref{thm:thm1} shows that the sample complexity scales in
$\mathcal{O}(\frac{\max\{d,M\}}{M\epsilon^2}(\frac{\lambda_1^*}{\Delta})^{10})$.  We note that under the assumption that $\mathbf{B}^*$ consists of orthonormal columns, the SOTA sample complexity in the over-parameterized regime scales in $\mathcal{O}(\frac{\max\{d,M\}}{M\epsilon^2}\kappa^4)$ \citep{tripuraneni2021provable}, where $\kappa = {\sigma_1((\mathbf{W}^*)^\top\mathbf{W}^*)}/{\sigma_r((\mathbf{W}^*)^\top\mathbf{W}^*)}$.  Under the same assumption on $\mathbf{B}^*$, we have $\lambda_1^* =\sqrt{\sigma_1((\mathbf{W}^*)^\top\mathbf{W}^*)} $, $\Delta = \sqrt{\sigma_k((\mathbf{W}^*)^\top\mathbf{W}^*)}$, and our sample complexity then becomes $\mathcal{O}(\frac{\max\{d,M\}}{M\epsilon^2}\kappa^5)$. The additional order of $\kappa$ in the bound is due to an initial state-dependent quantity bounded by $\frac{\lambda_1^*}{\Delta}$. The detailed analysis can be found in \Cref{{sec:proof-s3}}.
\end{remark}

\begin{remark}
 The sample complexity in \Cref{thm:thm1} requires that the size of \textit{each local dataset} be sufficiently large. This is in stark contrast to the sample complexity result in existing works \citep{collins2021exploiting}, which imposes a requirement on the \textit{total number of samples in the system} instead of on each individual client/task.  
We need the size for each local dataset to be sufficiently large to ensure that every $\phi_i$ can be locally estimated with a small error so that the top $k$ components of the ground truth $\mathbf{\Phi}$ can be correctly recovered. 
\end{remark}

\begin{theorem}[Convergence rate]\label{prop:prop1}
        Set $\gamma_1$, $\gamma_2$ and $\eta$ as in \Cref{thm:thm1}. Denote $\kappa_T = (1-\frac{\eta\Delta}{16})^T$.
        Then, for a constant {$T_\Rc$} (defined in \Cref{def:TR} in \Cref{apdx:B}) and any $T>T_\Rc$, 
        there exist positive constants $c_1$ and $c_2$ such that for any \(0\leq\delta\leq 1\), when number of samples per client satisfies %the number of samples per client satisfies 
        \[N\geq
        c_1\frac{(\bar{d} +\log\frac{1}{\delta}+\log T)(\sqrt{k(\lambda_1)^2 +E}+\sqrt{k}\sigma)^2}{\kappa_T^2\Delta^2},\]
        for all $T_\Rc<t\leq T$,  with probability at least $1-\delta$, 
        we have
            \begin{align}
        \frac{1}{M}\sum_{i\in[M]}\|\Bv^t w^t_i-\Bv^* w_i^*\|
        \leq
        \frac{c_2\lambda_1^2\sqrt{k}}{\sqrt{M}\eta\Delta^2}\Big(1-\frac{\eta\Delta}{16}\Big)^t.
    \end{align}
\end{theorem}

\begin{remark}%[Requirement for exponentially enlarged sample size]
\Cref{prop:prop1} shows that when the number of samples per client $N$ is sufficiently large, \alg converges {\it exponentially} fast. We note that the required number of samples grows exponentially in the total number of iterations. Such an exponential increase in the required number of samples is essential to guarantee that the `noise' level, which is the gradient estimation error, decays at least as fast as the decay rate of the representation estimation error, which is exponential. Similar phenomenon has been observed in the literature \citep{mitra2021linear,zhang2023metalearning}. 
In our problem, there are essentially two parts of `noise' in the learning process. One is the sub-exponential label noise $\xi_{i,j}$, and the other is the gradient discrepancy arising from the {\it under-parameterized} nature. This discrepancy persists even when $\Bv^t$ and $\Wv^t$ are nearly optimal, leading to an unavoidable gap between $\Bv^t\Wv^t$ and $\Phiv$. This gap behaves similarly to the sub-Gaussian noise in the convergence analysis, as elaborated in \cref{sec:proof-sketch}. Therefore, an exponential increase in the number of samples is required to cope with both parts of the noise and ensure the one-step improvement of the estimation error as iteration grows.
\end{remark}

\begin{remark}
 We also note that both the sample complexity in \Cref{thm:thm1} and the convergence rate in \Cref{prop:prop1} are influenced by $\Delta$, the gap between $\lambda_k$ and $\lambda_{k+1}$. A smaller $\Delta$ signifies a growing challenge in correctly identifying the top-$k$ principal components of $\Phiv$, leading to increased sample complexity and slower convergence. This is due to the challenge of accurately distinguishing and recovering the $k$-th and $(k+1)$-th significant components from the dataset when $\Delta$ is small. Note that in order to successfully distinguish $\sigma_k$ and $\sigma_{k+1}$, we need to estimate them to be $\Delta/2$-accurate, i.e., $|\hat{\sigma}_k-\sigma_k|\leq \Delta/2$ and $|\hat{\sigma}_{k+1}-\sigma_{k+1}|\leq \Delta/2$. Hence, the required number of samples per client would grow significantly when $\Delta$ is small, and this is arguably inevitable.
\end{remark}
\if{0}
\begin{remark}[Generalization performance]
    The convergence result also implies that $\Bv^t$ and $\Wv^t$ will approach the model of the optimal generalization performance, since when Assumption \ref{assump:sub-g-x} holds, the population risk defined in \eqref{def:linear-popu-risk} can be rewritten as 
    \begin{align}
        \frac{1}{M}\sum_{i\in[M]}\mathop{\Eb}\|x^T\Bv w_i-y\|^2=\frac{1}{M}\|\Bv\Wv-\Phiv\|_F^2,
    \end{align}
    then the following corollary shows the generalization performance of $\Bv^t$ and $\Wv^t$ learned from \alg:
\end{remark}
\begin{corollary}
    For any constant $\epsilon\geq 0$, let $\eta\leq \Delta^2/(36{\lambda^*_1}^2)$ and $N$ large enough, then there exist positive constants $c_8,c_9 >0$ such that 
\begin{align}
    \Lc(\Bv^t,\Wv^t)
    \leq
    \Lc(\Bv^*,\Wv^*)
    +
    c_8\sqrt{\Big(1-\frac{\eta\Delta}{16}\Big)^{t}}\quad,
\end{align}
with probability at least $1-e^{-c_9(d+k)}$ when $t$ is large.
\end{corollary}
\fi

\subsection{Proof Sketch}\label{sec:proof-sketch}
In this subsection, we outline the major challenges and main steps in the proof of \Cref{prop:prop1} while deferring the complete analysis to \Cref{apdx:B}. \Cref{thm:thm1} can be proved once \Cref{prop:prop1} is established. 

\if{0}
First, we transform the asymmetric matrix factorization problem into a symmetric problem, {to help us jointly analyze $\Bv^t$ and $\Wv^t$.}  Let $\tilde{\Bv}^t=\Uv^\top\Bv^t$ and $\tilde{\Wv}^t=\Wv^t\Vv$. Then, from \alg we have the following updating rules for $\tilde{\Bv}^t$ and $\tilde{\Wv}^t$:
    \begin{align*}
        \tilde{\Bv}^{t+1}
        = &
    \tilde{\Bv}^{t}-\eta\big(\tilde{\Bv}^{t}\tilde{\Wv}^{t}
    -\Lambdav\big)\big(\tilde{\Wv}^{t}\big)^\top\\
    &-\frac{\eta}{2}\tilde{\Bv}^{t}\big((\tilde{\Bv}^{t})^\top\tilde{\Bv}^{t}-\tilde{\Wv}^{t-1}(\tilde{\Wv}^{t})^\top\big)+\Uv^\top\Qv^{t+1},\\
    \tilde{\Wv}^{t+1}
     = &
    \tilde{\Wv}^{t}-\eta(\tilde{\Bv}^{t})^\top\big(\tilde{\mathbf{B}}^{t} \tilde{\Wv}^{t}-\Lambdav\big)+\frac{\eta}{2}\big((\tilde{\Bv}^{t})^\top\tilde{\Bv}^{t}\\
    &-\tilde{\Wv}^{t}(\tilde{\Wv}^{t})^\top\big)(\tilde{\Wv}^{t})^\top+\tilde{\Qv}^{t+1}\Vv,
    \end{align*}
where $\Qv^{t+1}$ and $\tilde{\Qv}^{t+1}$ are defined in Appendix \ref{apdx:B-prelim} as \eqref{def:Q} and \eqref{def:tildeQ} respectively. 
{To convert the asymmetric problem to a symmetric problem, we construct $\Bv^t_\star\in\Rb^{\dh\times k}$ and $\Wv^t_\star\in\Rb^{k\times \dh}$ that are padded from $\tilde{\Bv}^t$ and $\tilde{\Wv}$ by adding $\mathbf{0}$ columns or rows. For $\Qv^t$ and $\tilde{\Qv}^t$, we construct $\Qv^t_\star\in\mathbb{R}^{\dh\times k}$ and $\tilde{\Qv}^t_\star\in\mathbb{R}^{k\times \dh}$ by padding}. Then we define 
\begin{align}\label{def:Thetav}
    \Thetav^t 
        &=
        \bigg[\frac{(\Bv^t_\star)^\top+\Wv_\star^t}{\sqrt{2}}
        \quad
        \frac{(\Bv^t_\star)^\top-\Wv^t_\star}{\sqrt{2}}\bigg]^\top,\\    
        \Rv^{t}&=\bigg[\frac{(\Qv^{t}_\star)^\top\Uv_\star+\Vv^\top_\star\tilde{\Qv}^{t}_\star}{\sqrt{2}}\quad
\frac{(\Qv^{t}_\star)^\top\Uv_\star-\Vv^\top_\star\tilde{\Qv}^{t}_\star}{\sqrt{2}}\bigg]^\top.
\end{align}
Then, the updating rule of $\Thetav^t$ can be described as
\begin{align}\label{equ:equ13}
    \Thetav^{t+1} = \Thetav^{t}+\frac{\eta}{2}\tilde{\Lambdav}\Thetav^{t}-\frac{\eta}{2}\Thetav^{t}(\Thetav^{t})^\top\Thetav^{t}
    +\Rv^{t+1}.
\end{align}
To prove $\Bv^t\Wv^t$ converges to $\Bv^*\Wv^*$, we start by showing that $\Thetav^t(\Thetav^t)^\top$ converges to $\tilde{\Lambdav}_k$, through three main steps:

\mypara{Step 1.} 
First, we show that with a small random initialization, $\Thetav^t$ will enter a region containing the optima with high probability. {Inspired by} \citet{chen2023fast}, we define the following region:
\begin{align}
    \mathcal{R}=\left\{\Thetav=
    \begin{bmatrix}
        \Thetav_{k}\\
        \Thetav_{\text{res}}
    \end{bmatrix}\in\Rb^{2\dh\times k}\Bigg| 
    \begin{array}{l}
      \sigma_1^2(\Thetav)\leq2\lambda_1^*,\\
    \sigma_1^2(\Thetav_{\text{res}})\leq\lambda^*_k-\Delta/2,\\
    \sigma_k^2(\Thetav_k)\geq\Delta/4
    \end{array}
    \right\}
\end{align}
where $\Thetav = [\Thetav_k^\top, \Thetav_{\text{res}}^T]^\top$ for $\Thetav_k\in\Rb^{k\times k}$ and $\Thetav_k\in\Rb^{(\dh-k)\times k}$. Then we have the following theorem to ensure that with high probability, $\Thetav$ will enter this region with small random initialization.

\begin{theorem}[Informal]
\label{thm:thm2-informal}
    Assume all entries of $\Bv^0$ and $\Wv^0$ are independently sampled form $\Nc(0,\alpha^2)$ with $\alpha$ {satisfying} $0<\alpha\lesssim 1/10d$. Then with high probability, $\Thetav^t$ will enter the region $\Rc$ in at most $t=T_{\Rc}$ iterations, where $T_{\Rc}$ is defined in Theorem \ref{thm:thm2-formal} in the Appendix \ref{apdx:B-initial}.
\end{theorem}

We give the formal statement of Theorem \ref{thm:thm2-informal} in Appendix \ref{apdx:B-initial} as Theorem \ref{thm:thm2-formal}, together with the proof.
Compared with the proof of Theorem 3.2 in~\citet{chen2023fast}, our $\Thetav$ is estimated from the decentralized datasets; hence the region $\Rc$ is not guaranteed to be an absorbing region. The novelty is we can prove that $\Rc$ is an absorbing region with high probability.

\mypara{Step 2.} Next we show that when $\Thetav^t$ enter the region $\Rc$, it will stay in this region with high probability. This equivalently states that $\Rc$ is an absorbing region with high probability.

\begin{theorem}[Informal] 
\label{thm:thm3-informal}
Assume that $\Thetav^0\in\Rc$. For $N$ sufficiently large, one has $\Thetav^t\in\Rc$ for $t\geq 0$ with high probability.
\end{theorem}

We formally state and prove the Theorem \ref{thm:thm3-informal} in Appendix \ref{apdx:proof-absorbing}.

\mypara{Step 3.} Finally we show that when $N$ is sufficiently large, with high probability it holds that $\|\Thetav^t(\Thetav^t)^\top-\diag(\tilde{\Lambdav}_k,\mathbf{0})\|$ converges to $0$ at a linear rate when the initialization condition satisfies $ \Thetav^0\in\Rc$. {We note that this step is technically non-trivial. The previous work of \citet{chen2023fast} addresses a similar low-rank matrix approximation problem but uses the ground-truth model, which we do not have. Instead, we develop a new approach by lower bounding the number of samples needed for the {inverse of the SNR} to converge exponentially fast with high probability. This is given in \Cref{lemma:snr-informal}.}

\begin{lemma}[Informal] 
\label{lemma:snr-informal}
   Suppose $\eta\leq \Delta^2/(36\lambda^{*3}_1)$ and assume $\Thetav^t\in\Rc$ for all $0\leq t \leq T$. Let $N$ satisfy $N\gtrsim \frac{d-log\delta}{(1-\eta\Delta/16)^{2T}}$. Then for constant c, with probability at least $1-cT\delta$, we have
\begin{align}
        \frac{\sigma_1^2(\Thetav_{\text{res}}^{t})}{\sigma_k^2(\Thetav_k^{t})}\leq
        \frac{8\lambda_1^*}{\Delta}
       \Big(1-\frac{\eta\Delta}{16}\Big)^{2t}.
\end{align}
\end{lemma}

We formally present  \Cref{lemma:snr-informal} in \Cref{apdx:aux-lemmas} together with the proof. 

Next, we give a lemma that establishes the number of samples needed for $\Thetav_k$ to converge to $\tilde{\Lambdav}_k$, which is based on the convergence of inverse SNR.

\begin{lemma}[Informal]
\label{lemma:lemma2-informal}
Suppose $\eta\leq \Delta^2/(36\lambda^{*3}_1)$ and assume $\Thetav^t\in\Rc$ for all $0\leq t \leq T$. Let $N$ satisfy $N\gtrsim \frac{d-log\delta}{(1-\eta\Delta/16)^{2T}}$. Then for constant c, with probability at least $1-cT\delta$, we have
\begin{align}
    \sigma_1(\Dv^{t})\leq\frac{200{\lambda_1^*}^2}{\eta\Delta^2}\Big(1-\frac{\eta\Delta}{16}\Big)^{t},
\end{align}
where $\Dv^t = \|\Thetav_k^t(\Thetav_k^t)^\top - \tilde{\Lambdav}_k\|$.
\end{lemma}

The formal statement and proof of \Cref{lemma:lemma2-informal} are presented in \Cref{apdx:aux-lemmas}. Combining the two lemmas leads to the local convergence property of $\Thetav^t$:

\begin{theorem}
\label{thm:thm4-informal}
Suppose $\eta\leq \Delta^2/(36\lambda^{*3}_1)$ and assume $\Thetav^t\in\Rc$ for all $0\leq t \leq T$. Let $N$ satisfy $N\gtrsim \frac{d-log\delta}{(1-\eta\Delta/16)^{2T}}$. Then for constant c, with probability at least $1-cT\delta$, we have
\begin{align}
    \|\Thetav^t(\Thetav^t)^T-\diag(\tilde{\Lambdav}_k,\mathbf{0})\|_F\leq \kappa\sqrt{k}\Big(1-\frac{\eta\Delta}{16}\Big)^{t}.
\end{align}
\end{theorem}

The formal statement and proof of Theorem \ref{thm:thm4-informal} is given in  Appendix \ref{apdx:proof-absorbing}. 

We highlight that the most challenging part in the proof of Theorem \ref{thm:thm4-informal} is that the updating rule of $\Thetav^t$ {depends on} $\Rv^{t+1}$ in \eqref{equ:equ13}, which is caused by the distributed datasets.  {This is very different from the updating rule in \citet{jiang2022algorithmic,chen2023fast}, which relies on \emph{guaranteed} one-step improvement to establish convergence. In our case, however,  the one-step improvement is no longer guaranteed to exist.}

Finally, combining Theorems \ref{thm:thm2-informal}, \ref{thm:thm3-informal}, and \ref{thm:thm4-informal}, it is straightforward to show that when $t>T_{\Rc}$ and $N$ is sufficiently large, if $\Thetav^0$ satisfies the small random initialization stated in Theorem \ref{thm:thm2-informal}, it holds that $\|\Thetav^t(\Thetav^t)^\top-\diag(\tilde{\Lambdav}_k,\mathbf{0})\|\leq c_4\Big(1-\frac{\eta\Delta}{16}\Big)^{t}$ for constant $c_4 = \kappa\Big(1-\frac{\eta\Delta}{16}\Big)^{-T_{\Rc}}$.  

By applying Lemma \ref{lemma:bridge} in Appendix \ref{apdx:B-local-convg} we get
\begin{equation*}
    \begin{aligned}
        \|\tilde{\Bv}^t\tilde{\Wv}^t-\diag (\Lambdav_k,\mathbf{0})\|_F
        \leq
        \kappa\sqrt{k}\Big(1-\frac{\eta\Delta}{16}\Big)^{t}
    \end{aligned}
\end{equation*}
with high probability. Note that
$ \|\Bv^t\Wv^t-\Bv^*\Wv^*\|_F= \|\tilde{\Bv}^t\tilde{\Wv}^t-\diag (\Lambdav_k,\mathbf{0})\|_F$. Combing with the fact
$\sum_{i\in[M]}\|\Bv^t w_i^t-\Bv^*w_i^*\|^2= \|\Bv^t\Wv^t-\Bv^*\Wv^*\|_F^2$, we have
\begin{equation*}
    \begin{aligned}
    \sum_{i\in[M]}\|\Bv^t w_i^t-\Bv^*w_i^*\|^2 
    \leq
    \kappa^2k\Big(1-\frac{\eta\Delta}{16}\Big)^{2t}.
    \end{aligned}
\end{equation*}
Then \Cref{prop:prop1} follows by applying the Cauchy-Schwarz inequality.

\fi

\textbf{Challenges of the analysis.}
The analytical frameworks proposed by \citet{collins2021exploiting} and \citet{zhong2022feddar} for over-parameterized learning scenarios, as well as by \citet{chen2023fast} for low-rank matrix approximation, cannot handle the unique challenges that arise in the under-parameterized FRL framework, as elaborated below.

The first major challenge we encounter is to bound the gradient discrepancy on the update of $\Bv^t$, denoted as 
\( (\Bv^t\Wv^t - \Phiv)(\Wv^t)^\top - \sum_{i \in [M]} \frac{\Xv_i \Xv_i^\top}{N} \big(\Bv^t {w}_i^t - \phi_i\big)({w}_i^t)^\top. \)
Such difficulty is absent in the analyses in \citet{collins2021exploiting} and \citet{zhong2022feddar} because, in the over-parameterized regime and with a fixed number of samples per client per iteration, the error caused by the gradient discrepancy decays at a rate comparable to that of the representation estimation error. Therefore, the gradient discrepancy will gradually converge to zero. However, for the under-parameterized setting, even with the optimal $(\Bv^t, \Wv^t)$, i.e., when $\Bv^t\Wv^t = \Bv^*\Wv^*$, gradient discrepancy can still be non-zero, as the optimal representation cannot recover all local models, i.e., \(\Bv^*\Wv^*\neq \Phiv\). Instead, it only decreases when the number of samples \(N\) increases. This phenomenon indicates that an increase in the number of samples is essential to ensure one-step improvements of the estimated representations toward the ground-truth representation as the iteration progresses.

Another main challenge is ensuring that neither the gradient discrepancy nor noise-induced errors accumulate over iterations. This is critical as error accumulation can lead to significant deviation from the optimal solution, resulting in poor convergence and degraded model performance. To achieve this, we need to ensure the improvement of the estimation can dominate the effect of potential disturbances.

To tackle these new challenges, we first prove two concentration lemmas (\cref{lemma:concentration-of-Q} and \cref{lemma:concentration-of-tildeQ} in \cref{sec:aux-lemmas}) to ensure that the norm of the gradient discrepancy can be bounded when local datasets are sufficiently large. Next, to address the second challenge of avoiding the accumulation of gradient discrepancy and noise-induced errors over iterations, we develop iteration-dependent upper bounds for sample complexity (\cref{lemma:lemma10} and \cref{lemma:lemma12} in \cref{sec:proof-s3}). These bounds guarantee that the improvement in estimation, i.e., the 'distance' improvement between our estimated model and the optimal low-rank model, can mitigate potential disturbances caused by gradient discrepancy and noise. We establish this by introducing a novel approach to derive an accuracy-dependent upper bound for the per-client sample complexity, ensuring the error caused by the gradient discrepancy decays as fast as the {increase of the signal-to-noise ratio (SNR), formally introduced in \Cref{apdx:B}.}

\textbf{Main steps of the proof.} First, we transform the asymmetric matrix factorization problem into a symmetric problem by appropriately padding $\mathbf{0}$ columns or rows to $\Bv^t$ and $\Wv^t$ and constructing the updating matrices $\Thetav^t$ (see \Cref{apdx:B}). Our goal is then to prove that $\Thetav^t(\Thetav^t)^\top$ converges. We first show that, with a small random initialization, $\Thetav^t$ will enter a region containing the optima with high probability. Then, utilizing \cref{lemma:concentration-of-Q} and \cref{lemma:concentration-of-tildeQ}, we demonstrate that when $\Thetav^t$ enters the region $\Rc$, it will remain in this region with high probability despite gradient discrepancy and noise. Finally, utilizing \cref{lemma:lemma10} and \cref{lemma:lemma12}, we show that when $N$ is sufficiently large, $\Thetav^t(\Thetav^t)^\top$ converges at a linear rate with high probability under the influence of gradient discrepancy and noise, provided that the initialization condition satisfies $\Thetav^0 \in \Rc$.

\section{General \alg}\label{sec:gen-fedup}
In this section, we extend \alg designed for linear models to more general settings. 
Specifically, we use $\psi_{\Bv}$ to denote the representation, and assume linear local heads $f_{i}(z) = \Hv_i^\top z+b_i$, where $\Hv_i\in \Rb^{k\times m}, b_i\in \Rb^{m}$. This is motivated by the neural network architecture where all layers before the last layer are abstracted as the representation layer, and the last layer is linear.
Then, the objective function becomes
\begin{align}\label{obj:general-flute-obj}
\min_{\Bv,\{\Hv_i\},\{b_i\}}\frac{1}{M}\sum_{i\in[M]}\frac{1}{N}\sum_{(x,y)\in\Dc_i}
\ell\big(\Hv_i^\top \psi_{\Bv}(x)+b_i, y\big)+\lambda R(\{\Hv_i\}, \Bv),
\end{align}
where $R(\{\Hv_i\}, \Bv)$ is the regularization term to encourage the alignment of local models with the global optimum structure.

The general \alg algorithm for solving problem~\eqref{obj:general-flute-obj} is provided in Algorithm \ref{alg:general-flute} in \Cref{appx:general}. % \congc{describe the main differences and how the differences are tailored to nonlinear $f_{\Bv}$.}
 Given the non-linearity of {$\psi_{\Bv}$}, the penalty introduced in linear \alg is not directly applicable to the general problem. We thus formulate and design new penalty terms, following the same principles that motivated the design in the linear setting. This is to mitigate the local over-fitting induced by local updates and to encourage a structure benefit to global optimization. 
{As a concrete example, we present a design of the penalty term for the classification problem with CNN as a prediction model in \Cref{sec:real-world-datasets}.}

\section{Experimental Results}

\subsection{Synthetic Datasets}
We generate a  synthetic dataset as follows. First, we randomly generate $\phi_i$ according to a $d$-dimensional standard Gaussian distribution. For each $\phi_i$, we then randomly generate $N$ pairs of $(x,y)$, where $x$ is sampled from a standard Gaussian distribution, $\xi$ is sampled from a centered Gaussian distribution with variance $\sigma^2$, and $y = \phi_i^\top x+\xi$.  

In \Cref{fig:synthetic}, we compare \alg with FedRep~\citep{collins2021exploiting}. We measure the quality of the learned representation $\Bv^t$ and $\Wv^t$ over the metric $\frac{1}{M}\sum_{i\in[M]}\|\Bv^t w_i^t-\phi_i\|$.
We emphasize that FedRep requires {empirical covariance estimated from the local datasets} to be transmitted to the server for the initialization. Thus, it begins with a good estimate of the subspace spanned by $\Bv^*$. In contrast, \alg commences with a random initialization of both the representation and the heads. As a result, FedRep converges to a relatively small error within the few initial epochs, while \alg needs to go through more epochs to obtain a good estimate of the representation. However, as the learning progresses over more epochs, \alg eventually outperforms FedRep.
{To validate this hypothesis, we introduce FedRep(RI) in our experiments, which has the same initialization as \alg but is otherwise identical to FedRep. We see from \Cref{fig:synthetic} that when FedRep is randomly initialized, \alg outperforms FedRep(RI) in much fewer iterations.} 

We also observe that the performance gain of \alg is more pronounced in highly under-parameterized scenarios, i.e., where $k$ is relatively small. 
As $k$ increases, the gap between the convergence rates of \alg and FedRep narrows. 
These results demonstrate that \alg achieves better performance in the under-parameterized regime. In the additional experimental results included in \cref{apdx:experiment}, we also observe that when the number of participating clients $M$ increases, the average error of the model learned from \alg decreases, which is consistent with \Cref{prop:prop1}.

\begin{figure}[t]
\begin{subfigure}{\linewidth}
 \centering
    \hspace{-0.05\linewidth}
     \includegraphics[width=0.6\textwidth]{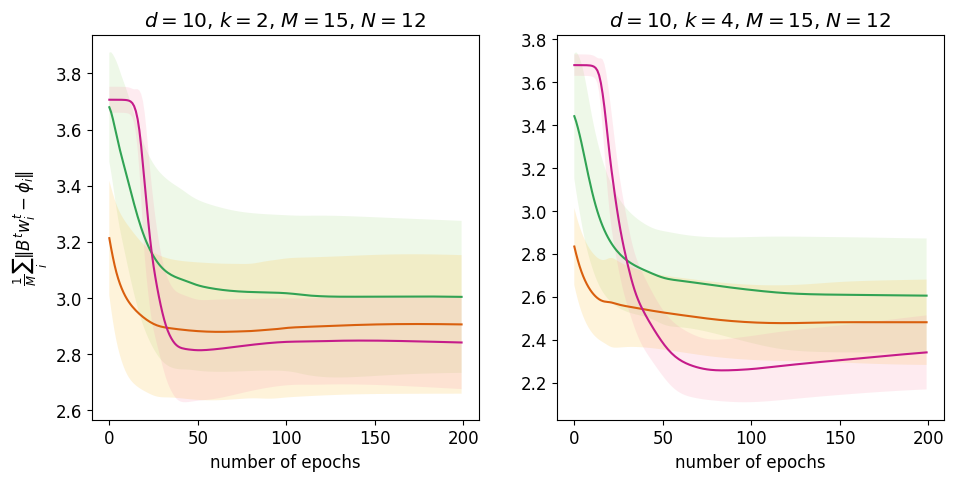}
 \end{subfigure}
 \begin{subfigure}{\linewidth}
 \centering
 \hspace{-0.05\linewidth}
     \includegraphics[width=0.6\textwidth]{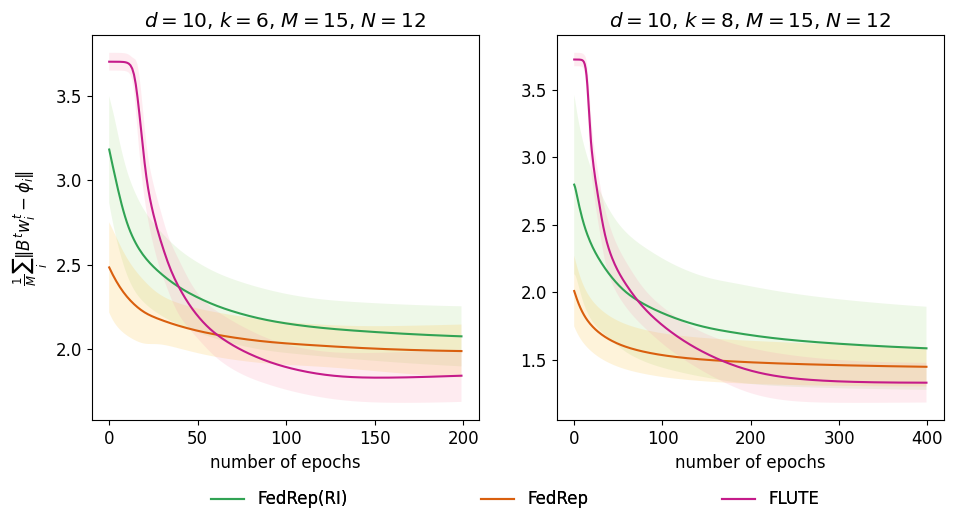}
 \end{subfigure}
 \vspace{-0.2in}
 \caption{Experimental results with synthetic datasets.}
 \label{fig:synthetic}
\end{figure}

\subsection{Real World Datasets}
\label{sec:real-world-datasets}
\textbf{Datasets and models.}
We now evaluate the performance of general \alg on multi-class classification tasks with real-world datasets CIFAR-10 and CIFAR-100~\citep{krizhevsky2009learning}. 
For all experiments, we adopt a convolutional neural network (CNN) with two convolution layers, two fully connected layers with ReLU activation, and a final fully connected layer with a softmax activation function. A detailed description of the CNN structure is deferred to \Cref{apdx:experiment} of the Appendix.

\textbf{Algorithms for comparison.} 
We compare \alg with several baseline algorithms, including FedAvg \citep{mcmahan2017communication}, Fed-LG \citep{liang2020think}, FedPer \citep{arivazhagan2019federated}, FedRep \citep{collins2021exploiting}, FedRod \citep{chen2021bridging} and FedCP \citep{zhang2023fedcp}. Fed-LG is designed to learn a common head shared across clients while allowing for localized representations, while FedPer and FedRep both assume shared representation and personalized local heads. FedRod extends the model considered in FedRep by adding another head layer into the local model, and FedCP further equips a conditional policy network into the local model. 
We also consider variants of \alg and FedRep, denoted as \algstar and FedRep*, respectively, under which we vary the number of updates of the local heads in each communication round, as elaborated later.

\textbf{Loss function and penalty.} 
For algorithms other than \alg and \algstar, the local loss function is chosen as
 $\Lc_i = \frac{1}{N}\sum_{(x,y)\in\Dc_i}\Lc_{\text{CE}}(\Hv_i^\top \psi_\Bv(x)+b_i,y)$,
where $\Lc_{\text{CE}}$ is the cross entropy loss.  
The local loss function for {\alg and \algstar} are specialized as
\begin{align}
&\Lc_i(\Bv,b,\Hv) 
    =
\frac{1}{N}\sum_{(x,y)\in\Dc_i}
\Lc_{\text{CE}}\big(\Hv_i^\top \psi_{\Bv}(x)+b_i, y\big)+\quad\lambda_1\|\psi_{\Bv}(x)\|^2+\lambda_2\|\Hv_i\|_F^2+\lambda_3\Nc\Cc_i(\Hv_i),
\end{align}
where $y\in\Rb^m$ is a one-hot vector whose $k$-th entry is $1$ if the corresponding $x$ belongs to class $k$, $\lambda_1$, $\lambda_2$ and $\lambda_3$ are non-negative regularization parameters. $\Nc\Cc_i(\Hv_i)$ is motivated by \citet{Papyan_2020} and set as
\[
\Nc\Cc_i(\Hv_i)=\norm{\frac{\Hv^\top_i\Hv_i}{\|\Hv^\top_i\Hv_i\|_F} -
    \frac{1}{\sqrt{m-1}}\mathbf{u}_i\mathbf{u}_i^\top\odot\mleft(\Iv_{m}-\frac{1}{m}\mathbf{1}_m\mathbf{1}^\top_{m}\mright)}_F,
\]
where $\mathbf{u}_i$ is an $m$-dimensional one-hot vector whose $c$-th entry is $1$ if $c\in\Cc_i$, and $\odot $ denotes the Hadamard product. We specialize the regularization term optimized on the server side as $R(\{\Hv_i\}) = \sum_{i}\Nc\Cc_i(\Hv_i)$. Note that for general \alg specified to a classification problem, we penalize $\|\psi_\Bv(x)\|$ instead of directly penalizing the parameter $\Bv$. Since $\|\psi_\Bv(x)\|^2$ depends on data, the regularization term is optimized partially on the client side and partially on the server side. 

Compared with the objective function in \cref{def:AMA-emp} for the linear case, the term $\lambda_3\mathcal{N}\mathcal{C}_i(\mathbf{H}_i)$ replaces term (I) and $\lambda_1\|\psi_{\mathbf{B}}(x)\|^2 + \lambda_2\|\mathbf{H}_i\|_F^2$ replaces term (II). The primary goal of introducing $\lambda_3\mathcal{N}\mathcal{C}_i(\mathbf{H}_i)$ is to mitigate local over-fitting that occurs during local updates in the training process. As elaborated in \Cref{appx:nc}, such a regularization term promotes a beneficial structure for the global model, facilitating efficient learning performance. This term shares the same motivation as the term (I) in the linear scenario, which focuses on distilling significant components from the model to mitigate local over-fitting effects. For term (II), we only replace $\|\mathbf{B}^\top\mathbf{B}\|_F^2$ with $\|\psi_{\mathbf{B}}(x)\|^2$, since the representation is not linear in general. 

\begin{table*}[t]%\label{tab:exp}
\footnotesize
\centering
 \vspace{-0.15in}
\caption{Average test accuracy on CIFAR-10 and CIFAR-100.}
\centering 
\hspace*{-0.5cm}
\begin{tabular}{c|Z|Z|Z|Z|Z|Z|Z|Z}
\midrule
\midrule
Dataset & \multicolumn{4}{|c|}{CIFAR-10}&\multicolumn{4}{c}{CIFAR-100}\\
\midrule
% & \multicolumn{4}{|c|}{CIFAR-10} \\
% \midrule
% \multicolumn{2}{c}{(\# clients$\times$ \# tasks/client)} 
Partition & 50 $\times$ 2 & 50 $\times$ 5 & 100 $\times$ 2 & 100 $\times$ 5 & 100 $\times$ 5 & 100 $\times$ 10 & 100 $\times$ 20 & 100 $\times$ 40\\
\midrule     
FedAvg & 34.460{\tiny $\pm$1.083} & 47.217{\tiny $\pm$0.395} & 41.584{\tiny $\pm$0.433} & 51.876{\tiny $\pm$0.675} & 20.212{\tiny $\pm$0.574} & 31.533{\tiny $\pm$0.519} & 34.659{\tiny $\pm$0.482} & 32.902{\tiny $\pm$0.195}\\
FedAvg-FT & 83.996{\tiny $\pm$0.948} & 71.465{\tiny $\pm$0.701} & 84.688{\tiny $\pm$0.437} & 70.884{\tiny $\pm$0.697} & 78.342{\tiny $\pm$0.574} & 66.660{\tiny $\pm$0.370} & 54.464{\tiny $\pm$0.178} & 44.858{\tiny $\pm$0.119}\\
Fed-LG & 82.724{\tiny $\pm$2.137} & 61.820{\tiny $\pm$0.409} & 83.019{\tiny $\pm$0.431 }& 62.957{\tiny $\pm$0.895} & 72.526{\tiny $\pm$0.692}   & 53.526{\tiny $\pm$0.151}   & 34.445{\tiny $\pm$0.375}   & 22.702{\tiny $\pm$0.315} \\
FedPer   & 85.173{\tiny $\pm$1.082}            & 74.015{\tiny $\pm$0.724}            & 86.168{\tiny $\pm$0.703}   & 73.666{\tiny $\pm$0.281}     & 76.001{\tiny $\pm$0.454}   & 67.100{\tiny $\pm$0.229}   & 56.066{\tiny $\pm$0.389}   & 44.689{\tiny $\pm$0.411} \\
FedRep   & 86.133{\tiny $\pm$0.775}            & 71.737{\tiny $\pm$0.296}            & 86.685{\tiny $\pm$0.766}   & 73.808{\tiny $\pm$0.561}    & 78.621{\tiny $\pm$0.159}   & 68.530{\tiny $\pm$0.255}   & 56.360{\tiny $\pm$0.245}   & 43.061{\tiny $\pm$0.476}    \\
FedRep*   & 87.320{\tiny $\pm$1.485}            & 75.766{\tiny $\pm$0.220}            & 87.177{\tiny $\pm$0.489}   & 75.296{\tiny $\pm$0.505}    & 78.892{\tiny $\pm$0.410}   & 68.630{\tiny $\pm$0.705}   & 56.654{\tiny $\pm$0.609}   & 42.025{\tiny $\pm$0.404}      \\
FedRoD & 79.476{\tiny $\pm$2.966} & 68.728{\tiny $\pm$1.750} & 83.296{\tiny $\pm$1.545} & 72.116{\tiny $\pm$0.788} & 74.299{\tiny $\pm$0.338} & 66.462{\tiny $\pm$0.284} & 57.280{\tiny $\pm$0.105} & 48.120{\tiny $\pm$0.186}\\
FedCP & 85.361{\tiny $\pm$1.605} & 71.603{\tiny $\pm$0.885} & 84.798{\tiny $\pm$0.489} & 71.344{\tiny $\pm$0.587} & 74.266{\tiny $\pm$0.559} & 66.426{\tiny $\pm$0.372} & 57.067{\tiny $\pm$0.483} & 43.638{\tiny $\pm$0.415}\\
\rowcolor{clr}
\alg    & 87.012{\tiny $\pm$0.453}     & 76.478{\tiny $\pm$0.484}     & 86.128{\tiny $\pm$1.007}   & 76.918{\tiny $\pm$0.712}    & 77.750{\tiny $\pm$0.615}   & 70.598{\tiny $\pm$0.282}   & 59.243{\tiny $\pm$0.334} &   48.169{\tiny $\pm$0.597}   \\
\rowcolor{clr}
\algstar     & 
{87.713{\tiny $\pm$1.365}}
   & {76.543{\tiny $\pm$0.921}}   & {88.567{\tiny $\pm$0.457}}   & {78.255{\tiny $\pm$0.688}}   & {79.560{\tiny $\pm$0.627}}   & {70.844{\tiny $\pm$0.419}}   & {59.714{\tiny $\pm$0.448}}   & {48.170{\tiny $\pm$0.440}}  \\
\midrule
\end{tabular}
\label{tab:table-cifar}
\vspace{-0.3cm}
\end{table*}

\textbf{Implementation and evaluation.}
{We use $m$ to denote the number of classes assigned to each client.}
For CIFAR-10 dataset, we consider four $(N,m)$ pairs: $(50,2)$, $(50,5)$, $(100,2)$ and $(100,5)$; For CIFAR-100 dataset, we consider four $(N,m)$ pairs: $(100,5)$, $(100,10)$, $(100,20)$ and $(100,40)$.

For experiments conducted on the CIFAR-10 dataset, all algorithms are executed over 100 communication rounds. For LG-Fed, FedPer, FedRoD, FedCP and \alg, each client performs one round of local updates in each communication round. 
FedRep performs one epoch of local head update and an additional epoch for the local representation update. 
Compared with FedRep, FedRep* processes 10 epochs to update its local heads and one epoch to update its representation.
For comparison, {\algstar} also runs 11 rounds of local updates, updating both representation and local head in the first round, followed by 10 rounds of only updating the local head.

The experiments on the CIFAR-100 dataset also use 100 communication rounds. The number of local updates for LG-Fed, FedPer, FedRoD, FedCP and \alg are set to 5. FedRep is configured to update the local representation and head for 5 epochs each, while FedRep* allocates 5 epochs for updating the local representation and 10 for updating the local head. \algstar runs 15 epochs of local updates, where the initial 5 epochs update both the representation and local head while the subsequent 10 epochs solely update the local head. 

\textbf{Averaged performance.} The results are reported in \Cref{tab:table-cifar}. It is evident that \alg and \algstar consistently outperform other baseline algorithms in all experiments conducted on CIFAR-10 and CIFAR-100 datasets. This superior performance is attributed to the tailored design that encourages the locally learned models to move towards a global optimal solution rather than a local optimum. We also observe that the gain of \alg and \algstar becomes more prominent with larger $N$ and $m$. Intuitively, larger $N$ and $m$ implies more severe under-parameterization for the given CNN model, and our algorithms exhibit more advantage for such cases. 

\section{Conclusion}
{
To the best of our knowledge, this paper represents the first effort in the study of federated representation learning in the under-parameterized regime, which is of great practical importance. We have proposed a novel FRL algorithm \alg that was inspired by asymmetric low-rank matrix approximation. \alg incorporates a novel regularization term in the loss function and solves the corresponding ERM problem in a federated manner. We proved the convergence of \alg and established the per-client sample complexity that is comparable to the over-parameterized result but with very different proof techniques. We also extended \alg to general (non-linear) settings which are of practical interest. \alg demonstrated superior performance over existing FRL solutions in both synthetic and real-world tasks, highlighting its advantages for efficient learning in the under-parameterized regime.
}

\section*{Acknowledgements}
The work of R. Liu and J. Yang was supported in part by the U.S. National Science Foundation under grants CNS-1956276,  CNS-2003131 and CNS-2114542. The work of C. Shen was supported in part by the U.S. National Science Foundation under grants ECCS-2332060, CPS-2313110, ECCS-2143559, and ECCS-2033671.

\bibliography{FRL_in_the_Under-Param_Regime}
\bibliographystyle{icml2024}

%----------------------------------------------------------
\onecolumn

\appendix
\allowdisplaybreaks

\textbf{Notations.} 
Throughout this paper, bold capital letters (e.g., $\Xv$) denote matrices, and calligraphic capital letters (e.g., $\Cc$) denote sets. We use $\text{tr}(\mathbf{X})$ to denote the trace of matrix $\Xv$, $\sigma_{\min}(\Xv)$ and $\sigma_{\max}(\Xv)$ to denote the minimum and maximum singular values of $\Xv$, respectively, and $\diag (x_1,\cdots,x_d)$ to denote a $d$-dimensional diagonal matrix with diagonal entries $x_1,\cdots,x_d$. $|\Cc|$ denotes the cardinality of set $\Cc$, 
and $\{X_i\}_{i\in[N]}$ denotes the set $\{X_1, \cdots, X_N\}$.
We use $\langle x,y\rangle$ to denote the inner product of $x$ and $y$, and $\|x\|$ to denote the Euclidean norm of vector $x$. 
We use $f \circ \psi$ to denote the composition of functions $f:\Rb^{k}\rightarrow \Rb^{m}$ and $\psi:\Rb^{d}\rightarrow \Rb^k$, i.e., $(f \circ \psi)(x)=f(\psi(x))$. $a\lesssim b$ indicates $a\leq Cb$ for a positive constant $C$.
$\mathbf{I}_d$ represents a $d\times d$ identity matrix, and $\mathbf{0}$ is a $d$-dimensional all-zero vector.

Denote $\dh := \max\{d,M\}$, $\dl:=\min\{d,M\}$, and $\Phiv_\star\in\Rb^{\dh\times \dh}$ as the matrix constructed from $\Phiv\in \Rb^{d\times M}$ by padding all-zero columns or rows. Define its SVD as $\Phiv_\star = \Uv_\star\Lambdav_\star\Vv^\top_\star$. Denote
$\tilde{\Lambdav} = \diag(2\Lambdav_\star, -2\Lambdav_\star)$ and
let $\lambda_1^* \geq \lambda_2^* \geq \cdots \geq \lambda_{2\dh}^*$ be the eigenvalues of $\tilde{\Lambdav}$, with $\Delta = \lambda_k^* - \lambda_{k+1}^*$. Note that the definition of $\Delta$ is consistent with the definition in \Cref{sec:section5}. For clarity of presentation, we use $\sigma_{\xi}$ to denote the standard deviation of the noise $\xi$ instead of $\sigma$ that is used in the main paper.

\section{Analysis of the \alg Linear Algorithm}\label{apdx:B}

\subsection{Preliminaries}\label{apdx:B-prelim}

We start with the updating rule of $\Bv^t$ and $\Wv^t$ in {Algorithm \ref{alg:fedUP}}. 

For $\Bv^t$, from \alg we have the following updating rule:
\begin{align*}
    {\Bv}^{t+1} 
    &= 
    \Bv^{t}-\frac{{\eta}}{N}\sum_{i\in[M]}\sum_{j\in[N]}x_{i,j}\big(x_{i,j}^\top\mathbf{B}^{t}{w}_i^{t}-y_{i,j}\big)({w}_i^{t})^\top-\frac{\eta}{2}\Bv^{t}\big((\Bv^{t})^\top\Bv^{t}-\Wv^{t}(\Wv^{t})^\top\big)\\
    &=
    \Bv^{t}-\frac{{\eta}}{N}\sum_{i\in[M]}\Xv_i\Xv_i^\top\big(\mathbf{B}^{t} {w}_i^{t}-\phi_i\big)({w}_i^{t})^\top-\frac{\eta}{2}\Bv^{t}\big((\Bv^{t})^\top\Bv^{t}-\Wv^{t}(\Wv^{t})^\top\big).
\end{align*}
Since data points $\{x_{i,j}\}$ are sampled from a standard Gaussian distribution, for large $N$, it holds that $\Xv_i\Xv_i^\top/N\approx \Iv$. Then, we introduce the following definition:
\begin{align}\label{def:Q}
    \Qv^{t+1} 
    \coloneqq 
    \eta\sum_{i\in[M]}\big(\mathbf{B}^{t} {w}_i^{t}-\phi_i\big)({w}_i^{t})^\top
    -
    \eta\sum_{i\in[M]}\frac{\Xv_i\Xv_i^\top}{N}\big(\mathbf{B}^{t} {w}_i^{t}-\phi_i\big)({w}_i^{t})^\top
    +
    \eta\sum_{i\in[M]}\frac{\Xv_{i}\Ev_{i}({w}_i^{t})^\top}{N}
    .
\end{align}
With this definition, the updating rule of $\Bv^t$ can be rewritten as 
\begin{align*}
    {\Bv}^{t+1} 
    =
     \Bv^{t}-\eta\big(\Bv^{t}\Wv^{t}-\Phiv\big)(\Wv^{t})^\top-\frac{\eta}{2}\Bv^{t}\big((\Bv^{t})^\top\Bv^{t}-\Wv^{t}(\Wv^{t})^\top\big)+\Qv^{t+1}.
\end{align*}
Now, we consider the updating rule of $\Wv$. Observe that each of its columns satisfies
\begin{align*}
    w_i^{t+1}
    &= 
    w_i^{t}-\frac{\eta}{N}\sum_{j\in[N]}(\Bv^{t})^\top x_{i,j}\big((x_{i,j})^\top\mathbf{B}^{t}{w}_i^{t}-y_{i,j}\big)\\
    &=w_i^{t}-\frac{\eta}{N}(\Bv^{t})^\top\Xv_i\Xv_i^\top\big(\mathbf{B}^{t} {w}_i^{t}-\phi_i\big).
\end{align*}
We define $\tilde{\Qv}^{t+1} \coloneqq [\tilde{q}_1^{t+1},\cdots,\tilde{q}_M^{t+1}]$, where each of its columns is given by
\begin{align}\label{def:tildeQ}
    \tilde{q}_i^{t+1}
    \coloneqq
    \eta(\Bv^{t})^\top\big(\mathbf{B}^{t} w_i^{t}-\phi_i\big)
    -
    \frac{\eta}{N}(\Bv^{t})^\top\Xv_i\Xv_i^\top\big(\mathbf{B}^{t} w_i^{t}-\phi_i\big)
    +
    \frac{\eta}{N}(\Bv^t)^\top\Xv_{i}\Ev_{i}
    .
\end{align}
Then, $\Wv^t$ is updated according to
\begin{align*}
    \Wv^{t+1}=\Wv^{t}-\eta(\Bv^{t})^\top(\mathbf{B}^{t} \Wv^{t}-\Phiv)+\frac{\eta}{2}\big((\Bv^{t})^\top\Bv^{t}-\Wv^{t}(\Wv^{t})^\top\big)\Wv^{t}+\tilde{\Qv}^{t+1}.
\end{align*}
Recall the SVD of $\Phiv$ is denoted as $\Phiv=\Uv\Lambdav\Vv^\top$. Further denote $\tilde{\Bv}^t=\Uv^\top\Bv^t$ and $\tilde{\Wv}^t=\Wv^t\Vv$. Then, we have 
    \begin{align*}
        \tilde{\Bv}^{t+1}
       & =
    \tilde{\Bv}^{t}-\eta\big(\tilde{\Bv}^{t}\tilde{\Wv}^{t}-\Lambdav\big)\big(\tilde{\Wv}^{t}\big)^\top-\frac{\eta}{2}\tilde{\Bv}^{t}\big((\tilde{\Bv}^{t})^\top\tilde{\Bv}^{t}-\tilde{\Wv}^{t-1}(\tilde{\Wv}^{t})^\top\big)+\Uv^\top\Qv^{t+1},\\
    \tilde{\Wv}^{t+1}
    &=
    \tilde{\Wv}^{t}-\eta(\tilde{\Bv}^{t})^\top\big(\tilde{\mathbf{B}}^{t} \tilde{\Wv}^{t}-\Lambdav\big)-\frac{\eta}{2}\big((\tilde{\Bv}^{t})^\top\tilde{\Bv}^{t}-\tilde{\Wv}^{t}(\tilde{\Wv}^{t})^\top\big)\tilde{\Wv}^{t}+\tilde{\Qv}^{t+1}\Vv.
    \end{align*}
Similar to the definition of $\Phiv_\star$, we construct $\Bv^t_\star\in\Rb^{\dh\times k}$ and $\Wv^t_\star\in\Rb^{k\times \dh}$ by padding all-zero columns or rows to $\tilde{\Bv}^t$ and $\tilde{\Wv}$, respectively. Similarly, we obtain $\Qv^t_\star\in\mathbb{R}^{\dh\times k}$ and $\tilde{\Qv}^t_\star\in\mathbb{R}^{k\times \dh}$ by padding all-zero columns or rows to $\Qv^t$ and $\tilde{\Qv}^t$, respectively. Then, we define $\Thetav^t$ and $\Rv^{t}$ as
\begin{align*}
    \Thetav^t 
        &=
        \bigg[\frac{(\Bv^t_\star)^\top+\Wv_\star^t}{\sqrt{2}}
        \quad
        \frac{(\Bv^t_\star)^\top-\Wv^t_\star}{\sqrt{2}}\bigg]^\top,\\    
        \Rv^{t}&=\bigg[\frac{(\Qv^{t}_\star)^\top\Uv_\star+\tilde{\Qv}^{t}_\star\Vv_\star}{\sqrt{2}}\quad
\frac{(\Qv^{t}_\star)^\top\Uv_\star-\tilde{\Qv}^{t}_\star \Vv_\star}{\sqrt{2}}\bigg]^\top.
\end{align*}
Then, the updating rule of $\Thetav^t$ can be described as
\begin{align}
    \Thetav^{t+1} = \Thetav^{t}+\frac{\eta}{2}\tilde{\Lambdav}\Thetav^{t}-\frac{\eta}{2}\Thetav^{t}(\Thetav^{t})^\top\Thetav^{t}
    +\Rv^{t+1}.\label{equ:44}
\end{align}
Let $\Thetav^t = [(\Thetav^t_{k})^\top(\Thetav^t_{\text{res}})^\top]^\top$ and $\Rv^t = [(\Rv^t_{k})^\top\quad(\Rv^t_{2\dh-k})^\top]^\top$
where $\Thetav^t_{k}\in\Rb^{k\times k}$, $\Thetav^t_{\text{res}}\in\Rb^{(2\dh-k)\times k}$,
$\Rv^t_{k}\in\Rb^{k\times k}$ and $\Rv^t_{2\dh-k}\in\Rb^{(2\dh-k)\times k}$.
Then, we decompose the updating rule of $\Thetav^t$ as
\begin{align}
    \Thetav^t_{k} &= \Thetav^{t-1}_{k}+\frac{\eta}{2}\tilde{\Lambdav}_k\Thetav^{t-1}_{k}-\frac{\eta}{2}\Thetav^{t-1}_{k}(\Thetav^{t-1})^\top\Thetav^{t-1}+\Rv^{t}_k,\label{eqn:Theta-k}\\
    \Thetav^t_{\text{res}} &= \Thetav^{t-1}_{\text{res}}+\frac{\eta}{2}\tilde{\Lambdav}_{\text{res}}\Thetav^{t-1}_{\text{res}}-\frac{\eta}{2}\Thetav^{t-1}_{\text{res}}(\Thetav^{t-1})^\top\Thetav^{t-1}+\Rv^{t}_{2d-k}.\label{eqn:Theta-res}
\end{align}
%-----------------------------------------------------------------------

\subsection{Proof of \Cref{prop:prop1}}

First, we restate \Cref{prop:prop1} as follows.
\begin{theorem}[Restatement of \Cref{prop:prop1}]
        Set $\lambda$ and $\eta$ as in \Cref{thm:thm1}. 
        Then for constant $T_\Rc$ and any $T>T_\Rc$, 
        there exist positive constants $c_1$ and $c_2$ such that when the number of samples per client satisfies $N\geq c_1\frac{(\dh-\log\delta+\log T)(k\sqrt{(\lambda_1^*)^2 +E}+\sqrt{k}\sigma_{\xi})^2}{\kappa_T^2\Delta^2}$, for all $T_\Rc<t\leq T$
        we have
            \begin{align}
        \frac{1}{M}\sum_{i\in[M]}\|\Bv^t w^t_i-\Bv^* w_i^*\|
        \leq
        \frac{c_2\kappa\sqrt{k}}{\sqrt{M}}\Big(1-\frac{\eta\Delta}{16}\Big)^t,\nonumber
    \end{align}
 with probability at least $1-\delta$, where $\kappa_T = (1-\frac{\eta\Delta}{16})^T$.
\end{theorem}

\textbf{Overview of the proof.} The proof of \Cref{prop:prop1} consists of three main steps.  
\begin{itemize}[topsep=0pt, leftmargin=10pt] 
\item \textbf{Step 1:} We show that with a small random initialization, $\Thetav^t$ will enter a region containing the optima with high probability (see \Cref{sec:proof-s1}).
\item \textbf{Step 2:} We show that once $\Thetav^t$ enters this region, it will stay in it with high probability (see \Cref{sec:proof-s2}). 
\item \textbf{Step 3:} We show that when $N$ is sufficiently large, with high probability it holds that $\|\Thetav^t(\Thetav^t)^\top-\diag(\tilde{\Lambdav}_k,\mathbf{0})\|$ converges to $0$ at a linear rate when the initialization satisfies $\Thetav^0\in\Rc$ (see \Cref{sec:proof-s3}).
\end{itemize}
We then put pieces together and prove \Cref{prop:prop1} in \Cref{sec:proof-prop1}. We introduce some auxiliary lemmas in \Cref{sec:aux-lemmas}.

\subsubsection{Step 1: Entering a Region with Small Random Initialization}\label{sec:proof-s1}

We first introduce the following definitions, adapted from the proof in~\citet{chen2023fast}.
Recall that $\Delta = \lambda_k^*-\lambda_{k+1}^*$. We define
\begin{align*}
    \Rc
   & =
    \bigg\{\Thetav^t=
    \begin{bmatrix}
        \Thetav^t_{k}\\
        \Thetav^t_{\text{res}}
    \end{bmatrix}\in\Rb^{2d\times k}\bigg|\sigma_1^2(\Thetav^t)\leq2\lambda_1^*,\quad
    \sigma_1^2(\Thetav^t_{\text{res}})\leq\lambda^*_k-\Delta/2,\quad
    \sigma_k^2(\Thetav^t_k)\geq\Delta/4\bigg\},\\
    \Rc_s
    & =
    \bigg\{\Thetav^t=
    \begin{bmatrix}
        \Thetav^t_{k}\\
        \Thetav^t_{\text{res}}
    \end{bmatrix}\in\Rb^{2d\times k}\bigg|\sigma_1^2(\Thetav^t)\leq2\lambda_1^*,\quad
    \sigma_1^2(\Thetav^t_{\text{res}})\leq\lambda^*_k-\Delta/2\bigg\}.
\end{align*}

Then, we establish the following proposition. 

\begin{proposition}
\label{thm:thm2-formal}
    Assume $\eta\leq\frac{1}{6\lambda_1^*}$ and all entries of $\Bv^0$ and $\Wv^0$ are independently sampled from $\Nc(0,\alpha^2)$ with a sufficiently small $\alpha$. Then, if 
    \begin{align*}
        \sqrt{N}\geq\max\Bigg\{\frac{\sqrt{\dh-\log\delta}}{\sqrt{c_1}}, \frac{3456\sqrt{\dh-\log\delta}\sqrt{\lambda_1^*}(\sqrt{k(\lambda_1^*)^2 +E}+\sqrt{k}\sigma_{\xi})}{\min\{\sigma_1(\Thetav_{\text{res}}^0),\sigma_1(\Thetav_{\text{res}}^0)\}\Delta\sqrt{c_1}}\Bigg\},
    \end{align*}
with probability at least $1-ct\delta$ for some constant $c>0$, $\Thetav^t$ will enter region $\Rc$ for some $t\in [T_\Rc]$, where
\begin{equation}\label{def:TR}
    \begin{aligned}
    T_\Rc = \frac{\log(\Delta/{(4\sigma_k^2(\Thetav_k^0))})}{2\log(1+\frac{\eta}{2}(\lambda_k^*-\Delta/2))}.
    \end{aligned}
\end{equation}
\end{proposition}

The proof of \Cref{thm:thm2-formal} relies on \Cref{lemma:lemma13} and \Cref{lemma:lemma14}, which will be introduced shortly.  
Before that, we state the following claim introduced in \citet{chen2023fast}:

%---------------------------------------------------------------------

\begin{claim}\label{cond:condition1}
 $\sigma_1^2(\Thetav^0)\leq \lambda_1^*$, $\sigma_1^2(\Thetav_{\text{res}}^0)\leq\lambda_k^*-\Delta/2$, $ \sigma_k^2(\Thetav_k^0)\leq\Delta/4$ and
 %   \begin{align}
    $    \sigma_1^2(\Thetav_{\text{res}}^0)
        \leq
        c_1\cdot
        \sigma_k(\Thetav_k^0)^{1+\kappa},$ 
%    \end{align}
where $c_1 = \frac{\Delta^{1-\kappa/2}}{2^{3-\kappa}}$ and $\kappa = \frac{\log(1+\frac{\eta}{2}\lambda_{k+1}^*+\frac{\eta}{8}\Delta)}{\log(1+\frac{\eta}{2}(\lambda_k^*-\Delta/2))}<1$.
\end{claim}
The following lemma shows that with a small random initialization, \Cref{cond:condition1} holds with high probability.
\begin{lemma}\label{lemma:lemma13}
Assume all entries of $\Bv^0$ and $\Wv^0$ are independently sampled from $\Nc(0,\alpha^2)$. Then, for any $\delta\in[0,1]$, if $\alpha$ is sufficiently small, \Cref{cond:condition1} holds with probability at least $1-\delta$.
\end{lemma}

\begin{proof}[Proof of \Cref{lemma:lemma13}]
Using \Cref{lemma:lemma16}, we have $\sigma_1(\Thetav^0)\leq\sqrt{\lambda_1^*}$ and $\sigma_1(\Thetav_{\text{res}}^0)\leq\sqrt{\lambda_k^*-\Delta/2}$ hold
with probability at least $1-2\exp(-(\frac{1}{\alpha^2}\sqrt{\lambda_k^*-\Delta/2}-2\sqrt{d})^2/2)$, 
and $\sigma_k(\Thetav_k^0)\leq\sqrt{\Delta}/2$ holds with probability at least $1-2\exp(-(\sqrt{\Delta}/(2\alpha^2)-2\sqrt{d})^2/2)$. Then for $\alpha$ small enough such that 
\begin{align*}
    \alpha\leq\min\Big\{\frac{\sqrt{\Delta}}{4\sqrt{d}+2\sqrt{2\log(2/\delta')}},\;
    \frac{\sqrt{\lambda_1^*-\Delta/2}}{2\sqrt{d}+\sqrt{2\log(2/\delta')}}
    \Big\},
\end{align*}
$\sigma_1(\Thetav^0)\leq\sqrt{\lambda_1^*}$, $\sigma_1(\Thetav_{\text{res}}^0)\leq\sqrt{\lambda_k^*-\Delta/2}$ and $\sigma_k(\Thetav_k^0)\leq\sqrt{\Delta}/2$ hold with probability at least $1-2\delta'$ for any $\delta'\in[0,1]$.

From \citet{rudelson2008littlewoodofford}, there exists
a constant $K$ that only depends on $\delta'$ such that with probability at least $1-\delta'$, we have $\sigma_k(\Thetav_k^0)\geq {\alpha^2}{K\sqrt{k}}$. 

Thus, when $\alpha$ is sufficiently small such that $\alpha^{4-2(1+\kappa)}\leq \frac{c_1(K\sqrt{k})^{1+\kappa}}{8d+4\log(2/\delta')}$, with probability at least $1-\delta'$, we have
\begin{align}
    c_1\sigma_k(\Thetav_k^0)^{1+\kappa}\geq c_1(\alpha^2 K\sqrt{k})^{1+\kappa}\geq \alpha^4(2\sqrt{d}+\sqrt{2\log(2/\delta')})^2.\nonumber
\end{align}
Note that from \Cref{lemma:lemma16}, with probability at least $1-\delta'$ we have
\begin{align*}
    \sigma_1^2(\Thetav_{\text{res}}^0)\leq
    \alpha^4(2\sqrt{d}+\sqrt{2\log(2/{\delta'})})^2.
\end{align*}
Then we conclude that with probability at least $1-4\delta'$, we have $\sigma_1(\Thetav^0)\leq\sqrt{\lambda_1^*}$, $\sigma_1(\Thetav^0_{\text{res}})\leq\sqrt{\lambda_1^*-\Delta/2}$, $\sigma_{k}(\Thetav_k^0)\leq\sqrt{\Delta}/2$ and $    \sigma_1^2(\Thetav_{\text{res}}^0)
        \leq
        c_1\cdot
        \sigma_k(\Thetav_k^0)^{1+\kappa}$.
Finally, the lemma follows by setting $\delta = 4\delta'$.
\end{proof}

%---------------------------------------------------------------------
Next, we introduce the following lemma, which shows that when \Cref{cond:condition1} holds, $\Thetav^t$ will enter the region $\Rc$ in a short time period.
\begin{lemma}\label{lemma:lemma14}
    Assume $\eta\leq\frac{1}{6\lambda_1^*}$ and \Cref{cond:condition1} holds. Then, if 
\begin{equation}
    \begin{aligned}
        \sqrt{N}\geq\max\Big\{\frac{\sqrt{\dh-\log\delta}}{\sqrt{c_1}}, \frac{3456  \sqrt{\dh-\log\delta} \sqrt{\lambda_1^*}(\sqrt{k(\lambda_1^*)^2 +E}+\sqrt{k}\sigma_{\xi})}{\min\{\sigma_1(\Thetav_{\text{res}}^0),\sigma_1(\Thetav_{\text{res}}^0)\}\Delta\sqrt{c_1}}\Big\},\label{ineq:ineq19}
    \end{aligned}
\end{equation}
with probability at least $1-ct\delta$, we have $\sigma_k(\Thetav^{t}_{k}) \geq\sqrt{\Delta}/2$ for some $t\in[0,\frac{\log(\Delta/{4\sigma_k^2(\Thetav_k^0)})}{2\log(1+\frac{\eta}{2}(\lambda_k^*-\Delta/2))}]$.
\end{lemma}

\begin{proof}[Proof of \Cref{lemma:lemma14}]

With \Cref{cond:condition1} holds, we have $\sigma_1^2(\Thetav^0)\leq2\lambda_1^*$, $
    \sigma_1^2(\Thetav_{\text{res}}^0)\leq\lambda^*_k-\Delta/2$. Then, 
    based on \Cref{lemma:second-upb-for-res}, for $N$ satisfying inequality \eqref{ineq:ineq19}, we have
   \begin{align}
    \sigma_1(\Thetav^t_{\text{res}})
       \leq
       \Big(1+\frac{\eta}{2}\lambda^*_{k+1}+\frac{\eta}{8}\Delta\Big)^t\sigma_1(\Thetav^{0}_{\text{res}}),\nonumber
\end{align}
holds with probability at least $1-ct\delta$.
Combining with \Cref{cond:condition1}, we obtain
\begin{align}
    \Big(1+\frac{\eta}{2}\lambda_{k+1}^*+\frac{\eta}{8}\Delta\Big)^{T_\Rc}\sigma_1^2(\Thetav_{\text{res}}^0)
    =
    \Big(\frac{\Delta}{4\sigma_k^2(\Thetav_k^0)}\Big)^{\kappa/2}\sigma_1^2(\Thetav_{\text{res}}^0)
    \leq
    \frac{\Delta}{8\sqrt{\lambda_1^*}}\sigma_k(\Thetav_k^0).\nonumber
\end{align}
Then, for all $t\leq T_\Rc$, we have 
\begin{align}
    \Big(1+\frac{\eta}{2}\lambda_{k+1}^*+\frac{\eta}{8}\Delta\Big)^{2t}\sigma_1^2(\Thetav_{\text{res}}^0)
    \leq
    \Big(1+\frac{\eta}{2}\lambda_{k+1}^*+\frac{\eta}{8}\Delta\Big)^{t+T_\Rc}\sigma_1^2(\Thetav_{\text{res}}^0)
    \leq
    \Big(1+\frac{\eta}{2}\lambda_{k}^*-\frac{\eta}{4}\Delta\Big)^t\frac{\Delta}{8\sqrt{\lambda_1^*}}\sigma_k(\Thetav_k^0).\nonumber
\end{align}
Let $T' = \min\{t>0|\sigma_k^2(\Thetav_k^t)\geq\Delta/4\}$. We then aim to prove that 
\begin{align}
    \sigma_k(\Thetav_k^{t})
    \geq
    \Big(1+\frac{\eta}{2}\lambda_{k}^*-\frac{\eta}{4}\Delta\Big)^t\sigma_k(\Thetav_k^0),\quad\forall t\leq\min\{T_\Rc,T'\}.\label{ineq:ineq23L}
\end{align}
We prove it by induction.

Assume \Cref{ineq:ineq23L} holds for some
$\tau\leq t$, where $t\leq\min\{T_\Rc,T'\}$. Then we have
\begin{align*}
    \sigma_1^2(\Thetav_{\text{res}}^\tau)
    \leq
    \Big(1+\frac{\eta}{2}\lambda_{k+1}^*+\frac{\eta}{8}\Delta\Big)^{2\tau}\sigma_1^2(\Thetav^0_{\text{res}})
    \leq
    \frac{\Delta}{8\sqrt{\lambda_1^*}}\Big(1+\frac{\eta}{2}\lambda_{k+1}^*-\frac{\eta}{4}\Delta\Big)^\tau \sigma_k(\Thetav_k^0)
    \leq
    \frac{\Delta}{8\sqrt{\lambda_1^*}}\sigma_k(\Thetav^\tau_k).
\end{align*}
We consider the next time step $\tau+1$.
Note that $\sigma_k(\Thetav^{\tau+1}_{k})$ can be lower bounded as
\begin{align*}
    \sigma_k(\Thetav^{\tau+1}_{k}) 
    &\geq
    \sigma_k(\Thetav^{\tau}_{k}+\frac{\eta}{2}\tilde{\Lambdav}_k\Thetav^{\tau}_{k}-\frac{\eta}{2}\Thetav^{\tau}_{k}(\Thetav^{\tau})^\top\Thetav^{\tau})-\sigma_1(\Rv^{\tau+1}_k)\\
    &\geq
    \sigma_k(\Thetav^{\tau}_{k}+\frac{\eta}{2}\tilde{\Lambdav}_k\Thetav^{\tau}_{k}-\frac{\eta}{2}\Thetav^{\tau}_{k}(\Thetav^{\tau}_k)^\top\Thetav^{\tau}_k)
    -
    \frac{\eta}{2}\sigma_1(\Thetav_k^{\tau}(\Thetav^{\tau}_{\text{res}})^\top\Thetav^{\tau}_{\text{res}})
    -\sigma_1(\Rv^{\tau+1}).
\end{align*}
Applying Lemma D.4 in \citet{jiang2022algorithmic} gives
\begin{align*}
   &\sigma_k(\Thetav^{\tau}_{k}+\frac{\eta}{2}\tilde{\Lambdav}_k\Thetav^{\tau}_{k}-\frac{\eta}{2}\Thetav^{\tau}_{k}(\Thetav^{\tau}_k)^\top\Thetav^{\tau}_k)
   -
    \frac{\eta}{2}\sigma_1(\Thetav_k^{\tau}(\Thetav^{\tau}_{\text{res}})^\top\Thetav^{\tau}_{\text{res}})\\
    &\geq
    \big(1-\frac{(\eta)^2}{2}\sigma_1(\tilde{\Lambdav}_k(\Thetav^{\tau}_k)^\top\Thetav^{\tau}_k)\big)(1+\frac{\eta}{2}\lambda_k^*)\sigma_k(\Thetav_k^{\tau})\big(1-\frac{\eta}{2}\sigma_k^2(\Thetav_k^{\tau})\big)
    -
    \frac{\eta}{2}\sigma_1(\Thetav_k^{\tau}(\Thetav^{\tau}_{\text{res}})^\top\Thetav^{\tau}_{\text{res}})\\
    &\geq
    \Big(1-\frac{(\eta)^2}{2}{\lambda_1^*}^2\Big)\Big(1+\frac{\eta\lambda_k^*}{2}\Big)\Big(1-\frac{\eta\Delta}{4}\Big)\sigma_k(\Thetav_k^\tau)-\frac{\eta\Delta}{8}\sigma_k(\Thetav_k^\tau).
\end{align*}
Then, for $\eta\leq\frac{\Delta^2}{18{\lambda_1^*}^3}$, we have
\begin{align}
   \sigma_k(\Thetav^{\tau}_{k}+\frac{\eta}{2}\tilde{\Lambdav}_k\Thetav^{\tau}_{k}-\frac{\eta}{2}\Thetav^{\tau}_{k}(\Thetav^{\tau}_k)^\top\Thetav^{\tau}_k)
   -
    \frac{\eta}{2}\sigma_1(\Thetav_k^{\tau}(\Thetav^{\tau}_{\text{res}})^\top\Thetav^{\tau}_{\text{res}})
    \geq
    \Big(1+\frac{\eta\lambda_k^*}{2}-\eta\Delta\frac{71}{288}\Big)\sigma_k(\Thetav_k^\tau).\label{ineq:130}
\end{align}
According to \Cref{lemma:concentration-of-Q} and \Cref{lemma:concentration-of-tildeQ}, if
\begin{align*}
    \sqrt{N}
    \geq
    \max\Big\{\frac{\sqrt{\dh-\log\delta}}{\sqrt{c_1}}, 
    \frac{3456 \sqrt{\lambda_1^*}(\sqrt{k(\lambda_1^*)^2 +E}+\sqrt{k}\sigma_{\xi})\sqrt{\dh-\log\delta}}{\sigma_k(\Thetav_k^0)\Delta\sqrt{c_1}}
    \Big\},
\end{align*}
then, with probability at least $1-2\delta$, we have $\sigma_1(\Rv^{\tau+1})\leq \frac{1}{288}\eta\Delta\sigma_k(\Thetav_{k}^0)$. 

Combining with \Cref{ineq:130} gives 
\begin{align} 
    \sigma_k(\Thetav^{\tau+1}_{k}) 
    &\geq
    \Big(1+\frac{\eta\lambda_k^*}{2}-\eta\Delta\frac{71}{288}\Big)\sigma_k(\Thetav_k^\tau)
    -\frac{1}{288}\eta\Delta\sigma_k(\Thetav_{k}^0)\nonumber \\
    &\geq
    \Big(1+\frac{\eta\lambda_k^*}{2}-\eta\Delta\frac{71}{288}\Big)\Big(1+\frac{\eta}{2}\lambda_{k}^*-\frac{\eta}{4}\Delta\Big)^\tau\sigma_k(\Thetav_k^0)
    -\frac{1}{288}\eta\Delta\sigma_k(\Thetav_{k}^0)\nonumber \\
    &\geq
    \Big(1+\frac{\eta\lambda_k^*}{2}-\eta\Delta\frac{71}{288}\Big)\Big(1+\frac{\eta}{2}\lambda_{k}^*-\frac{\eta}{4}\Delta\Big)^\tau\sigma_k(\Thetav_k^0)
    -\frac{1}{288}\Big(1+\frac{\eta}{2}\lambda_{k}^*-\frac{\eta}{4}\Delta\Big)^\tau\eta\Delta\sigma_k(\Thetav_{k}^0)\nonumber \\
    &=
    \Big(1+\frac{\eta}{2}\lambda_{k}^*-\frac{\eta}{4}\Delta\Big)^{\tau+1}\sigma_k(\Thetav_k^0).\nonumber
\end{align}
Then, we conclude that with probability at least $1-ct\delta$ for some constant $c$, we have
\begin{align*}
    \sigma_k(\Thetav^{t}_{k}) 
    \geq
    \Big(1+\eta\lambda_k^*/2-\eta\Delta/4\Big)^{t}\sigma_k(\Thetav_k^0).
\end{align*}
Here we claim $T_\Rc\leq T'$ always holds, since if $T_\Rc> T'$, we must have
\begin{align*}
    \sigma_k(\Thetav^{T_\Rc}_{k}) 
    \geq
    \Big(1+\eta\lambda_k^*/2-\eta\Delta/4\Big)^{T_\Rc}\sigma_k(\Thetav_k^0)\geq\sqrt{\Delta}/2,
\end{align*}
which contradicts the definition of $T'$. The proof is thus complete.
\end{proof}

\subsubsection{Step 2: Trapped in the Absorbing Region }\label{sec:proof-s2}
We start by introducing \Cref{lemma:lemma3}, \Cref{lemma:lemma4} and \Cref{lemma:lemma5}.

\begin{lemma}\label{lemma:lemma3}
    Assume $\eta\leq\frac{2}{5\lambda_1^*}$ and $\Thetav^0\in\Rc_s$. Then, if $\sqrt{N}\geq \frac{12\sqrt{2} \sqrt{\dh{k}- {k}\log\delta}}{\sqrt{c_1}}$ for constant $c_1$, with probability at least $1-2t\delta$, we have $\sigma_1(\Thetav^{\tau})\leq\sqrt{2\lambda_1^*}$ hold for all $\tau\leq t$.
\end{lemma}
\begin{proof}[Proof of \Cref{lemma:lemma3}]
    Assume that $\Thetav^{\tau-1}\in\Rc_s$. Then, utilizing \Cref{equ:44}, we have
    \begin{align}
        \sigma_1(\Thetav^\tau)
        \leq
        \sigma_1(\Thetav^{\tau-1})\big(1+\frac{\eta}{2}\lambda_1^*-\frac{\eta}{2}\sigma_1^2(\Thetav^{\tau-1})\big)+\sigma_1(\Rv^{\tau}).\nonumber
    \end{align}%\jingc{is this $t$ the same?}
Note that $\sigma_1(\Thetav^{\tau-1})\big(1+\frac{\eta}{2}\lambda_1^*-\frac{\eta}{2}\sigma_1^2(\Thetav^{\tau-1})\big)$ reaches its maximum at $\sigma_1(\Thetav^{\tau-1}) = \sqrt{\frac{2+\eta\lambda_1^*}{3\eta}}$. For $\eta\leq\frac{2}{5\lambda_1^*}$, we have $\sqrt{\frac{2+\eta\lambda_1^*}{3\eta}}\geq\sqrt{2\lambda_1^*}$. Thus, $\sigma_1(\Thetav^{\tau-1})\big(1+\frac{\eta}{2}\lambda_1^*-\frac{\eta}{2}\sigma_1^2(\Thetav^{\tau-1})\big)$ is monotonically increasing for $\sigma_1(\Thetav^{\tau-1})\in[0,\sqrt{2\lambda_1^*}]$. Then,
\begin{align}
    \sigma_1(\Thetav^t\tau)
    \leq
    \sqrt{2\lambda_1^*}\big(1-\frac{\eta}{2}\lambda_1^*\big)+\sigma_1(\Rv^{\tau}).\nonumber
\end{align} 
We prove $\sigma_1(\Thetav^\tau)\leq \sqrt{2\lambda_1^*}$ by induction. First, since $\Thetav^0\in\Rc_s$, we have $\sigma_1(\Thetav^0)\leq \sqrt{2\lambda_1^*}$. 
Then, assume $\sigma_1(\Thetav^\tau)\leq \sqrt{2\lambda_1^*}$ holds for time step $0\leq\tau<t$. According to \Cref{lemma:concentration-of-Q} and \Cref{lemma:concentration-of-tildeQ}, if $\sigma_1(\Thetav^\tau)\leq \sqrt{2\lambda_1^*}$, and $\sqrt{N}\geq \frac{12\sqrt{2} \sqrt{\dh {k}- {k}\log\delta}}{\sqrt{c_1}}$, with probability at least $1-2\delta$, we have $\sigma_1(\Rv^{\tau+1})\leq \frac{\sqrt{2}}{2}\eta(\lambda_1^*)^{\frac{3}{2}}$. Thus,
\begin{align}
    \sigma_1(\Thetav^{\tau+1})
    \leq
    \sqrt{2\lambda_1^*}\big(1-\frac{\eta}{2}\lambda_1^*\big)+\sigma_1(\Rv^{\tau+1})\leq\sqrt{2\lambda_1^*}.\nonumber
\end{align}
Then, by induction, with probability at least $1-2t\delta$, we have $\sigma_1(\Thetav^\tau)\leq\sqrt{2\lambda_1^*}$ for all $\tau\leq t$. Then the proof is complete.
\end{proof}

%---------------------------------------------------------
\begin{lemma}\label{lemma:lemma4}
    Assume $\eta\leq\frac{1}{6\lambda_1^*}$ and $\Thetav^0\in\Rc_s$. Then, if $$\sqrt{N}\geq\max\left\{\frac{\sqrt{\dh-\log\delta}}{\sqrt{c_1}}, \frac{48 \sqrt{\lambda_1^*}(\sqrt{k(\lambda_1^*)^2 +E}+\sqrt{k}\sigma_{\xi})\sqrt{\dh -\log\delta}}{\Delta\sqrt{c_1(\lambda_k^*-\Delta/2)}}\right\}$$ for constant $c_1$, with probability at least $1-2t\delta$, we have $\sigma_1(\Thetav_{\text{res}}^{\tau})\leq\sqrt{\lambda_k^*-\Delta/2}$ holds for all $\tau\leq t$.%\jingc{?}
\end{lemma}

\begin{proof}[Proof of \Cref{lemma:lemma4}]
We prove it by induction. Note that $\sigma_1(\Thetav^{0}_{\text{res}})\leq\sqrt{\lambda_k^*-\Delta/2}$ since $\Thetav^0\in\Rc_s$. Assume $\sigma_1(\Thetav_{\text{res}}^{\tau-1})\leq\sqrt{\lambda_k^*-\Delta/2}$ holds for some $\tau$. We aim to show that the inequality holds for $\sigma_1(\Thetav_{\text{res}}^{\tau})$ as well. 

Based on \Cref{lemma:upb-of-theta-res}, for $\eta\leq 1/6\lambda^*_1$ and $\sigma_1(\Thetav^\tau)\leq\sqrt{2\lambda_1^*}$, we have
   \begin{align}
       \sigma_1(\Thetav^\tau_{\text{res}})
       &\leq
       \Big(1+\frac{\eta}{2}\big(\lambda^*_{k+1}-\sigma_1^2(\Thetav^{\tau-1}_{\text{res}})-\sigma_k^2(\Thetav^{\tau-1}_{k})\big)\Big)\sigma_1(\Thetav^{\tau-1}_{\text{res}})+\sigma_1(\Rv^{\tau}_{2\dh-k})\nonumber \\ 
       &\leq
       \Big(1+\frac{\eta}{2}\big(\lambda^*_{k+1}-\sigma_1^2(\Thetav^{\tau-1}_{\text{res}})\big)\Big)\sigma_1(\Thetav^{\tau-1}_{\text{res}})+\sigma_1(\Rv^{\tau}_{2\dh-k}).\nonumber
   \end{align}
Note that when $\sigma_1(\Thetav_{\text{res}}^{\tau-1})\geq 0$, $\Big(1+\frac{\eta}{2}\big(\lambda^*_{k+1}-\sigma_1^2(\Thetav^{\tau-1}_{\text{res}})\big)\Big)\sigma_1(\Thetav^{\tau-1}_{\text{res}})$ is maximized at $\sigma_1(\Thetav^{\tau-1}_{\text{res}}) = \sqrt{\frac{2+\eta\lambda_{k+1}^*}{3\eta}}$. Since we assume $\eta\leq\frac{1}{6\lambda_1^*}\leq\frac{2}{3\lambda_k^*+\lambda_{k+1}^*}$, it holds that $\sqrt{\frac{2+\eta\lambda_{k+1}^*}{3\eta}}\geq\sqrt{\lambda_k^*-\Delta/2}$. Then, $\Big(1+\frac{\eta}{2}\big(\lambda^*_{k+1}-\sigma_1^2(\Thetav^{\tau-1}_{\text{res}})\big)\Big)\sigma_1(\Thetav^{\tau-1}_{\text{res}})$ is monotonically increasing for $0\leq\sigma_1(\Thetav^{\tau-1}_{\text{res}})\leq\sqrt{\lambda^*_k-\Delta/2}$. We thus have
\begin{align}
    \sigma_1(\Thetav^\tau_{\text{res}})
    \leq
    \sqrt{\lambda^*_k-\Delta/2}\big(1-\frac{\eta\Delta}{4}\big)+\sigma_1(\Rv^{\tau}).\label{ineq:ineq22}
\end{align}
According to \Cref{lemma:concentration-of-Q} and \Cref{lemma:concentration-of-tildeQ}, if $\sqrt{N}\geq\max\big\{\frac{\sqrt{\dh-\log\delta}}{\sqrt{c_1}}, \frac{48 \sqrt{\lambda_1^*}(\sqrt{k(\lambda_1^*)^+E}+\sqrt{k}\sigma_{\xi})\sqrt{\dh-\log\delta}}{\Delta\sqrt{c_1(\lambda_k^*-\Delta/2)}}\big\}$, with probability at least $1-2\delta$, we have
\begin{align}
    \sigma_1(\Rv^\tau)
    \leq
    \frac{\Delta \eta\sqrt{\lambda_k^*-\Delta/2}}{4}.\label{ineq:ineq23}
\end{align}
Then by combining \eqref{ineq:ineq22} and \eqref{ineq:ineq23}, we have $\sigma_1(\Thetav^\tau_{\text{res}})\leq\sqrt{\lambda^*_k-\Delta/2}$. The proof is thus complete.
\end{proof}

%--------------------------------------------------------
\begin{lemma}\label{lemma:lemma5}
    Assume $\eta\leq\frac{\Delta^2}{32{\lambda_1^*}^3}$, $\sigma_k(\Thetav_k^0){\geq}\sqrt{\Delta}/2$ and $\Thetav^\tau\in\Rc_s$ for all $0<\tau\leq t$. Then, if 
    \begin{align*}
    \sqrt{N}
    \geq
    \max\bigg\{\frac{\sqrt{\dh-\log\delta}}{\sqrt{c_1}}
    ,\quad
    \frac{6144\sqrt{\lambda_1^*}(\sqrt{k(\lambda_1^*)^2 +E}+\sqrt{k}\sigma_{\xi})\sqrt{\dh-\log\delta}}{\Delta\sqrt{c_1}}\bigg\},
\end{align*}
    for some contact $c_1$, with probability at least $1-2t\delta$, we have $\sigma_k(\Thetav^{\tau}_k)\geq\sqrt{\Delta}/2$ hold for all $0<\tau\leq t$.
\end{lemma}

\begin{proof}[Proof of \Cref{lemma:lemma5}]
We prove it by induction. First note that $\sigma_k(\Thetav^{0}_k)\geq\sqrt{\Delta}/2$ under the assumption of \Cref{lemma:lemma5}. Assume that $\sigma_k(\Thetav^{\tau}_k)\geq\sqrt{\Delta}/2$ holds for some $t$. We then show that $\sigma_k(\Thetav^{\tau+1}_k)\geq\sqrt{\Delta}/2$ holds as well.

Since for all $\tau\leq t$ it holds that $\Thetav^\tau\in\Rc_s$ and $N$ satisfies the condition described in \Cref{lemma:lemma5}, based on \Cref{lemma:concentration-of-Q,lemma:concentration-of-tildeQ}, with probability at least $1-2t\delta$, it holds that $\sigma_1(\Rv^{\tau}_{k})\leq\frac{\eta\Delta^2}{512\sqrt{\lambda^*_1}}$ for all $0\leq\tau \leq t$. From the intermediate result of \Cref{lemma:lbd-of-theta-k}, for $\eta\leq\frac{\Delta^2}{16{\lambda^*_1}^3}$, we have
\begin{align*}
        \sigma_k^2({\Thetav}^{\tau+1}_{k})
        &\geq 
        \Big(1+\eta\big(\lambda^*_k-\sigma_1^2(\Thetav^{\tau}_{\text{res}})-\sigma_k^2(\Thetav^{\tau}_{k})\big)\Big)\sigma_k^2(\Thetav^{\tau}_{k})-{\eta}^2{\lambda_1^*}^3-4\sqrt{\lambda_1^*}\sigma_1(\Rv_k^\tau).
\end{align*}
Combining with the fact that
\begin{align*}
    \sigma_k^2(\Thetav^{\tau+1}_{k}) 
    \geq
    \Big(1+\eta\big(\lambda^*_k-\sigma_1^2(\Thetav^{\tau}_{\text{res}})-\sigma_k^2(\Thetav^{\tau}_{k})\big)\Big)\sigma_k^2(\Thetav^{\tau}_{k})-{\eta}^2{\lambda_1^*}^3
    -\frac{\eta\Delta^2}{128},
\end{align*}
we have
\begin{align*}
    \sigma_k^2(\Thetav^{\tau+1}_{k}) 
    &\geq
    \Big(1+\eta\big(\Delta/2-\sigma_k^2(\Thetav^{\tau}_{k})\big)\Big)\sigma_k^2(\Thetav^{\tau}_{k})-{\eta}^2{\lambda_1^*}^3-\frac{\eta\Delta^2}{128}\\
    &\geq
    (1+\eta(\Delta/2-\Delta/4))\Delta/4-{\eta}^2{\lambda_1^*}^3-\frac{\eta\Delta^2}{128}\\
    &\geq
    \frac{\Delta}{4}+\eta\frac{\Delta^2}{16}-\eta\frac{\Delta^2}{32}-\frac{\eta\Delta^2}{128}\\
    &\geq
    \frac{\Delta}{4}.
\end{align*}
The proof is thus complete.
\end{proof}

%-------------------------------------------------------

Combining \Cref{lemma:lemma3,lemma:lemma4}, we conclude that for $N$ sufficiently large, $\Rc_s$ is an absorbing region with high probability, i.e., starting from $\Thetav^0\in\Rc_s$, the subsequent $\Thetav^t$ will stay in $\Rc_s$ for all $t>0$ with high probability, which is summarized in the following proposition. 
  
\begin{proposition}
\label{thm:thm3-formal}
Assume $\Thetav^0\in\Rc$. If
$
\sqrt{N}
\geq
c\frac{\sqrt{\lambda_1^*}(\sqrt{k(\lambda_1^*)^2 +E}+\sqrt{k}\sigma)\sqrt{\dh-\log\delta}}{\Delta^{\frac{3}{2}}}
$
for some constant $c$, then, with probability at least $1-t\delta$, we have $\Thetav^\tau\in\Rc$ for all $0\leq\tau\leq t$. 
\end{proposition}

\if{0}
\begin{proof}[Proof of \Cref{thm:thm3-formal}]
    Combining \Cref{lemma:lemma3,lemma:lemma4}, we prove that for $N$ sufficiently large, $\Rc_s$ is an absorbing region with high probability, which states that starting from $\Thetav^0\in\Rc_s$, the subsequent $\Thetav^t$ will stay in $\Rc_s$ for all $t>0$ with high probability:

In \Cref{lemma:lemma5}, we will \jingc{?} show that for $N$ large enough, it holds that $\sigma_k(\Thetav_k^\tau)\geq\sqrt{\Delta}/2$. Note that by combing with lemma \ref{lemma:lemma3}, lemma \ref{lemma:lemma4}, and lemma \ref{lemma:lemma5}, we show that $\Rc^t$ is an absorbing region with high probability.
\end{proof}
\fi

\subsubsection{Step 3: Local Convergence of \(\Thetav^t(\Thetav^t)^\top\)}\label{sec:proof-s3}
We next show that when $N$ is sufficiently large, with high probability, $\|\Thetav^t(\Thetav^t)^\top-\diag(\tilde{\Lambdav}_k,\mathbf{0})\|$ converges to $0$ {exponentially fast} when $ \Thetav^0\in\Rc$.

Firstly, we establish the following lemma that lower bounds the number of samples needed for the {inverse SNR} to converge exponentially fast with high probability.

\begin{lemma}\label{lemma:lemma10}
 Denote $\sigma_{\text{ref}}^{t+1} = \sqrt{\lambda_1^*}\Big(1-\frac{\eta\Delta}{16}\Big)^{t+1}$. Assume $\eta\leq \Delta^2/(36\lambda^{*3}_1)$, $\Thetav^\tau\in\Rc$ for all $0\leq \tau \leq t$, and 
   \begin{equation*}
   \begin{aligned}
    \sqrt{N}
    \geq
    c\cdot\max\bigg\{\sqrt{\dh-\log\delta}
    ,
    \frac{\sqrt{\dh-\log\delta}\sqrt{\lambda_1^*}(\sqrt{k(\lambda_1^*)^2 +E}+\sqrt{k}\sigma_{\xi})}{\sigma_{\text{ref}}^{t+1}\Delta}
    ,
    \frac{\sqrt{\dh-\log\delta}{\lambda_1^*}(\sqrt{k(\lambda_1^*)^2+E}+\sqrt{k}\sigma)}{\Delta^2}\bigg\}
   \end{aligned}
   \end{equation*}
for some constant $c$. Then, with probability at least $1-(t+1)\delta$, we have
   \begin{align}
    {\sigma_1(\Thetav_{\text{res}}^{t+1})}
    \leq  
    \frac{2\sqrt{2}{\lambda_1^*}}{\sqrt{\Delta}}\Big(1-\frac{\eta\Delta}{16}\Big)^{t+1}.\nonumber
\end{align}
\end{lemma}
\begin{proof}
[Proof of \Cref{lemma:lemma10}]
Based on \Cref{lemma:interlacing}, we have $\sigma_1(\Rv^{\tau+1}_{2\dh-k})\leq\sigma_1(\Rv^{\tau+1})$ and $\sigma_1(\Rv^{\tau+1}_{k})\leq\sigma_1(\Rv^{\tau+1})$. Then, 
it follows that
\begin{align*}
    \sigma_1(\Rv^{\tau+1}_{2\dh-k})
    \leq
    \sqrt{\sigma^2_1((\Qv^{\tau+1}_\star)^\top\Uv_\star)+\sigma_1^2(\tilde{\Qv}^{\tau+1}_\star\Vv_\star)}
    \quad \text{and}\quad
    \sigma_1(\Rv^{\tau+1}_{k})
    \leq
    \sqrt{\sigma^2_1((\Qv^{\tau+1}_\star)^\top\Uv_\star)+\sigma_1^2(\tilde{\Qv}_\star^{\tau+1}\Vv_\star)}.
\end{align*}

Substitute the $\sigma$ in \Cref{lemma:concentration-of-Q,lemma:concentration-of-tildeQ} by $\sigma_{\text{ref}}^{t+1}$. Then, if
\begin{align*}
    \sqrt{N}
    \geq
    \max\bigg\{\frac{\sqrt{\dh-\log\delta} }{\sqrt{c_1}}
    ,
    \frac{192\sqrt{\dh-\log\delta} \sqrt{\lambda_1^*}(\sqrt{k(\lambda_1^*)^2 +E}+\sqrt{k}\sigma_{\xi})}{\sigma_{\text{ref}}^{t+1}\Delta\sqrt{c_1}}
    ,
    \frac{6144\sqrt{\dh-\log\delta} {\lambda_1^*}(\sqrt{k(\lambda_1^*)^2 +E}+\sqrt{k}\sigma_{\xi})}{\Delta^2\sqrt{c_1}}\bigg\},
\end{align*}
with probability at least $1-2\delta$, we have
\begin{align}
    \sigma_1(\Rv^{\tau}_{2d-k})
    \leq
    \frac{\sigma_{\text{ref}}^{ t+1}\eta\Delta}{16}
    \quad \text{and} \quad
    \sigma_1(\Rv^{\tau}_{k})
    \leq
    \frac{\eta\Delta^2}{512\sqrt{\lambda^*_1}}.\label{ineq:upb-of-two-Rs}
\end{align}

We prove the lemma by considering two cases. In the first case, we assume $\sigma_{1}(\Thetav^\tau_{\text{res}})\geq \sigma_{\text{ref}}^{t+1}$ for all $0\leq \tau\leq t$. In the second case, we assume there exists at least one time step in $[0, t]$ such that $\sigma_{1}(\Thetav^\tau_{\text{res}})<\sigma_{\text{ref}}^{t+1}$, and we denote the last time step satisfying this condition as $t'$.

We start from the first case.
Combining \Cref{ineq:upb-of-two-Rs} with \Cref{lemma:upb-of-theta-res} gives 
\begin{align}
        \sigma_1(\Thetav^{\tau+1}_{\text{res}})
        &\leq
        \Big(1+\frac{\eta}{2}\big(\lambda^*_{k+1}-\sigma_1^2(\Thetav^{\tau}_{\text{res}})-\sigma_k^2(\Thetav^{\tau}_{k})\big)\Big)\sigma_1(\Thetav^{\tau}_{\text{res}})
        +
        \frac{\sigma_1(\Thetav^\tau_{\text{res}})\eta\Delta}{16}\label{ineq:ineq42a}\\
        &=
       \Big(1+\frac{\eta}{2}\big(\lambda^*_{k+1}
       +\Delta/8
       -\sigma_1^2(\Thetav^{\tau}_{\text{res}})-\sigma_k^2(\Thetav^{\tau}_{k})\big)\Big)\sigma_1(\Thetav^{\tau}_{\text{res}}),\nonumber
\end{align}
where in \Cref{ineq:ineq42a} we use the assumption that $\sigma_{\text{ref}}^{t+1}\leq \sigma_1(\Thetav_{\text{res}}^\tau)$. 

Then, using the fact that $\sigma_1^2(\Thetav^\tau)\leq2\lambda_1^*$ and $\eta\leq\frac{\Delta}{16{\lambda^*_1}^2}$ we obtain
\begin{align}
    \sigma_1^2(\Thetav^{\tau+1}_{\text{res}})
    &\leq
    \Big(1+\eta\big(\lambda^*_{k+1}
       +\Delta/8
       -\sigma_1^2(\Thetav^{\tau}_{\text{res}})-\sigma_k^2(\Thetav^{\tau}_{k})\big)+4{\eta}^2{\lambda_1^*}^2\Big)\sigma_1^2(\Thetav^{\tau}_{\text{res}})\label{ineq:ineq44a}\\
       &\leq
       \Big(1+\eta\big(\lambda^*_{k+1}
       +\Delta/8
       -\sigma_1^2(\Thetav^{\tau}_{\text{res}})-\sigma_k^2(\Thetav^{\tau}_{k})\big)+\frac{\eta\Delta}{4}\Big)\sigma_1^2(\Thetav^{\tau}_{\text{res}})\nonumber\\
       &\leq
       \Big(1-\eta\Delta/8+\eta\big(\lambda^*_{k}
       -\Delta/2
       -\sigma_1^2(\Thetav^{\tau}_{\text{res}})-\sigma_k^2(\Thetav^{\tau}_{k})\big)\Big)\sigma_1^2(\Thetav^{\tau}_{\text{res}}),\label{ineq:93}
\end{align}
where \Cref{ineq:ineq44a} holds since $\big(\lambda^*_{k+1}+\Delta/8-\sigma_1^2(\Thetav^{\tau}_{\text{res}})-\sigma_k^2(\Thetav^{\tau}_{k})\big)^2\leq16{\lambda_1^*}^2$. 

Next, combining \Cref{lemma:lbd-of-theta-k} and \Cref{ineq:upb-of-two-Rs} leads to
\begin{align}
        \sigma_k^2(\Thetav^{\tau+1}_{k}) 
        &\geq
        \sigma_k^2(\tilde{\Thetav}_k^{\tau+1})-4\sqrt{\lambda^*_1}\sigma_1(\Rv^{\tau+1}_k)\nonumber\\
        &\geq
        \sigma_k^2(\tilde{\Thetav}_k^{\tau+1})-\frac{\eta\Delta^2}{128}\nonumber\\
        &\geq
        \sigma_k^2(\tilde{\Thetav}_k^{\tau+1})-\frac{\eta\Delta}{8}\sigma_k^2(\Thetav^{\tau}_{k})\label{ineq:ineq49a}\\
        &\geq
        \Big(1+\eta\big(\lambda^*_k-\sigma_1^2(\Thetav^{\tau}_{\text{res}})-\sigma_k^2(\Thetav^{\tau}_{k})\big)-\frac{\eta\Delta}{4}\Big)\sigma_k^2(\Thetav^{\tau}_{k})-\frac{\eta\Delta}{8}\sigma_k^2(\Thetav^{\tau}_{k})\nonumber\\
        &=
        \Big(1+\eta\Delta/8+\eta\big(\lambda^*_k-\Delta/2-\sigma_1^2(\Thetav^{\tau}_{\text{res}})-\sigma_k^2(\Thetav^{\tau}_{k})\big)\Big)\sigma_k^2(\Thetav^{\tau}_{k}),\label{eqn:98}
\end{align}
where in \Cref{ineq:ineq49a} we use the fact $\sigma_k(\Thetav^{\tau}_k)\geq\frac{\Delta}{4}$. 

Then, combining \Cref{ineq:93} with \Cref{eqn:98} we have
\begin{align}
    \frac{\sigma_1^2(\Thetav_{\text{res}}^{\tau+1})}{\sigma_k^2(\Thetav_k^{\tau+1})}
    &\leq
    \frac{\Big(1-\eta\Delta/8+\eta\big(\lambda^*_{k}
       -\Delta/2
       -\sigma_1^2(\Thetav^{\tau}_{\text{res}})-\sigma_k^2(\Thetav^{\tau}_{k})\big)\Big)\sigma_1^2(\Thetav^{\tau}_{\text{res}})}{\Big(1+\eta\Delta/8+\eta\big(\lambda^*_k-\Delta/2-\sigma_1^2(\Thetav^{\tau}_{\text{res}})-\sigma_k^2(\Thetav^{\tau}_{k})\big)\Big)\sigma_k^2(\Thetav^{\tau}_{k})}\nonumber\\
       &\leq
       \frac{3/2-\eta\Delta/8}{3/2+\eta\Delta/8}
       \cdot
       \frac{\sigma_1^2(\Thetav^{\tau}_{\text{res}})}{\sigma_k^2(\Thetav^{\tau}_{k})}\label{ineq:ineq53a}\\
       &\leq
       \Big(1-\frac{\eta\Delta}{6}\Big)\frac{\sigma_1^2(\Thetav^{\tau}_{\text{res}})}{\sigma_k^2(\Thetav^{\tau}_{k})}\nonumber\\
       &\leq
       \Big(1-\frac{\eta\Delta}{16}\Big)^2\frac{\sigma_1^2(\Thetav^{\tau}_{\text{res}})}{\sigma_k^2(\Thetav^{\tau}_{k})}\label{ineq:ineq55a},
\end{align}
where \Cref{ineq:ineq53a} holds when $-1/2\leq\eta\big(\lambda^*_{k}-\Delta/2-\sigma_1^2(\Thetav^{\tau}_{\text{res}})-\sigma_k^2(\Thetav^{\tau}_{k})\big)\leq1/2$, which is valid when $\eta\leq \Delta^2/(36{\lambda_1^*}^3)$, and \Cref{ineq:ineq55a} holds since $(1-\eta\Delta/6)\leq(1-\eta\Delta/16)$ is valid for positive $\eta$. 

Then, with probability at least $1-2(t+1)\delta$, we have
\begin{align*}
    {\sigma_1^2(\Thetav_{\text{res}}^{t+1})}
    \leq
    \Big(1-\frac{\eta\Delta}{16}\Big)^{2(t+1)}
    \frac{\sigma_1^2(\Thetav^{0}_{\text{res}})}{\sigma_k^2(\Thetav^{0}_{k})}{\sigma_k^2(\Thetav_k^{t+1})}
    \leq  
    \frac{8{\lambda_1^*}^2}{\Delta}\Big(1-\frac{\eta\Delta}{16}\Big)^{2(t+1)}.
\end{align*}

For the second case, note at time step $t'$ we have $\sigma_{1}(\Thetav_{\text{res}}^{t'})<\sigma_{\text{ref}}^{ t+1}$, and for all $t'<\tau\leq t$ we have $\sigma_{1}(\Thetav_{\text{res}}^\tau)\geq\sigma_{\text{ref}}^{t+1}$. Similar to the previous analysis, we show that with probability at least $1-2(t+1-t')\delta$, we have
\begin{align*}
    {\sigma_1^2(\Thetav_{\text{res}}^{t+1})}
    &\leq
    \Big(1-\frac{\eta\Delta}{6}\Big)^{t+1-t'}
    \frac{\big(\sigma_{\text{ref}}^{t+1}\big)^2}{\sigma_k^2(\Thetav^{t'}_{k})}{\sigma_k^2(\Thetav_k^{t+1})}\\
&\leq
    \frac{8\lambda_1^*}{\Delta}\big(\sigma_{\text{ref}}^{  t+1}\big)^2\Big(1-\frac{\eta\Delta}{16}\Big)^{2(t+1-t')}\\
    &\leq
    \frac{8{\lambda_1^*}^2}{\Delta}\Big(1-\frac{\eta\Delta}{16}\Big)^{2(2t+2-t')}\\\
    &\leq  
    \frac{8{\lambda_1^*}^2}{\Delta}
    \Big(1-\frac{\eta\Delta}{16}\Big)^{2(t+1)}.
\end{align*}
The proof is complete by combining the two cases.
\end{proof}
The following lemma characterizes the number of samples needed for $\Thetav_k$ to converge to $\tilde{\Lambdav}_k$, which is based on the convergence of the inverse SNR.
\begin{lemma}\label{lemma:lemma12}
   Assume $\eta\leq \Delta^2/(36\lambda^{*3}_1)$, $\Thetav^t\in\Rc$ for all $0\leq \tau \leq t$ and $N$ satisfies $\sqrt{N}
    \geq 
    c\cdot \max\bigg\{
    \sqrt{\dh-\log\delta}
    ,
    \frac{\sqrt{\dh-\log\delta}{\lambda_1^*}(\sqrt{k(\lambda_1^*)^2 +E}+\sqrt{k}\sigma_{\xi})}{\sigma_{D}^{t+1}\Delta}
    \bigg\}$ for some constant $c$. Then, with probability at least $1-(t+1)\delta$, we have
\begin{align}
    \sigma_1(\Dv^{t+1})\leq\frac{200{\lambda_1^*}^2}{\eta\Delta^2}\Big(1-\frac{\eta\Delta}{16}\Big)^{t+1},\nonumber
\end{align}
where $ \sigma_{  D}^{ t+1} = \min\Big\{3\lambda_1^*,(1-\frac{\eta\Delta}{16})^{t+1}\frac{{\lambda_1^*}^2}{\eta\Delta^2}\Big\}$.
\end{lemma}

\begin{proof}
[Proof of \Cref{lemma:lemma12}]
We denote $\Dv^\tau = {\Thetav}_k^{\tau}({\Thetav}_k^{\tau})^\top-\tilde{\Lambdav}_k$. For $\tilde{\Thetav}^{\tau}_k$ defined in \Cref{lemma:lbd-of-theta-k}, we have
\begin{align}
    \Dv^\tau = \tilde{\Thetav}_k^{\tau}(\tilde{\Thetav}_k^{\tau})^\top-\tilde{\Lambdav}_k   +\tilde{\Thetav}_k^{\tau}(\Rv^{\tau}_k)^\top+\Rv^{\tau}_k(\tilde{\Thetav}_k^{\tau})^T +
    \Rv^{\tau}_k(\Rv^{\tau}_k)^\top.\label{equ:equ83}
\end{align}
Let $\sigma_{  D}^{t+1} = \min\Big\{3\lambda_1^*,(1-\frac{\eta\Delta}{16})^{t+1}\frac{{\lambda_1^*}^2}{\eta\Delta^2}\Big\}$. 
Then, if 
\begin{align}
    \sqrt{N}
    \geq 
    \max\bigg\{
    \frac{\sqrt{\dh-\log\delta} }{\sqrt{c_1}}
    ,
    \frac{1152\sqrt{\dh-\log\delta} {\lambda_1^*}(\sqrt{k(\lambda_1^*)^2 +E}+\sqrt{k}\sigma_{\xi})}{\sigma_{D}^{t+1}\Delta\sqrt{c_1}}
    \bigg\},\label{ineq:N-lemma12}
\end{align}
we have $\|\Rv^\tau_k\|
    \leq
    \frac{\sigma_{  D}^{ t+1}\eta\Delta}{96\sqrt{\lambda_1^*}}$. It follows that
\begin{align}
    \|\tilde{\Thetav}_k^\tau(\Rv_k^\tau)^\top\|
    \leq
    \frac{\sigma_{  D}^{t+1}\eta\Delta}{48},\nonumber
\end{align}
and
\begin{align}
    \|\Rv^\tau_k(\Rv_k^\tau)^\top\|
    \leq
    \Big(\frac{\sigma_{  D}^{t+1}\eta\Delta}{96\sqrt{\lambda_1^*}}\Big)^2
    \leq
    \frac{\sigma_{  D}^{t+1}\eta\Delta}{48}\label{ineq:ineq61a},
\end{align}
where \Cref{ineq:ineq61a} holds since $\frac{\sigma_{  D}^{  t+1}\eta\Delta}{96\sqrt{\lambda_1^*}}\leq 1$. 
Then, with probability at least $1-2\delta$, we have
     \begin{align}
         \sigma_{1}\Big(\tilde{\Thetav}_k^{\tau}(\Rv^{\tau}_k)^\top+\Rv^{\tau}_k(\tilde{\Thetav}_k^{\tau})^\top +
         \Rv^{\tau}_k(\Rv^{\tau}_k)^\top\Big)\leq\frac{\eta\Delta}{16}\sigma_{D}^{t+1}.\nonumber
     \end{align}
We prove the lemma by considering two cases: In the first case, we assume $\sigma_{1}(\Dv^\tau)\geq \sigma_{D}^{t+1}$ for all $0\leq \tau\leq t$; In the second case, we assume there is at least one time step in $[0, t]$ such that $\sigma_{1}(\Dv^\tau)<\sigma_{D}^{t+1}$, and we denote the latest time step satisfies this condition as $t'$.

We start from the first case.
From Section A.3 in \citet{chen2023fast}, we have
\begin{align*}
    \sigma_{1}(\Dv^{\tau})
    &\leq
    (1-\frac{\eta\Delta}{8})\sigma_{1}(\Dv^{\tau-1})+
    \sigma_1^2(\Thetav_{\res}^{\tau-1})+ 
     {\frac{\eta\Delta}{16}}\sigma_{  D}^{t+1}\\
    &\leq
    (1-\frac{\eta\Delta}{8})\sigma_{1}(\Dv^{\tau-1})+(1-\frac{\eta\Delta}{16})^{2(\tau-1)}\frac{8{{\lambda_1^*}^2}}{{\Delta}}+\frac{\eta\Delta}{16}\sigma_1(\Dv^{\tau-1})\\
    &\leq
    (1-\frac{\eta\Delta}{16})\sigma_{1}(\Dv^{\tau-1})+(1-\frac{\eta\Delta}{16})^{2(\tau-1)}\frac{8{\lambda_1^*}^2}{{\Delta}}.
\end{align*}
Then, for $N$ satisfying \Cref{ineq:N-lemma12}, with probability at least $1-2(t+1)\delta$, we have
\begin{align}
    \frac{\sigma_1(\Dv^{t+1})}{(1-\frac{\eta\Delta}{16})^{t+1}}
    &\leq
    \sigma_1(\Dv^0)+\sum_{i=0}^{t}\bigg({1-\eta\Delta/16}\bigg)^i\frac{8{\lambda^*_1}^2}{(1-\eta\Delta/16){\Delta}}\nonumber\\
    &\leq \sigma_1(\Dv^0)+\frac{130{\lambda_1^*}^2}{\eta\Delta^2}\nonumber\\
    &\leq \frac{200{\lambda_1^*}^2}{\eta\Delta^2},  \label{ineq:ineq69a}
\end{align}
where \Cref{ineq:ineq69a} follows from the fact that $\lambda_1^*/\eta\Delta^2\geq1$ and $\sigma_1(\Dv^0)\leq 3 \lambda_1^*\leq3 {\lambda_1^*}^2/\eta\Delta^2$.  
Therefore, we conclude that $\sigma_1(\Dv^{t+1})\leq(1-\eta\Delta/16)^{ t+1}\cdot\frac{200{\lambda_1^*}^2}{\eta\Delta^2}$.

For the second case, note at time step $t'$ we have $\sigma_{1}(\Dv^{t'})<\sigma_{D}^{t+1}$, and for all $t'<\tau\leq t$ we have $\sigma_{1}(\Dv^{ {\tau}})\geq\sigma_{D}^{t+1}$. Similar to the previous analysis, we show that with probability at least $1-2(t+1-t')\exp(-c_2(d+k))$, it has
\begin{align*}
    \frac{\sigma_1(\Dv^{t+1})}{(1-\frac{\eta\Delta}{16})^{t+1-t'}}
    &\leq
    \sigma_1(\Dv^{t'})
    +
    \sum_{i=0}^{t+1-t'}\bigg({1-\eta\Delta/16}\bigg)^i\frac{8{\lambda^*_1}^2}{(1-\eta\Delta/16)\Delta},\\
    &\leq
    \frac{{\lambda_1^*}^2}{\eta\Delta^2}(1-\frac{\eta\Delta}{16})^{t+1}+\frac{130{\lambda_1^*}^2}{\eta\Delta^2}\\
    &\leq
    \frac{200{\lambda_1^*}^2}{\eta\Delta^2}.
\end{align*}
The proof is complete by combining the two cases.
\end{proof}

Then, we aim to show the local convergence property of $\Thetav^t$ stated in the following proposition.
\begin{proposition}\label{thm:thm4-formal}
Assume $\Thetav^0\in\Rc^0$, $\eta\leq \Delta^2/(36\lambda^{*3}_1)$ and $N$ satisfies
\begin{equation}
    \begin{aligned}
    \sqrt{N}
    \geq
    \max\bigg\{\frac{ \sqrt{\dh-\log\delta} }{\sqrt{c_1}}
    ,
    \frac{1152  \sqrt{\dh-\log\delta} (\sqrt{k(\lambda_1^*)^2 +E}+\sqrt{k}\sigma_{\xi})}{\kappa^t \Delta\sqrt{c_1}}
    ,
    \frac{6144 \sqrt{\dh-\log\delta} {\lambda_1^*}(\sqrt{k(\lambda_1^*)^2 +E}+\sqrt{k}\sigma_{\xi})}{\Delta^2\sqrt{c_1}}\bigg\},\label{ineq:ineq31}
    \end{aligned}
\end{equation}
for constant $c_1$ and
$\kappa^{t}= \big(1-\frac{\eta\Delta}{16}\big)^{t+1}$. 
Define  {$\kappa = \frac{{\lambda_1^*}^2}{\eta\Delta^2}$}. Then, with probability at least $1-ct\delta$, we have
    \begin{align}
    \|\Thetav^t(\Thetav^t)^\top-\diag(\tilde{\Lambdav}_k,\mathbf{0})\|_F
    \leq
     {400}\kappa\sqrt{k}
    \Big(1-\frac{\eta\Delta}{16}\Big)^{t}.\nonumber
    \end{align}
\end{proposition}

\begin{proof}
First, by combining \Cref{lemma:lemma3,lemma:lemma4,lemma:lemma5}, we conclude that if $N$ satisfies \eqref{ineq:ineq31}, $\Thetav^\tau\in\Rc^\tau$ holds for all $\tau\leq t$ with probability at least $1-ct\delta$. 

Then, based on \Cref{lemma:lemma10}, if $N$ satisfies \Cref{ineq:ineq31}, with probability at least $1-2t\delta$, we have
\begin{align}
    {\sigma_1(\Thetav_{\text{res}}^{t})}
    \leq  
    \frac{2\sqrt{2}{\lambda_1^*}}{\sqrt{\Delta}}\Big(1-\frac{\eta\Delta}{16}\Big)^{t}.\label{ineq:63}
\end{align}
Define $\Dv^t = {\Thetav}_k^{t}({\Thetav}_k^{t})^\top-\tilde{\Lambdav}_k$. Based on \Cref{lemma:lemma12}, 
we have
\begin{align*}
    \sqrt{N}
    \geq
    \frac{1152  \sqrt{\dh-\log\delta} (\sqrt{k(\lambda_1^*)^2 +E}+\sqrt{k}\sigma_{\xi})}{\kappa^t \Delta\sqrt{c_1}}
    \geq
    \frac{1152\sqrt{\dh-\log\delta} {\lambda_1^*}(\sqrt{k(\lambda_1^*)^2 +E}+\sqrt{k}\sigma_{\xi})}{\sigma_{D}^{t+1}\Delta\sqrt{c_1}}.
\end{align*}
Under the same conditions in \Cref{ineq:ineq31}, it holds that
\begin{align}
    \sigma_1(\Dv^{t})
    \leq
    \frac{200{\lambda_1^*}^2}{\eta\Delta^2}
    \Big(1-\frac{\eta\Delta}{16}\Big)^{t}.\label{ineq:64}
\end{align}
By combining \Cref{ineq:63} and \Cref{ineq:64}, we have
\begin{align}
    \|\Thetav^t(\Thetav^t)^\top-\diag(\tilde{\Lambdav}_k,\mathbf{0})\|_F
    &\leq
    \sqrt{k}\|\Dv^{t}\|+2\sqrt{k\lambda_1^*}\|\Thetav_{\text{res}}^t\|\\\label{ineq:ineq99}
    &\leq
     {400}\max\Big\{\frac{\sqrt{k}{\lambda_1^*}^\frac{3}{2}}{\sqrt{\Delta}},\frac{\sqrt{k}{\lambda_1^*}^2}{\eta\Delta^2}\Big\}
    \Big(1-\frac{\eta\Delta}{16}\Big)^{t}.\nonumber
\end{align}
Note that the randomness {in $\{\Thetav^t\}_t$} comes from $\{\Rv^t\}_t$. If $N$ satisfies \Cref{ineq:ineq31}, we have \Cref{ineq:63} \Cref{ineq:64} and the event $\Thetav^\tau\in\Rc^\tau,\forall 0<\tau\leq t$ holds with probability at least $1-ct\delta$. Noting that
$\max\Big\{\frac{\sqrt{k}{\lambda_1^*}^\frac{3}{2}}{\sqrt{\Delta}},\frac{\sqrt{k}{\lambda_1^*}^2}{\eta\Delta^2}\Big\} = \frac{\sqrt{k}{\lambda_1^*}^2}{\eta\Delta^2}$, the proof is complete.
\end{proof}

\subsubsection{Putting All Together}\label{sec:proof-prop1}
Combining \Cref{thm:thm2-formal,thm:thm3-formal,thm:thm4-formal}, it is straightforward to show that if $t>T_{\Rc}$, $N$ satisfies $N\geq c_2\frac{(\dh-\log\delta+\log T)(\sqrt{k(\lambda_1^*)^2 +E}+\sqrt{k}\sigma_{\xi})^2}{\kappa_T^2\Delta^2}$ for some constant $c_2$, and $\Thetav^0$ satisfies the small random initialization condition stated in \Cref{thm:thm2-formal}, then it holds that $\|\Thetav^t(\Thetav^t)^\top-\diag(\tilde{\Lambdav}_k,\mathbf{0})\|_F\leq c_4\kappa\sqrt{k}\Big(1-\frac{\eta\Delta}{16}\Big)^{t}$ for all $t$ satisfies $T_\Rc<t<T$ with probability at least $1-\delta$, where $\kappa_T = (1-\frac{\eta\Delta}{16})^T$ and constant $c_4 =  {400}\Big(1-\frac{\eta\Delta}{16}\Big)^{-T_{\Rc}}$.  

Applying \Cref{lemma:bridge}, with probability at least $1-\delta$, we have
\begin{equation*}
    \begin{aligned}
        \|\tilde{\Bv}^t\tilde{\Wv}^t-\diag (\Lambdav_k,\mathbf{0})\|_F
        \leq
        c_4\kappa\sqrt{k}\Big(1-\frac{\eta\Delta}{16}\Big)^{t}.
    \end{aligned}
\end{equation*}
Note that
$ \|\Bv^t\Wv^t-\Bv^*\Wv^*\|_F= \|\tilde{\Bv}^t\tilde{\Wv}^t-\diag (\Lambdav_k,\mathbf{0})\|_F$. Combing with the fact that
$\sum_{i\in[M]}\|\Bv^t w_i^t-\Bv^*w_i^*\|^2= \|\Bv^t\Wv^t-\Bv^*\Wv^*\|_F^2$, we have
\begin{equation*}
    \begin{aligned}
    \sum_{i\in[M]}\|\Bv^t w_i^t-\Bv^*w_i^*\|^2 
    \leq
    c_4^2\kappa^2k\Big(1-\frac{\eta\Delta}{16}\Big)^{2t}.
    \end{aligned}
\end{equation*}
Then, applying the Cauchy-Schwarz inequality gives
\begin{equation*}
    \begin{aligned}
    \Big(\sum_{i\in[M]}\|\Bv^t w_i^t-\Bv^*w_i^*\|\Big)^2 
    \leq
    c_4^2M\kappa^2k\Big(1-\frac{\eta\Delta}{16}\Big)^{2t},
    \end{aligned}
\end{equation*}
which immediately implies that
\begin{equation*}
        \begin{aligned}
    \frac{1}{M}\sum_{i\in[M]}\|\Bv^t w_i^t-\Bv^*w_i^*\|
    \leq
    c_4\kappa\sqrt{\frac{k}{M}}\Big(1-\frac{\eta\Delta}{16}\Big)^{t}.
    \end{aligned}
\end{equation*}

\subsubsection{Auxiliary Lemmas}\label{sec:aux-lemmas}

\begin{lemma}[Concentration of $\|\Uv^\top_\star\Qv^{\tau+1}_\star\|$]\label{lemma:concentration-of-Q}
    For any $T\geq0$, assume $\Thetav^\tau\in\Rc_s$ holds for all $0<\tau\leq t$. Then, we have the following results for any $0\leq \tau\leq t$ and $c_2\geq0$ with probability at least $1-2\delta$: 
    \begin{itemize}
        \item If $\sqrt{N}\geq\max\big\{\frac{ \sqrt{\dh-\log\delta} }{\sqrt{c_1}}, \frac{192  \sqrt{\dh-\log\delta} \sqrt{\lambda_1^*}(\sqrt{k(\lambda_1^*)^2 +E}+\sqrt{k}\sigma_{\xi})}{\sigma\Delta\sqrt{c_1}}\big\}$, then it holds that
        \begin{align}
            \|\Uv^\top_\star\Qv^{\tau+1}_\star\|
            \leq
            \frac{\sigma\eta\Delta}{16\sqrt{2}}.
        \end{align}
        \item If $\sqrt{N}\geq\max\big\{\frac{ \sqrt{\dh-\log\delta} }{\sqrt{c_1}}, \frac{6144 \sqrt{\dh-\log\delta}\lambda_1^*(\sqrt{k(\lambda_1^*)^2 +E}+\sqrt{k}\sigma_{\xi})}{\Delta^2\sqrt{c_1}}\big\}$, then it holds that
        \begin{align}
            \|\Uv^\top_\star\Qv^{\tau+1}_\star\|  
            \leq
            \frac{\eta\Delta^2}{512\sqrt{2}\sqrt{\lambda^*_1}}.            
        \end{align}
    \end{itemize}
\end{lemma}

\begin{proof}
[Proof of \Cref{lemma:concentration-of-Q}]
Recall that $\Qv^{\tau+1}_\star$ is defined as
\begin{align}
    \Qv^{\tau+1}_\star 
    =
    \eta\sum_{i\in[M]}\big(\mathbf{B}^{\tau}_\star {w}_{i}^{\tau}-\phi_{\star i}\big)({w}_{i}^{\tau})^\top
    -
    \eta\sum_{i\in[M]}\frac{\Xv_{\star i}\Xv_{\star i}^\top}{N}\big(\mathbf{B}^{\tau}_{\star} {w}_{i}^{\tau}-\phi_{\star i}\big)({w}_i^{\tau})^\top
    +
    \eta\sum_{i\in[M]}\frac{\Xv_{\star i}\Ev_{i}({w}_i^{\tau})^\top}{N}
    ,
\end{align}
where $\phi_{\star i}$ and $\Xv_{\star i}$ are padded versions of $\phi_i$ and $\Xv_{i}$, respectively. 
To upper bound the norm of $\Qv_{\star}^{\tau+1}$, we decompose it into two parts:
\begin{align*}
    \frac{1}{\eta}\|\Qv^{\tau+1}_\star\|
    &\leq
    \underbrace{\norm{\sum_{i\in[M]}\big(\mathbf{B}^{\tau}_\star {w}_{i}^{\tau}-\phi_{\star i}\big)({w}_{i}^{\tau})^\top-\sum_{i\in[M]}\frac{\Xv_{\star i}\Xv_{\star i}^\top}{N}\big(\mathbf{B}^{\tau}_{\star} {w}_{i}^{\tau}-\phi_{\star i}\big)({w}_i^{\tau})^\top}}_{\Ac_{1}^{\tau+1}}
    +
    \underbrace{\norm{\sum_{i\in[M]}\frac{\Xv_{\star i}\Ev_{i}({w}_{i}^{\tau})^\top}{N}}}_{\Ac_{2}^{\tau+1}}.
\end{align*}
For $\Ac_1^{\tau+1}$,by applying Lemma 5.4 in \citet{vershynin2010introduction}, there exists a $\frac{1}{4}$-net $\mathcal{N}_k$ on the unit sphere $S^{k-1}$ and a $\frac{1}{4}$-net $\mathcal{N}_d$ on the unit sphere $S^{d-1}$ such that
\begin{align*}
    \Ac_1^{\tau+1}\leq
    2\max_{u\in\mathcal{N}_d,v\in\mathcal{N}_k}\mleft|
    \sum_{i\in[M]}\frac{1}{N}\sum_{j\in[N]}u^\top\big(\mathbf{B}^{\tau} {w}_i^{\tau}-\phi_i\big)({w}_i^{\tau})^\top v
    -
    \sum_{i\in[M]}\frac{1}{N}\sum_{j\in[N]}u^\top x_{i,j}x_{i,j}^\top\big(\mathbf{B}^{\tau} {w}_i^{\tau}-\phi_i\big)({w}_i^{\tau})^\top v
    \mright|.
\end{align*}
Denote $c^{\tau}_i = \|{\Bv}^{\tau}w_i^\tau-\phi_i\|$ and $c_w^{\tau} = \max_{i}\{\|w_i^\tau\|\}$. Observe that $u^\top x_{i,j}x_{i,j}^\top\big(\mathbf{B}^{\tau} {w}_i^{\tau}-\phi_i\big)({w}_i^{\tau})^\top v-u^{\top}(\Bv^\tau w_i^\tau-\phi_i)(w_i^{\tau})^\top v$ is a sub-exponential random variable with sub-exponential norm $c'c^\tau_i c^\tau_w$ for some constant $c'$, where \(c'\) depends on the distribution of \(x\). Then, based on the tail bound for sub-exponential random variables, 
there exists a constant $c_2'>0$ such that for any $s\geq0$, %it holds that
\begin{align*}
    &\mathbb{P}\mleft\{
    \frac{1}{N}\mleft(\sum_{i\in[M]}\sum_{j\in[N]}u^\top\big(\mathbf{B}^{\tau} {w}_i^{\tau}-\phi_i\big)({w}_i^{\tau})^\top v
    -
    \sum_{i\in[M]}\sum_{j\in[N]}u^\top x_{i,j}x_{i,j}^\top\big(\mathbf{B}^{\tau} {w}_i^{\tau}-\phi_i\big)({w}_i^{\tau})^\top v\mright)\geq s
    \mright\}\nonumber\\
    &\leq
    \text{exp}\Bigg(-Nc_2' \text{min}\Bigg(\frac{s^2}{\sum_{i\in[M]}(c^{\tau}_ic^{\tau}_w)^2},\frac{s}{\max_{i}\{c_i^{\tau}c^{\tau}_w\}}\Bigg)\Bigg).
\end{align*}
Taking the union bound over all $u\in\mathcal{N}_d$ and $v\in\mathcal{N}_k$, with probability at least $1-9^{d+k}\text{exp}\Big(-Nc_2' \text{min}\Big\{\frac{s^2}{\sum_{i\in[M]}(c^{\tau}_ic^{\tau}_w)^2},\frac{s}{\max_i\{c^{\tau}c^{\tau}_w\}}\Big\}\Big)$, we have
\begin{align*}
    \norm{\sum_{i\in[M]}\big(\mathbf{B}^{\tau} {w}_i^{\tau}-\phi_i\big)({w}_i^{\tau})^\top
    -
    \sum_{i\in[M]}\frac{\Xv_i\Xv_i^\top}{N}\big(\mathbf{B}^{\tau} {w}_i^{\tau}-\phi_i\big)({w}_i^{\tau})^\top}
    \leq
    2s.
\end{align*}
Since $\sigma_1^2(\Thetav^\tau)\leq2\lambda_1^*$, we have
\begin{align*}
\|(\tilde{\Bv}^\tau)^\top\tilde{\Bv}^\tau+\tilde{\Wv}^\tau(\tilde{\Wv}^\tau)^\top\|
    =\sigma_1^2(\Thetav^\tau)
    \leq
    2\lambda_1^*.
\end{align*}
Note that $(\tilde{\Bv}^\tau)^\top\tilde{\Bv}^\tau$ and $\tilde{\Wv}^\tau(\tilde{\Wv}^\tau)^\top$ are PSD matrices. It follows that
\begin{align*}
    \|\tilde{\Bv}^\tau\|\leq\sqrt{2\lambda_1^*}
    \quad\text{and}\quad
    \|\tilde{\Wv}^\tau\|\leq\sqrt{2\lambda_1^*},
\end{align*}
which implies that $\|\Bv^\tau\|\leq\sqrt{2\lambda_1^*}$ and $\|\Wv^\tau\|\leq\sqrt{2\lambda_1^*}$. Since $c^{\tau}_i = \|{\Bv}^{\tau}w_i^\tau-\phi_i\|$ and $c_w^{\tau} = \max_{i}\{\|w_i^\tau\|\}$,  we have $\sum_{i\in[M]}(c_i^{\tau}c_w^{\tau})^2\leq {2\lambda_1^*}\|\Bv^{\tau}\Wv^{\tau}-\Phiv\|_F^2\leq 4\lambda_1^*(k(\lambda_1^*)^2 +E)$ and $\max_{i}\{c_i^{\tau}c_w^{\tau}\}\leq3\sqrt{2}(\lambda_1^*)^{\frac{3}{2}}$, where $E = \sum_{i}(\lambda_i)^2$. 
Let $s=\sqrt{18\lambda_1^*(k(\lambda_1^*)^2 +E)}\cdot\sqrt{\log(1/\delta)/d+6}\cdot\sqrt{d}/\sqrt{Nc_2'}$. Then, if $N$ is sufficiently large such that $(\sqrt{\log(1/\delta)/d+2})\sqrt{d}/\sqrt{Nc_2'}\leq 1$, we have
\begin{equation*}
    \begin{aligned}
        \frac{s}{3\sqrt{2}(\lambda_1^*)^{\frac{3}{2}}}
        \geq
        \frac{s}{\sqrt{18\lambda_1^*(k(\lambda_1^*)^2 +E)}}
        \geq
        \frac{s^2}{18\lambda_1^*(k(\lambda_1^*)^2 +E)}.
    \end{aligned}
\end{equation*}
Then, with probability at least $1-\delta$, we have
\begin{align*}
    \norm{\sum_{i\in[M]}\big(\mathbf{B}^{\tau} {w}_i^{\tau}-\phi_i\big)({w}_i^{\tau})^\top
    -
    \sum_{i\in[M]}\frac{\Xv_i\Xv_i^\top}{N}\big(\mathbf{B}^{\tau} {w}_i^{\tau}-\phi_i\big)({w}_i^{\tau})^\top}
    \leq
    2\sqrt{18\lambda_1^*(k(\lambda_1^*)^2 +E)}\cdot\sqrt{\log(1/\delta)/d+6}\cdot\sqrt{\frac{d}{Nc}}.
\end{align*}
Therefore, with probability at least $1-\delta$, we have
\begin{align*}
    \Ac^{\tau+1}_1\leq \sqrt{\lambda_1^*\big(k(\lambda_1^*)^2 +E\big)}\frac{6\sqrt{2}\sqrt{d-\log\delta}}{\sqrt{Nc_2}},
\end{align*}
where $c_2 = \frac{c_2'}{6}$. 

Next, we consider $\Ac_2^{\tau+1}$. Similar to the above analysis, note that $u^\top x_{i,j} \xi_{i,j}(w_i^\tau)^\top v$ is a centered sub-exponential random variable with sub-exponential norm $c''\sigma_{\xi}\|w_i^\tau\|$ {for some constant $c''$}. Based on the tail bound for sub-exponential random variables, 
there exists a constant $c_3'>0$ such that for any $s\geq0$, 
\begin{align*}
    \mathbb{P}\mleft\{
    \sum_{i\in[M]}u^\top\frac{\Xv_{\star i}\Ev_{\star i}({w}_i^{\tau})^\top}{N}v\geq s
    \mright\}\nonumber
    \leq
    \text{exp}\Bigg(-Nc_3' \text{min}\Bigg(\frac{s^2}{\sigma^2_{\xi}\|\Wv^\tau\|_F^2},\frac{s}{\sigma_{\xi}\sqrt{2\lambda_1^*}}\Bigg)\Bigg).
\end{align*}
Combining with the fact $\|\Wv^\tau\|_F^2\leq 2k\lambda_1^* $ and taking the union bound over all $u\in\Nc_d$ and $v\in\Nc_k$, we have that inequality 
$\norm{\sum_{i\in[M]}\frac{\Xv_{\star i}\Ev_{\star i}^\tau({w}_i^{\tau})^\top}{N}}\leq 2s$ holds with probability at least $1-9^{d+k}\text{exp}\Big(-Nc \text{min}\Big\{\frac{s^2}{2\sigma^2_{\xi} k\lambda_1^*},\frac{s}{\sigma_{\xi}\sqrt{2\lambda_1^*}}\Big\}\Big)$. Then, let $s=\sqrt{2\sigma^2k\lambda_1^*}\cdot\sqrt{\log(1/\delta)/d+6}\cdot\sqrt{d}/\sqrt{Nc}$. If $N$ is sufficiently large such that $(\sqrt{\log(1/\delta)/d+2})\sqrt{d}/\sqrt{Nc_3'}\leq 1$ we have $\min\{\frac{s^2}{2\sigma^2_{\xi} k\lambda_1^*}, \frac{s}{\sigma_{\xi}\sqrt{2\lambda_1^*}}\} = \frac{s^2}{2\sigma^2_{\xi}k\lambda_1^*}$. Therefore, with a probability at least $1-\delta$, we have
\begin{align*}
    \Ac_2^{\tau+1}\leq 2 \sqrt{2\sigma^2_{\xi}k\lambda_1^*}\cdot\sqrt{\log(1/\delta)/d+6}\sqrt{\frac{d}{Nc_3'}}
    \leq
    (\lambda_1^*)^{\frac{1}{2}}\sigma_{\xi}\frac{6\sqrt{2}\sqrt{dk-k\log\delta}}{\sqrt{Nc_3}},
\end{align*}
where $c_3 = \frac{c_3'}{24}$. Combining the upper bounds of $\Ac^{\tau+1}_1$ and $\Ac^{\tau+1}_2$, we conclude that following inequality holds with probability at least $1-\delta$:
\begin{align*}
    \|\Uv^\top_\star\Qv^{\tau+1}_\star\|\leq\|\Qv^{\tau+1}_\star\|
    &\leq\eta (\Ac^{\tau+1}_1+\Ac^{\tau+1}_2)\\
    &\leq
    \eta(\lambda_1^*)^{\frac{1}{2}}\sqrt{k(\lambda_1^*)^2 +E}\frac{6\sqrt{2}\sqrt{d-\log(\delta/2)}}{\sqrt{Nc_2}}
    +
    \eta(\lambda_1^*)^{\frac{1}{2}}\sigma_{\xi}\frac{6\sqrt{2}\sqrt{dk-k\log(\delta/2)}}{\sqrt{Nc_3}}\\
    &\leq
    \eta\sqrt{\lambda_1^*}\mleft(\sqrt{k(\lambda_1^*)^2 +E}+\sqrt{k}\sigma_{\xi}\mright)\frac{6\sqrt{2}\sqrt{d-\log\delta}}{\sqrt{Nc_1}},
\end{align*}
where $c_1 = \frac{1}{2}\min\{c_2,c_3\}$.
Thus, for any $\sigma\geq0$, if 
$\sqrt{N}\geq \frac{192 \sqrt{d-\log\delta}\sqrt{\lambda_1^*}(\sqrt{k(\lambda_1^*)^2 +E}+\sqrt{k}\sigma_{\xi})}{\sigma\Delta\sqrt{c_1}}$,  with probability at least $1-\delta$ it holds
\begin{align*}
    \|\Uv^\top_\star\Qv^{\tau+1}_\star\|
    \leq
    \frac{\sigma\eta\Delta}{16\sqrt{2}}.
\end{align*}
Similarly, if $\sqrt{N}\geq \frac{6144\sqrt{d-\log\delta}\lambda_1^*(\sqrt{k(\lambda_1^*)^2 +E}+\sqrt{k}\sigma_{\xi})}{\Delta^2\sqrt{c_1}}$, with probability at least $1-\delta$, 
\begin{align*}
    \|\Uv^\top_\star\Qv^{\tau+1}_\star\|  
    \leq
    \frac{\eta\Delta^2}{512\sqrt{2}\sqrt{\lambda^*_1}}.
\end{align*}
\end{proof}

%-------------------------------------------------
\begin{lemma}[Concentration of $\|\tilde{\Qv}^{\tau+1}_\star\Vv_\star\|$]\label{lemma:concentration-of-tildeQ}
    For any $t\geq0$, assume $\Thetav^\tau\in\Rc_s$ holds for all $0<\tau\leq t$. Then, we have the following results for any $0\leq \tau\leq t$ and $c_2\geq0$ with probability at least $1-2\delta$:
    \begin{itemize}
        \item If $\sqrt{N}\geq\max\big\{\frac{ \sqrt{\dh-\log\delta} }{\sqrt{c_1}}, \frac{192  \sqrt{\dh-\log\delta} \sqrt{\lambda_1^*}(\sqrt{k(\lambda_1^*)^2 +E}+\sqrt{k}\sigma_{\xi})}{\sigma\Delta\sqrt{c_1}}\big\}$, then it holds that
        \begin{align*}
            \|\tilde{\Qv}^{\tau+1}_\star\Vv_\star\|
            \leq
            \frac{\sigma\eta\Delta}{16\sqrt{2}}.
        \end{align*}
        \item If $\sqrt{N}\geq\max\big\{\frac{ \sqrt{\dh-\log\delta} }{\sqrt{c_1}}, \frac{6144 \sqrt{\dh-\log\delta} \lambda_1^*(\sqrt{k(\lambda_1^*)^2 +E}+\sqrt{k}\sigma_{\xi})}{\Delta^2\sqrt{c_1}}\big\}$, then it holds that
        \begin{align*}
            \|\tilde{\Qv}^{\tau+1}_\star\Vv_\star\|
            \leq
            \frac{\eta\Delta^2}{512\sqrt{2}\sqrt{\lambda^*_1}}.            
        \end{align*}
    \end{itemize}
\end{lemma}

\begin{proof}
[Proof of \Cref{lemma:concentration-of-tildeQ}]
This proof resembles the proof of \Cref{lemma:concentration-of-Q}.
According to Lemma 5.4 in \citet{vershynin2010introduction}, there exists a $\frac{1}{4}$-net $\mathcal{N}_k$ on the unit sphere $S^{k-1}$ and a $\frac{1}{4}$-net $\mathcal{N}_M$ on the unit sphere $S^{M-1}$ so that
\begin{align*}
    \frac{1}{\eta}\|\tilde{\Qv}^{\tau+1}_\star\|
    &\leq
    2\max_{u\in\mathcal{N}_M,v\in\mathcal{N}_k}\mleft|
    \sum_{i\in[M]}v^\top\tilde{q}_i^{\tau+1}u_i
    +
    v^\top\frac{1}{N}\sum_{i\in[M]}(\Bv^\tau)^\top\Xv_{i}\Ev_{i}u_i
    \mright|\\
    &\leq
    \underbrace{
    2\max_{u\in\mathcal{N}_M,v\in\mathcal{N}_k}\mleft|
    \frac{1}{N}\mleft(\sum_{i\in[M]}\sum_{j\in[N]}
    v^\top
    (\Bv^{\tau})^\top\big(\mathbf{B}^{\tau} w_i^{\tau}-\phi_i\big)
    u_i
    -
    \sum_{i\in[M]}\sum_{j\in[N]}
    v^\top
    (\Bv^{\tau})^\top x_{i,j}x_{i,j}^\top\big(\mathbf{B}^{\tau} w_i^{\tau}-\phi_i\big)
    u_i\mright)
    \mright|}_{\Ac^{\tau+1}_3}\\
    &+
    \underbrace{
    2\max_{u\in\mathcal{N}_M,v\in\mathcal{N}_k}\mleft|
    v^\top\frac{1}{N}\sum_{i\in[M]}(\Bv^\tau)^\top\Xv_{i}\Ev_{i}u_i
    \mright|}_{\Ac^{\tau+1}_4}.
\end{align*}

Let $c^\tau_B = \|\Bv^\tau\|$ and recall that $c^\tau_i = \|\Bv^\tau w_i^\tau-\phi_i\|$. Based on the tail bound for sub-exponential random variables, 
there exists a constant $c>0$ such that for any $s\geq0$, 
\begin{align}
    &\mathbb{P}\mleft\{
    \frac{1}{N}\mleft(\sum_{i\in[M]}\sum_{j\in[N]}
    v^\top
    (\Bv^{\tau})^\top\big(\mathbf{B}^{\tau} w_i^{\tau}-\phi_i\big)
    u_i
    -
    \sum_{i\in[M]}\sum_{j\in[N]}
    v^\top
    (\Bv^{\tau})^\top x_{i,j}x_{i,j}^\top\big(\mathbf{B}^{\tau} w_i^{\tau}-\phi_i\big)
    u_i\mright)\geq s
    \mright\}\nonumber\\
    &\leq
    \text{exp}\Bigg(-Nc \text{min}\Bigg(\frac{s^2}{\sum_{i\in[M]}(c^{\tau}_ic^{\tau}_B)^2},\frac{s}{\max_{i}\{c_i^{\tau}c^{\tau}_B\}}\Bigg)\Bigg).\nonumber
\end{align}
Taking the union bound over all $u\in\mathcal{N}_M$ and $v\in\mathcal{N}_k$, with probability at least $1-9^{M+k}\text{exp}\Big(-Nc \text{min}\Big\{\frac{s^2}{\sum_{i\in[M]}(c^{\tau}_ic^{\tau}_B)^2},\frac{s}{\max_i\{c^{\tau}c^{\tau}_B\}}\Big\}\Big)$, we have $\Ac^{\tau+1}_3\leq 2s$. Let $s=(\lambda_1^*)^{\frac{1}{2}}\sqrt{k(\lambda_1^*)^2 +E}\cdot\sqrt{\log(1/\delta)/\dh+6}\cdot\sqrt{\dh}/\sqrt{Nc}$. If $N$ is sufficiently large such that $(\sqrt{\log(1/\delta)/d+2})\sqrt{\dh}/\sqrt{Nc}\leq 1$, there exits constant $c_2$ such that with probability at least $1-\delta$, we have
\begin{align*}
    \Ac_3^{\tau+1}
    \leq
    \eta(\lambda_1^*)^{\frac{1}{2}}
    \sqrt{k(\lambda_1^*)^2 +E}
    \frac{6\sqrt{2}\sqrt{\dh-\log\delta}}{\sqrt{Nc_2}}.
\end{align*}
For term $\Ac^{\tau+1}_{4}$, from the tail bound for sub-exponential random variables,
there exists a constant $c>0$ such that for any $s\geq0$, 
\begin{align*}
    \mathbb{P}\mleft\{
    \sum_{i\in[M]}v^\top\frac{(\Bv^\tau)^\top\Xv_{i}\Ev_{i}}{N}u_i\geq s
    \mright\}\nonumber
    &\leq
    \text{exp}\Bigg(-Nc \text{min}\Bigg(\frac{s^2}{\sigma^2_{\xi}\sum_{i\in[M]}\|u_i\Bv^\tau\|^2},\frac{s}{\sigma_{\xi}\sqrt{2\lambda_1^*}}\Bigg)\Bigg)\\
    &\leq
    \text{exp}\Bigg(-Nc \text{min}\Bigg(\frac{s^2}{2\sigma^2_{\xi}\lambda_1^*},\frac{s}{\sigma_{\xi}\sqrt{2\lambda_1^*}}\Bigg)\Bigg).
\end{align*}
Let $s=\sqrt{2\sigma^2_{\xi} k\lambda_1^*}\cdot\sqrt{\log(1/\delta)/\dh+6}\cdot\sqrt{\dh}/\sqrt{Nc}$. If $N$ is sufficiently large such that $(\sqrt{\log(1/\delta)/\dh+2})\sqrt{\dh}/\sqrt{Nc}\leq 1$, there exists constant $c_3$ such that with a probability at least $1-\delta$, we have
\begin{align*}
    \Ac_4^{\tau+1}
    \leq
    (\lambda_1^*)^{\frac{1}{2}}\sigma_{\xi}\frac{6\sqrt{2}\sqrt{\dh k-k\log\delta}}{\sqrt{Nc_3}}.
\end{align*}
 Combining the upper bounds of $\Ac^{\tau+1}_3$ and $\Ac^{\tau+1}_4$, we conclude that following inequality holds with probability at least $1-\delta$: 
\begin{align*}
    \|\tilde{\Qv}^{\tau+1}_\star\Vv_\star\|
    \leq
    \|\tilde{\Qv}^{\tau+1}_\star\|
    &\leq
    \eta (\Ac^{\tau+1}_3+\Ac^{\tau+1}_4)\\
    &\leq
    \eta(\lambda_1^*)^{\frac{1}{2}}\sqrt{k(\lambda_1^*)^2 +E}\frac{6\sqrt{2}\sqrt{\dh-\log(\delta/2)}}{\sqrt{Nc_2}}
    +
    \eta(\lambda_1^*)^{\frac{1}{2}}\sigma_{\xi}\frac{6\sqrt{2}\sqrt{\dh k- k\log(\delta/2)}}{\sqrt{Nc_3}}\\
    &\leq
    \eta\sqrt{\lambda_1^*}\mleft(\sqrt{k(\lambda_1^*)^2 +E}+\sqrt{k}\sigma_{\xi}\mright)\frac{6\sqrt{2}\sqrt{\dh-\log\delta}}{\sqrt{Nc_1}},
\end{align*}
where $c_1 = \frac{1}{2}\min\{c_2,c_3\}$. 
Thus, for any $\sigma\geq0$, if $\sqrt{N}\geq \frac{192 \sqrt{\dh-\log\delta}\sqrt{\lambda_1^*}(\sqrt{k(\lambda_1^*)^2 +E}+\sqrt{k}\sigma_{\xi})}{\sigma\Delta\sqrt{c_1}}$,  with probability at least $1-\delta$ it holds that
\begin{align*}
    \|\tilde{\Qv}^{\tau+1}_\star\Vv_\star\|
    \leq
    \frac{\sigma\eta\Delta}{16\sqrt{2}}.
\end{align*}
Similarly, if $\sqrt{N}\geq \frac{6144\sqrt{\dh-\log\delta}\lambda_1^*(\sqrt{k(\lambda_1^*)^2 +E}+\sqrt{k}\sigma_{\xi})}{\Delta^2\sqrt{c_1}}$, with probability at least $1-\delta$, 
\begin{align*}
    \|\tilde{\Qv}^{\tau+1}_\star\Vv_\star\|
    \leq
    \frac{\eta\Delta^2}{512\sqrt{2}\sqrt{\lambda^*_1}}.
\end{align*}

\end{proof}

%--------------------------------------------------

\begin{lemma}\label{lemma:upb-of-theta-res}
%Lemma A.5 in \citet{chen2023fast}.
   Suppose $\eta\leq 1/6\lambda^*_1$ and $\sigma_1(\Thetav^\tau)\leq\sqrt{2\lambda_1^*}$. Then, it holds that
   \begin{align*}
       \sigma_1(\Thetav^{\tau+1}_{\text{res}})
       \leq
       \Big(1+\frac{\eta}{2}\big(\lambda^*_{k+1}-\sigma_1^2(\Thetav^{\tau}_{\text{res}})-\sigma_k^2(\Thetav^{\tau}_{k})\big)\Big)\sigma_1(\Thetav^{\tau}_{\text{res}})+\sigma_1(\Rv^{\tau}_{2\dh-k}).
   \end{align*}
\end{lemma}

\begin{proof}
[Proof of \Cref{lemma:upb-of-theta-res}]\label{proof:proof-of-lemma-10}
According to \Cref{eqn:Theta-res}, we can rewrite $\Thetav^{\tau+1}_{\text{res}}$ as
\begin{align}
    \Thetav^{\tau+1}_{\text{res}} 
    &=
    \frac{1}{2}\Thetav^{\tau}_{\text{res}} 
    -
    \frac{\eta}{2}\Thetav^{\tau}_{\text{res}} (\Thetav^{\tau}_{\text{res}} )^\top\Thetav^{\tau-1}_{\text{res}} 
    +
    \Big(\frac{1}{4}\Iv_{2\dh-k}+\frac{\eta}{2}\tilde{\Lambdav} 
 _{\text{res}} \Big)\Thetav^{\tau}_{\text{res}} 
    +
    \Thetav_{\text{res}}^{\tau} \Big(\frac{1}{4}\Iv_k-\frac{\eta}{2}(\Thetav^{\tau}_{k} )^\top\Thetav^{\tau}_{k}\Big)
    +
    \Rv^{\tau+1}_{2\dh-k}.\label{ineq:ineq62}
\end{align}
From Lemma A.5 in \citet{chen2023fast} we have the following inequalities:
\begin{align}
 &   \sigma_{1}\Big(\frac{1}{2}\Thetav^{\tau}_{\text{res}} -\frac{\eta}{2}\Thetav^{\tau}_{\text{res}} (\Thetav^{\tau}_{\text{res}} )^\top\Thetav^{\tau}_{\text{res}}\Big)
    \leq
    \frac{1}{2}\sigma_1(\Thetav^{\tau}_{\text{res}})
    -
    \frac{\eta}{2}\sigma_1^3(\Thetav^{\tau}_{\text{res}}),\label{ineq:ineq63}\\
&    \sigma_1\Big(\Big(\frac{1}{4}\Iv_{2\dh-k}+\frac{\eta}{2}\Lambdav_{\text{res}} \Big)\Thetav^{\tau}_{\text{res}}\Big)
    \leq
    \big(\frac{1}{4}+\frac{\eta\lambda^*_{k+1}}{2}\big)\sigma_1(\Thetav^{\tau}_{\text{res}}),\label{ineq:ineq64}\\
 &   \sigma_1\Big(\Thetav^{\tau}_{\text{res}} \Big(\frac{1}{4}\Iv_k-\frac{\eta}{2}(\Thetav^{\tau}_{k} )^\top\Thetav^{t}_{k} \Big)\Big)
    \leq
    \sigma_1(\Thetav^{\tau}_{\text{res}})\big(\frac{1}{4}-\frac{\eta}{2}\sigma_k^2(\Thetav^{\tau}_{k})\big).\label{ineq:ineq65}
\end{align}
Substituting \Cref{ineq:ineq63,ineq:ineq64,ineq:ineq65} into \Cref{ineq:ineq62} proves the lemma.
\end{proof}

%--------------------------------------------------

\begin{lemma}
\label{lemma:lbd-of-theta-k}
    Suppose $\sigma_1(\Thetav^t)\leq\sqrt{2\lambda_1^*}$ and $\eta\leq\frac{\Delta^2}{16{\lambda^*_1}^3}$. Then, it holds that
    \begin{align*}
        \sigma_k^2(\Thetav^{\tau+1}_{k}) 
    \geq
    \Big(1+\eta\big(\lambda^*_k-\sigma_1^2(\Thetav^{\tau}_{\text{res}})-\sigma_k^2(\Thetav^{\tau}_{k})\big)-\frac{\eta\Delta}{4}\Big)\sigma_k^2(\Thetav^{\tau}_{k})
    -4\sqrt{\lambda^*_1}\sigma_1(\Rv^{\tau}_k).
    \end{align*}
\end{lemma}

\begin{proof}[Proof of \Cref{lemma:lbd-of-theta-k}]
Denote $\tilde{\Thetav}^{\tau+1}_{k} = \Thetav^{\tau}_{k}+\frac{\eta}{2}\tilde{\Lambdav}_k\Thetav^{\tau}_{k}-\frac{\eta}{2}\Thetav^{\tau}_{k}(\Thetav^{\tau})^\top\Thetav^{\tau}$. Based on Lemma A.6 and Lemma 2.3 in \citet{chen2023fast}, we have
    \begin{align}
        \sigma_k^2(\tilde{\Thetav}^{\tau+1}_{k})
        &\geq 
        \Big(1+\eta\big(\lambda^*_k-\sigma_1^2(\Thetav^{\tau}_{\text{res}})-\sigma_k^2(\Thetav^{\tau}_{k})\big)\Big)\sigma_k^2(\Thetav^{\tau}_{k})-{\eta}^2{\lambda_1^*}^3\nonumber\\
        &\geq
        \Big(1+\eta\big(\lambda^*_k-\sigma_1^2(\Thetav^{\tau}_{\text{res}})-\sigma_k^2(\Thetav^{\tau}_{k})\big)-\frac{\eta\Delta}{4}\Big)\sigma_k^2(\Thetav^{\tau}_{k}).\label{ineq:77}
    \end{align}
Combining with $\eta\leq \frac{1}{16\lambda^*_1}$ gives
\begin{align*}
    \sigma_1(\tilde{\Thetav}^{\tau+1}_{k})
    &\leq
    \sigma_1(\Thetav^{\tau}_{k})
    +
    \sigma_1\Big(\frac{\eta}{2}\tilde{\Lambdav}_k\Thetav^{\tau}_{k}\Big)
    +
    \sigma_1\Big(\frac{\eta}{2}\Thetav^{\tau}_{k}(\Thetav^{\tau})^\top\Thetav^{\tau}\Big)\\
    &\leq
    \sqrt{2\lambda_1^*}+\frac{1}{32}\sqrt{\lambda_1^*}+\frac{\sqrt{2}}{16}\sqrt{\lambda_1^*}\\
    &\leq
    2\sqrt{\lambda_1^*}.
\end{align*}
Thus, $\sigma_{{k}}(\tilde{\Thetav}^{\tau+1}_{k})\leq\sigma_1(\tilde{\Thetav}^{\tau+1}_{k})\leq 2\sqrt{\lambda^*_1}$.
Combining with the fact that $\sigma_k(\Thetav^{\tau+1}_{k}) 
    \geq
    \sigma_k(\tilde{\Thetav}_k^{\tau+1})-\sigma_1(\Rv^{\tau+1}_k)$, we have
\begin{align}
    \sigma_k^2(\Thetav^{\tau+1}_{k}) 
    &\geq
    \Big(\sigma_k(\tilde{\Thetav}_k^{\tau+1})-\sigma_1(\Rv^{\tau+1}_k)\Big)^2\nonumber\\
    &\geq
    \sigma_k^2(\tilde{\Thetav}_k^{\tau+1})-2\sigma_k(\tilde{\Thetav}_k^{\tau+1})\cdot
    \sigma_1(\Rv^{\tau+1}_k)\nonumber\\
    &\geq
    \sigma_k^2(\tilde{\Thetav}_k^{\tau})-4\sqrt{\lambda^*_1}
    \sigma_1(\Rv^{\tau}_k).\label{ineq:83}
\end{align}
The lemma thus follows by substituting \Cref{ineq:77} into \Cref{ineq:83}.
\end{proof}

%-------------------------------------------------------

\begin{lemma}\label{lemma:second-upb-for-res}
Assume $\eta\leq\frac{1}{6\lambda_1^*}$ and $\sigma_1(\Thetav^0)\leq \sqrt{2{\lambda_1^*}}$ hold. Then, if 
\begin{equation}
    \begin{aligned}
        \sqrt{N}\geq\max\Big\{\frac{ \sqrt{\dh-\log\delta} }{\sqrt{c_1}}, \frac{96 \sqrt{\lambda_1^*}(\sqrt{k(\lambda_1^*)^2 +E}+\sqrt{k}\sigma_{\xi})\sqrt{\dh-\log\delta}}{\sigma_1(\Thetav_{\text{res}}^0)\Delta\sqrt{c_1}}\Big\},\label{ineq:160}
    \end{aligned}
\end{equation}
with probability at least $1-ct\delta$ for some constant $c$, we have
\begin{align}
    \sigma_1(\Thetav^\tau_{\text{res}})
       \leq
       \Big(1+\frac{\eta}{2}\lambda^*_{k+1}+\frac{\eta}{8}\Delta\Big)^\tau\sigma_1(\Thetav^{0}_{\text{res}})\nonumber
\end{align} holds for all $\tau\leq t$.
\end{lemma}

%---------------------------------------
\begin{proof}
[Proof of \Cref{lemma:second-upb-for-res}]
   Suppose $\eta\leq 1/6\lambda^*_1$ and $\sigma_1(\Thetav^\tau)\leq\sqrt{2\lambda_1^*}$. We have
   \begin{align*}
       \sigma_1(\Thetav^\tau_{\text{res}})
       &\leq
       \Big(1+\frac{\eta}{2}\big(\lambda^*_{k+1}-\sigma_1^2(\Thetav^{\tau-1}_{\text{res}})-\sigma_k^2(\Thetav^{\tau-1}_{k})\big)\Big)\sigma_1(\Thetav^{\tau-1}_{\text{res}})+\sigma_1(\Rv^{\tau}_{2\dh-k})\\
       &\leq
       \Big(1+\frac{\eta}{2}\lambda^*_{k+1}\Big)\sigma_1(\Thetav^{\tau-1}_{\text{res}})+\sigma_1(\Rv^{\tau}).
   \end{align*}
Combining \Cref{lemma:concentration-of-Q} and \Cref{lemma:concentration-of-tildeQ}, when $N$ satisfies \Cref{ineq:160}, we have $\sigma_1(\Rv^t)\leq \frac{\eta}{8}\Delta\sigma_1(\Thetav_{\text{res}}^0)$. The lemma follows by induction.
\end{proof}

%--------------------------------------------

\begin{lemma}\label{lemma:bridge}
    If $\|\Thetav(\Thetav)^\top-\diag (\tilde{\Lambdav}_k,\mathbf{0})\|_F\leq\delta$ for some $\delta>0$, then 
    $\|\Bv\Wv - \diag(\Lambdav_k,\mathbf{0})\|_F\leq \delta$.
\end{lemma}
\begin{proof}[Proof of \Cref{lemma:bridge}]
    Note that $\|\Thetav(\Thetav)^T-\diag (\tilde{\Lambdav}_k,\mathbf{0})\|_F\leq\delta$ implies that
\begin{align*}
    \Big\|\Big(\frac{\Bv+\Wv^\top}{\sqrt{2}}\Big)\Big(\frac{\Bv+\Wv^\top}{\sqrt{2}}\Big)^\top-2\diag (\Lambdav_k,\mathbf{0})\Big\|_F\leq \delta,
\end{align*}
and
\begin{align*}
    \Big\|\Big(\frac{\Bv-\Wv^\top}{\sqrt{2}}\Big)\Big(\frac{\Bv-\Wv^\top}{\sqrt{2}}\Big)^\top\Big\|_F\leq \delta.
\end{align*}
Then, we have
\begin{align*}
   & \|\Bv\Wv+\Wv^\top\Bv^\top-2\diag(\Lambdav_k,\mathbf{0})\|_F\\
    &=
    \Big\|\Big(\frac{\Bv+\Wv^\top}{\sqrt{2}}\Big)\Big(\frac{\Bv+\Wv^\top}{\sqrt{2}}\Big)^\top-2\diag (\Lambdav_k,\mathbf{0})-
    \Big(\frac{\Bv-\Wv^\top}{\sqrt{2}}\Big)\Big(\frac{\Bv-\Wv^\top}{\sqrt{2}}\Big)^\top
    \Big\|_F\\
    &\leq
    \Big\|\Big(\frac{\Bv+\Wv^\top}{\sqrt{2}}\Big)\Big(\frac{\Bv+\Wv^\top}{\sqrt{2}}\Big)^\top-2\diag (\Lambdav_k,\mathbf{0})\Big\|_F+\Big\|\Big(\frac{\Bv-\Wv^\top}{\sqrt{2}}\Big)\Big(\frac{\Bv-\Wv^\top}{\sqrt{2}}\Big)^\top\Big\|_F\\
    &\leq 2\delta.
\end{align*}
    Combining with the fact $\|\Bv\Wv+\Wv^\top\Bv^\top-2\diag(\Lambdav_k,\mathbf{0})\|_F = 2\|\Bv\Wv-\diag(\Lambdav
 _k,\mathbf{0})\|_F$, the proof is complete.
\end{proof}

%--------------------------------------------

\begin{lemma}[Theorem 2.13 in~\citet{davidson2001local}]\label{lemma:lemma16}
    Let $N\geq n$ and $\Av$ be an $N\times n$ matrix whose entries are
    IID standard Gaussian random variables. Then, for any $\epsilon\geq0$, with probability at least $1-2\exp(-\epsilon^2/2)$, we have
    \begin{equation*}
        \begin{aligned}
            \sqrt{N}-\sqrt{n}-\epsilon
            \leq
            \sigma_{\min}(\Av)
            \leq
            \sigma_{\max}(\Av)
            \leq
            \sqrt{N}+\sqrt{n}+\epsilon.
        \end{aligned}
    \end{equation*}
\end{lemma}

\begin{lemma}[Eigenvalue Interlacing Theorem \citep{hwang2004cauchy}]\label{lemma:interlacing}
    For a symmetric matrix $\Av\in\Rb^{d\times d}$, let $\Bv\in\Rb^{k\times k}$, $k<d$, be a principal matrix of $\Av$. Denote the eigenvalues of $\Av$ as $\lambda_1\geq\cdots\geq\lambda_d$ and the eigenvalues of $\Bv$ as $\mu_1\geq\cdots\geq\mu_d$. Then, for any $i\in[k]$, it holds that
    \begin{equation*}
        \begin{aligned}
            \lambda_{i+d-k}\leq\mu_i\leq\lambda_i.
        \end{aligned}
    \end{equation*}
\end{lemma}

\if{0}
\subsection{Sample Complexity of \alg in Over-Parameterized Regime}
In this section, we discuss the sample complexity of \alg in the \textit{over-parameterized} regime, i.e., $\lambda_{k+1}=0$ and $\Delta = \lambda_k^*$.
We define $n$ as the number of samples participating in the training \textit{per client in each iteration} and define $\Ac^t = \|\Bv^t\Wv^t-\Phiv\|_F$. We start with a modified concentration result.

\begin{lemma}[Concentration lemma]\label{lemma:concentrations}
    For any $t \geq0$, assume $\Thetav^t\in\Rc_s$. Then, for constant $c_1$ and any $\psi>0$, if  
\begin{align}
    \sqrt{n}\geq\max\Big\{\frac{\sqrt{\log(1/\delta)+d}}{\sqrt{c_1}}
    ,\quad 
    \frac{\sqrt{\lambda_1^*}\sqrt{\log(1/\delta)+d}}{\psi\sqrt{c_1}}\Big\},
\end{align}
with probability at least $1-\delta$, we have
$$
\|\Uv^\top_\star\Qv^{t+1}_\star\|
\leq
\psi\Ac^t,\qquad
\|\tilde{\Qv}_\star^{t+1}\Vv_\star\|
\leq
\psi\Ac^t.
$$

\end{lemma}

\begin{proof}
[Proof of lemma \ref{lemma:concentrations}]
From \Cref{lemma:concentration-of-Q}, we have following inequality holds with probability at least $1-\delta$
\begin{align}
    \frac{1}{\eta}\|\Qv^{t+1}_\star\|
    \leq
    2(\sqrt{\log(1/\delta)/d+2})c^tc_w^t\sqrt{\frac{d}{nc_1}}.
\end{align}
Recall that $c_t = \max_i\{\|\Bv^t w_i^t-\phi_i\|\}$ and $c_w^t = \max_i\{\|w_i^t\|\}$. Then, we have $c_t\leq\|\Bv^t\Wv^t-\Phiv_i\|_F$ and $c_w^t \leq \sqrt{2\lambda_1^*}$.
For constant $c_1$ and any $\psi>0$, if 
\begin{align}
    \sqrt{n}\geq\max\big\{\frac{\sqrt{\log(1/\delta)+d}}{\sqrt{c_1}}, 
    \frac{\sqrt{\lambda_1^*}\sqrt{\log(1/\delta)+d}}{\psi\sqrt{ c_1}}\big\},\label{ineq:ineq153}
\end{align}
with probability at least $1-\delta$, we have
\begin{align}
\|\Uv^\top_\star\Qv^{t+1}_\star\|
\leq
\psi\Ac^t.
\end{align}

Similarly, for $\tilde{\Qv}_\star^{t+1}$, if $n$ satisfies \Cref{ineq:ineq153}, then with probability at least $1-\delta$, we have
\begin{align}
    \|\tilde{\Qv}_\star^{t+1}\Vv_\star\|\leq\psi\Ac^t.
\end{align}
\end{proof}

Our goal is to show that $\Ac^t\leq A (1-\frac{\Delta\eta}{16})^t$ for constant $\Ac$, and we prove it by induction. First, define the following event
\begin{align}
    \Ec_t = \Big\{\forall 0<\tau\leq t\Big| \|\tilde{\Qv}_\star^{\tau+1}\Vv_\star\|\leq\psi\Ac^\tau,\quad
    \|\Uv^\top_\star\Qv^{\tau+1}_\star\|\leq\psi\Ac^\tau
    ,\quad
    \Thetav^\tau\in\Rc
    \Big\}.
\end{align}
We assume that event $\Ec_T$ is true. Also assume that $\Ac^t\leq A (1-\frac{\Delta\eta}{16})^t$ holds for all $0<t\leq T$.
Based on \Cref{ineq:ineq99}, we have $\Ac^t\leq \sqrt{k}\|\Dv^t\|+2\sqrt{k\lambda_1^*}\|\Thetav_{\text{res}}^t\|$. Now consider $\sigma_1(\Thetav_{\res}^t)$. From \Cref{lemma:upb-of-theta-res}, for $\Thetav^t\in\Rc$ and $\eta\leq 1/6\lambda^*_1$, we have
\begin{align}
    \sigma_1(\Thetav^{t+1}_{\text{res}})
       &\leq
       \Big(1+\frac{\eta}{2}\big(\lambda^*_{k+1}-\sigma_1^2(\Thetav^{t}_{\text{res}})-\sigma_k^2(\Thetav^{t}_{k})\big)\Big)\sigma_1(\Thetav^{t}_{\text{res}})+\sigma_1(\Rv^{t}_{2\dh-k})\\
       &\leq
       \Big(1-\frac{\eta\Delta}{8} \Big)\sigma_1(\Thetav^t_{\text{res}})+\sqrt{2}\psi\Ac^t\\
       &\leq
       \Big(1-\frac{\eta\Delta}{8} \Big)\sigma_1(\Thetav^t_{\text{res}})+\sqrt{2}\psi A \Big(1-\frac{\eta\Delta}{16}\Big)^t.
\end{align}

Let \rp{$\psi\leq\frac{\eta\Delta\sigma_1(\Thetav_{\text{res}}^0)}{16\sqrt{2} A }$}. By induction we have $\sigma_1(\Thetav_{\text{res}}^t)\leq\sigma_1(\Thetav_{\text{res}}^0)(1-\eta\Delta/16)^t$. 

Now consider $\sigma_1(\Dv^t)$. Based on \Cref{equ:equ83}, we have
\begin{align}
    \Dv^t = \tilde{\Thetav}_k^{t}(\tilde{\Thetav}_k^{t})^\top-\tilde{\Lambdav}_k   +\tilde{\Thetav}_k^{t}(\Rv^{t}_k)^\top+\Rv^{t}_k(\tilde{\Thetav}_k^{t})^T +
    \Rv^{t}_k(\Rv^{t}_k)^\top.
\end{align}
Combining with the proof in Section A.3 of \citet{chen2023fast}, we have
\begin{align}
    \sigma_{1}(\Dv^{t})
    &\leq
    (1-\frac{\eta\Delta}{8})\sigma_{1}(\Dv^{t-1})+
    \sigma_1^2(\Thetav_{\res}^{ t-1})+ 
    \sigma_1(\tilde{\Thetav}_k^{t}(\Rv^{t}_k)^\top+\Rv^{t}_k(\tilde{\Thetav}_k^{t})^T +
    \Rv^{t}_k(\Rv^{t}_k)^\top),
\end{align}

Let \rp{$\psi\leq \frac{\lambda_1^*}{ A }$}, we have 
\begin{align}
    \sigma_{1}(\Dv^{t})
    &\leq
    (1-\frac{\eta\Delta}{8})\sigma_{1}(\Dv^{t-1})+
    \sigma_1^2(\Thetav_{\res}^{ t-1})
    + 
    3\lambda_1^*\sigma_1(\Rv^t)\\
    &\leq
    (1-\frac{\eta\Delta}{8})\sigma_{1}(\Dv^{t-1})+
    \lambda_k^*\sigma_1(\Thetav_{\res}^{0})\Big(1-\frac{\eta\Delta}{16}\Big)^{t-1}
    + 
    3\lambda_1^*\sigma_1(\Rv^t)\\
    &\leq
    (1-\frac{\eta\Delta}{8})\sigma_{1}(\Dv^{t-1})+
    \lambda_k^*\sigma_1(\Thetav_{\res}^{0})\Big(1-\frac{\eta\Delta}{16}\Big)^{t-1}
    + 
    3\sqrt{2}\lambda_1^*\psi\Ac^{t-1}.
\end{align}

Let \rp{$\psi\leq \frac{\eta\Delta\sigma_1(\Dv^0)}{48\sqrt{2}\lambda_1^* A }$}. By induction, we have 
\begin{align}
    \sigma_{1}(\Dv^{t})
    \leq
    \Big(\sigma_1(\Dv^0)+\frac{16\lambda_k^*}{\eta\Delta}\sigma_1^{2}(\Thetav^0_{\text{res}})\Big)(1-\eta\Delta/16)^t
\end{align}

Then, for $t = T+1$, we have
\begin{align}
    \Ac^{T+1}
    &\leq \sqrt{k}\|\Dv^{T+1}\|+\sqrt{2k\lambda_1^*}\|\Thetav_{\text{res}}^{T+1}\|\\
    &\leq
    \sqrt{k}\Big(\sigma_1(\Dv^0)+\frac{16}{\eta\Delta}\sigma_1^{2}(\Thetav^0_{\text{res}})\Big)\Big(1-\frac{\eta\Delta}{16}\Big)^{T+1}
    +
    2\sqrt{k\lambda_1^*}\sigma_1(\Thetav_{\text{res}}^0)\Big(1-\frac{\eta\Delta}{16}\Big)^{T+1}\\
    &\leq
    A\Big(1-\frac{\eta\Delta}{16}\Big)^{T+1}
\end{align}
where the last inequality holds by setting 
\begin{align}
    A = \max\Big\{\sqrt{k}\Big(\sigma_1(\Dv^0)+\frac{16}{\eta\Delta}\sigma_1^{2}(\Thetav^0_{\text{res}})\Big)+ 2\sqrt{k\lambda_1^*}\sigma_1(\Thetav_{\text{res}}^0)
    ,\quad
    \Ac^0\Big\}.
\end{align}
Then we proved \jingc{?}that $\Ac^t\leq A(1-\eta\Delta/16)^t$. Denote \begin{align}
    \psi = \min\Big\{
    \frac{\eta\Delta\sigma_1(\Thetav_{\text{res}}^0)}{16\sqrt{2} A },\;
    \frac{\lambda_1^*}{ A },\;
    \frac{\eta\Delta\sigma_1(\Dv^0)}{48\sqrt{2}\lambda_1^* A },\;
    \sqrt{\Delta}
    \Big\}.
\end{align} 
Combining \Cref{thm:thm3-formal} with \Cref{lemma:concentrations}, 
if $\psi\leq\sqrt{\Delta}$, 
\begin{align}
    \sqrt{n}\geq c_2\frac{(\lambda_1^*)^2\sqrt{\log(1/\delta)+d}}{\psi\Delta^{\frac{3}{2}}},
\end{align}
for some constant $c_2$, then, with probability at least $1-c_3t\delta$, we have 
\begin{align}
    \Ac^t\leq c_4\sqrt{M}\Big(1-\frac{\eta\Delta}{16}\Big)^{t},
\end{align}
where $c_3$ is some constant independent with $d$, $k$, $M$, $n$, $t$ and $\delta$. 

Note that $\psi\sqrt{M}$ is also a constant independent with $d$, $k$, $M$, $n$, $t$ and $\delta$. Therefore, there exists constant $c_4$ such that for $n\geq c_4\frac{k(\log(t)-\log(\delta)+d)}{M}$, with probability at least $1-\delta$, we have
\begin{align}
    \|\Bv^t\Wv^t-\Phiv\|_F\leq c_3\sqrt{k}\Big(1-\frac{\eta\Delta}{16}\Big)^{t}.
\end{align}
This is equivalent to
\begin{equation}
        \begin{aligned}
    \frac{1}{M}\sum_{i\in[M]}\|\Bv^t w_i^t-\Bv^*w_i^*\|
    \leq
    c_3\sqrt{\frac{k}{M}}\Big(1-\frac{\eta\Delta}{16}\Big)^{t},
    \end{aligned}
\end{equation}

then for $n$
\fi

\section{General \alg}\label{appx:general}
\subsection{Details of General \alg}
\begin{algorithm}[H]
\caption{General \alg}\label{alg:general-flute}
\begin{algorithmic}[1]
\STATE {\bf Input:} Learning rates $\eta_{l}$ and $\eta_{r}$, regularization parameters $\lambda_1$, $\lambda_2$ and $\lambda_3$, communication round $T$
\STATE {\bf Initialization:} {Server initializes model parameters $\Bv^0, \{b_i^0\}, \{\Hv_i^0\}$}
\FOR{$t=\{0,\cdots,T-1\}$}
% \jingc{strange to start with $t=0$}\rpc{starting from $1$ will induce so many $t+1$ superscript in the algorithm, looks messy. }
\STATE Server samples a batch of clients $\Ic^{t+1}$
\STATE Server sends $\Bv^{t}$, and $\Hv_i^{t}$ to all client $i\in\Ic^{t+1}$
% \For{client $i$ in $\Ic^t$}
% \State Server send $\Bv^{t-1}$, and $\Hv_i^{t-1}$ to client $i$
% \EndFor
\FOR{client $i\in [M]$ in parallel}
\IF{$i\in\Ic^{t+1}$}
\STATE $\Bv^{t,0}_i\leftarrow \Bv^{t}$, $b_i^{t,0}\leftarrow b_i^{t}$ and $\Hv_i^{t,0}\leftarrow \Hv_i^{}$
\FOR{$\tau=\{0,\cdots,\Tc-1\}$}
\STATE $\Hv_i^{t,\tau+1} \leftarrow \text{GRD}(\Lc_i(\Bv_i^{t, \tau},b_i^{t,\tau},\Hv_i^{t,\tau});\Hv_i^{t,\tau},\eta_l)$
\STATE $b_i^{t,\tau+1} \leftarrow \text{GRD}(\Lc_i(\Bv_i^{t,\tau},b_i^{t,\tau},\Hv_i^{t,\tau});b_i^{t,\tau},\eta_l)$
\STATE $\Bv_i^{t,\tau+1}  \leftarrow \text{GRD}(\Lc_i(\Bv_i^{t,\tau},b_i^{t,\tau},\Hv_i^{t,\tau});\Bv^{t,\tau},\eta_l)$
% \begin{align*}
% \Hv_i^{t,\tau+1} &\leftarrow \text{GRD}(\Lc_i(\Bv_i^{t, \tau},b_i^{t,\tau},\Hv_i^{t,\tau});\Hv_i^{t,\tau},\eta_l)\\
% b_i^{t,\tau+1} &\leftarrow \text{GRD}(\Lc_i(\Bv_i^{t,\tau},b_i^{t,\tau},\Hv_i^{t,\tau});b_i^{t,\tau},\eta_l)\\
% \Bv_i^{t,\tau+1} & \leftarrow \text{GRD}(\Lc_i(\Bv_i^{t,\tau},b_i^{t,\tau},\Hv_i^{t,\tau});\Bv^{t,\tau},\eta_l)
% \end{align*}
\ENDFOR
\STATE $\Bv_i^{t+1}\leftarrow \Bv_i^{t,\Tc}$, $b_i^{t+1}\leftarrow b_i^{t,\Tc}$ and $\Hv_i^{t+1}\leftarrow \Hv_i^{t,\Tc}$
\STATE Sends $\Bv^{t+1}_i$, $b_i^{t+1}$ and $\Hv_i^{t+1}$ to the server
\ELSE
\STATE $b_i^{t+1}\leftarrow b_i^{t}$
\ENDIF
\ENDFOR
% \FOR{client $i\notin\Ic^{t+1}$ in parallel}
% \STATE $b_i^{t+1}\leftarrow b_i^{t}$
% \ENDFOR
\STATE  Server updates:
\STATE \quad$\Bv^{t+1} = \frac{1}{rM}\sum_{i\in\Ic^{t+1}}\Bv_i^{t+1}$
\STATE \quad$\{\Hv_i^{t+1}\}_{i\in\Ic^{t+1}} \leftarrow \text{GRD}(R(\{\Hv_i^{t+1}\}_{i\in\Ic^{t+1}}, \Bv^{t+1});\{\Hv_i^{t+1}\}_{i\in\Ic^{t+1}},\eta_r)$
\STATE \quad$\Hv_i^{t+1} \leftarrow \Hv_i^{t}, \forall i\notin\Ic^{t+1}$
% \begin{align*}
% &\Bv^{t+1} = \frac{1}{rM}\sum_{i\in\Ic^{t+1}}\Bv_i^{t+1}\\
% & \{\Hv_i^{t+1}\}_{i\in\Ic^{t+1}} \leftarrow \text{GRD}(R(\{\Hv_i^{t+1}\}_{i\in\Ic^{t+1}}, \Bv^{t+1});\{\Hv_i^{t+1}\}_{i\in\Ic^{t+1}},\eta_r)\\
% &\Hv_i^{t+1} \leftarrow \Hv_i^{t}, \forall i\notin\Ic^{t+1}
% \end{align*}
\ENDFOR
\end{algorithmic}
\end{algorithm}

% \congc{I rewrote \Cref{alg:general-flute} wrt the client loop. Double check to confirm.}
The General \alg is presented in \Cref{alg:general-flute}, where $\text{GRD}(f;\theta,\alpha)$ denotes the update of variable $\theta$  using the gradient of the function $f$ with respect to $\theta$ and the step size $\alpha$.
The local loss function $\Lc_i$ is defined as
\begin{align}
    \Lc_i(\Bv,b,\Hv) 
    =
    \frac{1}{N}\sum_{(x,y)\in\Dc_i}\Lc\big(\Hv^\top f_{\Bv}(x)+b, y\big).
\end{align}

In this work, we instantiate the general \alg by a federated multi-class classification problem. In this case, the local loss function is specialized as
\begin{align}\label{obj:gen-fedup}
\Lc_i(\Bv,b,\Hv) 
    =
\frac{1}{N}\sum_{(x,y)\in\Dc_i}
\Lc_{\text{CE}}\big(\Hv_i^\top f_{\Bv}(x)+b_i, y\big)+\lambda_1\|f_{\Bv}(x)\|_2^2+\lambda_2\|\Hv_i\|_F^2+\lambda_3\Nc\Cc_i(\Hv_i),
\end{align}
where $y\in\Rb^m$ is a one-hot vector whose $k$-th entry is $1$ if the corresponding $x$ belongs to class $k$ and $0$ otherwise, and $\lambda_1$, $\lambda_2$ and $\lambda_3$ are non-negative regularization parameters. $\Lc_{\text{CE}}(\cdot)$ is the cross-entropy loss, where {for a one-hot vector $y$ whose $k$-th entry is $1$, we have}:
\begin{align}
    \Lc_{\text{CE}}(\hat{y},y) = -\log\mleft(\frac{\exp(\hat{y}_k)}{\sum_{i\in[c]}\exp(\hat{y}_i)}\mright).
\end{align}

% \st{where $y$ is a one-hot vector where its $k$-th entry of $y$ is $1$.}
% \congc{doesn't make sense...}. 
$\Nc\Cc_i(\Hv_i)$, inspired by the concept of neural collapse~\citep{Papyan_2020},  is defined as
\begin{align}
    \Nc\Cc_i(\Hv_i)
    =
    \norm{\frac{\Hv^\top_i\Hv_i}{\|\Hv^T_i\Hv_i\|_F} -
    \frac{1}{\sqrt{m-1}}\mathbf{u}_i\mathbf{u}_i^\top\odot\mleft(\Iv_{m}-\frac{1}{m}\mathbf{1}_m\mathbf{1}^\top_{m}\mright)}_F ,
\end{align}
where $\mathbf{u}_i$ is an $m$-dimensional one-hot vector whose $c$-th entry is $1$ if $c\in\Cc_i$ and $0$ otherwise. Also, we specialize the regularization term optimized on the server side as $R(\{\Hv_i\}) = \sum_{i}\Nc\Cc_i(\Hv_i)$.
% \begin{align}
%     R(\{\Hv_i\}) = \sum_{i}\Nc\Cc_i(\Hv_i).
% \end{align}

\subsection{Additional Definition}
\begin{definition}[$k$-Simplex ETF, Definition 2.2 in~\citet{tirer2022extended}]
\label{def:ETF}
The standard simplex equiangular tight frame (ETF) is a collection of points in $\Rb^k$ specified by the columns of 
\begin{equation}
    \begin{aligned}
        \Mv=\sqrt{\frac{k}{k-1}}\mleft(\Iv_k-\frac{1}{k}\mathbf{1}_k\mathbf{1}_k^\top\mright).
    \end{aligned}
\end{equation}
Consequently, the standard simplex EFT obeys
\begin{equation}
    \begin{aligned}
        \Mv^\top\Mv = \Mv\Mv^\top=\frac{k}{k-1}\mleft(\Iv_k-\frac{1}{k}\mathbf{1}_k\mathbf{1}_k^\top\mright).
    \end{aligned}
\end{equation}
\end{definition}
In this work, we consider a (general) simplex ETF as a collection of points in $\Rb^d$, $d\geq k$ specified by the columns of $\tilde{\Mv}\propto \sqrt{\frac{k}{k-1}}\Pv\mleft(\Iv_k-\frac{1}{k}\mathbf{1}_k\mathbf{1}_k^\top\mright)$, where $\Pv\in\Rb^{d\times k}$ is an orthonormal matrix. Consequently, $\tilde{\Mv}^\top\tilde{\Mv}\propto \Mv\Mv^\top=\frac{k}{k-1}\mleft(\Iv_k-\frac{1}{k}\mathbf{1}_k\mathbf{1}_k^\top\mright)$.

\subsection{More Discussion on General \alg}\label{appx:nc}

Firstly, we explain the concept of \emph{neural collapse}.

\mypara{Neural collapse.} Neural collapse (NC) was experimentally identified in~\citet{Papyan_2020}, and they outlined four elements in the neural collapse phenomenon: 
\begin{itemize}
    \item (NC1) Features learned by the model (output of the representation layers) for samples within the same class tend to converge toward their average, essentially causing the within-class variance to diminish;
    \item (NC2) When adjusted for their overall average, the final means of different classes display a structure known as a simplex equiangular tight frame (ETF);
    \item (NC3) The weights of the final layer, which serves as the classifier, align with this simplex ETF structure;
    \item (NC4) Consequently, after this collapse occurs, classification decisions are made based on measuring the nearest class center in the feature space. 
\end{itemize}
   
Next, we discuss some observations on the vanilla multi-classification problem, i.e., no additional regularization term and no client-side optimization, which is given as
\begin{align}\label{def:vanilla-def}
\Lc_i(\Bv,b,\Hv) 
    =
\frac{1}{N}\sum_{(x,y)\in\Dc_i}
\Lc_{\text{CE}}\big(\Hv_i^\top f_{\Bv}(x)+b_i, y\big)+\lambda_1\|f_{\Bv}(x)\|_2^2+\lambda_2\|\Hv_i\|_F^2.
\end{align}
The first observation, which directly comes from Theorem 3.2 in~\citet{tirer2022extended}, describes the phenomena of local neural collapse, which could happen when the model is locally trained for long epochs. 

\begin{observation}\label{obs:local-nc}
When $f_{\Bv}(\cdot)$ is sufficiently expressive such that $f_\Bv(x)$ can be viewed as a free variable. and the feature dimension $k$ is no smaller than the number of total classes $m$, locally learned $\Bv$ and $\Hv_i$ that optimize the objective function \eqref{def:vanilla-def} must satisfy:
\begin{align}
 &   f_{\Bv^*}(x_1) = f_{\Bv^*}(x_2),\quad\forall x_1,x_2\in\Dc_i^c,  c\in \Cc_i\\
 &   \frac{\langle f_{\Bv^*}(x),h_{i,c}^*\rangle}{\|f_{\Bv^*}(x)\|\cdot\|h_{i,c}^*\|}=1,
    \quad\forall x\in\Dc_i^c, c\in \Cc_i\\
&    \frac{\Hv^\top_i\Hv_i}{\|\Hv^\top_i\Hv_i\|_F} 
    =
    \frac{1}{\sqrt{m'-1}}\mathbf{u}_i\mathbf{u}_i^\top\odot\mleft(\Iv_{m}-\frac{1}{m'}\mathbf{1}_m\mathbf{1}^\top_{m}\mright),
\end{align}
where $\mathbf{u}_i$ is a $m$-dimensional one-hot vector whose $c$-th entry is $1$ if $c\in\Cc_i$ and $0$ otherwise, and $m'$ is the number of classes per client.
\end{observation}

The above observation states that NC1, NC2, and NC3 happen locally, implying: 1) $h_{i,c}=\mathbf{0}$ if $c\notin \Cc_i$; and 2) the sub-matrix of $\Hv_i$ constructed by columns $h_{i,c}$ with $c\in \Cc_i$ will form a K-Simplex ETF (c.f. Definition \ref{def:ETF}) up to some scaling and rotation. We conclude that if there exist $\Bv$ and $\Hv_1,\cdots,\Hv_M$ such that they are the optimal models for all clients, then
the data from the same class may be mapped to different points in the feature space by $f_{\Bv^*}$ when data are drawn from different clients. However, this condition usually cannot be satisfied in the under-parameterized regime, due to the less expressiveness of the under-parameterized model.

To further demonstrate the phenomenon in the under-parameterized regime, we assume that in the under-parameterized regime, a well-performed representation $f_\Bv$ should map data from the same class but different clients to the same feature mean:
\begin{condition}\label{cond:under-para}
    For client $i$ and $j$, if class $c\in\Cc_i$ and $c\in\Cc_j$, then $\frac{1}{|\Dc_i^c|}\sum_{x:(x,y)\in\Dc_i^c}f_{\Bv}(x)=\frac{1}{|\Dc_j^c|}\sum_{x:(x,y)\in\Dc_j^c}f_{\Bv}(x)$.
    % \begin{align}
    %     \frac{1}{|\Dc_i^c|}\sum_{x:(x,y)\in\Dc_i^c}f_{\Bv}(x)=\frac{1}{|\Dc_j^c|}\sum_{x:(x,y)\in\Dc_j^c}f_{\Bv}(x)
    % \end{align}
\end{condition}
With this condition, we have the following observation that also comes from Theorem 3.1 in~\citet{zhu2021geometric}, which describes the neural collapse in the under-parameterized regime. 
\begin{observation}\label{obs:global-nc}
When \Cref{cond:under-para} holds and the feature dimension $k$ is no smaller than the number of total classes $m$, any global optimizer $\Bv^*, \Hv_1^*, \cdots, \Hv_M^*$ of \eqref{def:vanilla-def} satisfies
\begin{align}
  &  f_{\Bv^*}(x_1) = f_{\Bv^*}(x_2),\quad\forall x_1,x_2\in\Dc_i^c, i\in[M], c\in \Cc_i,\\
  &  \frac{\langle f_{\Bv^*}(x),h_{i,c}^*\rangle}{\|f_{\Bv^*}(x)\|\cdot\|h_{i,c}^*\|}=1,
    \quad\forall x\in\Dc_i^c, i\in[M], c\in \Cc_i,\\
   & \frac{\Hv^\top_i\Hv_i}{\|\Hv^\top_i\Hv_i\|_F} 
    =
    \frac{1}{\sqrt{m-1}}\mathbf{u}_i\mathbf{u}_i^\top\odot\mleft(\Iv_{m}-\frac{1}{m}\mathbf{1}_m\mathbf{1}^\top_{m}\mright), \quad \forall i\in[M],
\end{align}
where $\mathbf{u}_i$ is a $m$-dimensional one-hot vector whose $c$-th entry is $1$ if $c\in\Cc_i$ and $0$ otherwise.
\end{observation}
%\clearpage

Comparing these two observations, we conclude that in the under-parameterized case, the optimal models $h_{i,c}$ and $h_{j,c}$ are of the same direction when class $c$ is included in both $\Cc_i$ and $\Cc_j$. It implies that the globally optimized model performs differently compared with the locally learned model.  In \Cref{fig:nc2}, we present an example to illustrate how $\Hv$ performs differently when it is globally or locally optimized. %\congc{Reorganize figure 2 as 2 rows and 3 columns, and use the whole page width. Change the description of the next paragraph accordingly.}

In \Cref{fig:nc2}, we consider the scenario that the number of clients $M=3$, total number of data classes $m=3$, number of data classes per client $m'=2$, client $1$ contains data of class $1$ and class $2$, client $2$ contains data of class $1$ and class $3$, and client $3$ contains data of class $2$ and class $3$. The first row of the three sub-figures shows the structure of normalized columns of $\Hv_1$, $\Hv_2$, and $\Hv_3$ when they are locally optimized, and the second row of the three sub-figures shows those optimize \eqref{def:vanilla-def}. We observe that under this setting, the locally optimized heads are in opposite directions, which perform differently compared with the global optimal heads.

Inspired by such observations, we add $\mathcal{NC}_i$ to the local loss function and also optimize $R(\{\Hv_i\})$, to ensure that the personalized heads also contribute to the global performance. This principle aligns with our motivation to design the linear \alg.
\begin{figure}[h]
 \centering
\includegraphics[width=0.8\linewidth]{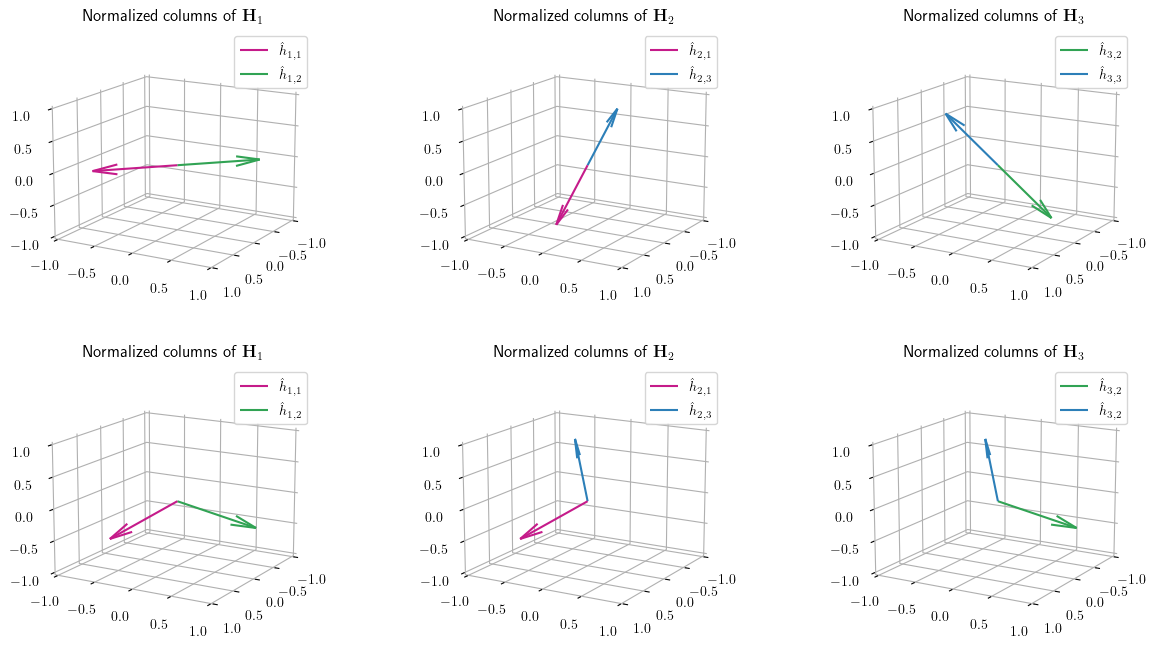} 
\caption{Behavior of locally optimized heads and globally optimized heads.}
 \label{fig:nc2}
\end{figure}

\if{0}
\begin{multicols}{2}
\begin{figure}[H]
 \begin{subfigure}{1\linewidth}
  \includegraphics[width=\textwidth]{figs/nc2-local-global.png}
 \end{subfigure}
 \caption{\small Structures of the optimal $\Hv$.}
\label{fig:nc2}
\end{figure}
\columnbreak
In Fig. \ref{fig:nc2}, we consider the scenario that the number of clients $M=3$, total number of data classes $m=3$, number of data classes per client $m'=2$, client $1$ contains data of class $1$ and class $2$, client $2$ contains data of class $1$ and class $3$ and client $3$ contains data of class $2$ and class $3$. The first column of the three sub-figures shows the structure of normalized columns of $\Hv_1$, $\Hv_2$, and $\Hv_3$ when optimizing objective function \eqref{obj:multi-classification} in over-parameterized case, and the second column of the three sub-figures shows those in under-parameterized case. We observe that under this
setting, the normalized optimal prediction head $\hat{h}_{i,c}$ of class $c$, i.e., the normalized $c$-th column of $\Hv_i$, is varying across different clients, which is contradicted by the under-parameterized case, such that $\hat{h}_{i,c}$ of the same class $c$ is identical across different clients.
\end{multicols}
\fi
% \clearpage
%------------------------------------

\section{Additional Experimental Results}\label{apdx:experiment}

% \congc{You gave many figures without any explanations in this section. What is the point of presenting them? Do you expect the reviewers to find interesting results by themselves? If you don't have points to make with these figures, you should remove them. This is not about ``checking the box''. It is about stating your points that are supported by experimental results.}

\subsection{Synthetic Datasets}\label{apx:c1-syn}

\mypara{Implementation Details.}
In the experiments conducted on synthetic datasets shown in \cref{fig:fig2-syn}, $\Lambdav \in \Rb^{\underline{d} \times \underline{d}}$ is generated by setting the $i$-th singular value to be $\frac{2\underline{d}}{i+1}$. We randomly generate $\Uv \in \Rb^{d \times \underline{d}}$ with $\underline{d}$ orthonormal columns and $\Vv \in \Rb^{\underline{d} \times M}$ with $\underline{d}$ orthonormal rows. The ground-truth model is then $\Phiv = \Uv \Lambdav \Vv^\top$, where each column $\phi_i$ represents the local ground-truth model for client $i$. Each client generates $N$ samples $(x, y)$ from $y = x^\top \phi_i + \xi_i$, where $x$ is sampled from a standard Gaussian distribution and every entry of $\xi_i$ is IID sampled from $\Nc(0, 0.3)$. The learning rate is set to $\eta = 0.03$, and for random initialization, we set $\alpha = \frac{1}{10d}$.

\mypara{Parameter Settings.}
For experiments on synthetic datasets shown in \cref{fig:fig2-syn}, we set $d=10$. We select the value of $k$ from the set $\{2, 4, 6, 8\}$, $M$ from the set $\{15, 30\}$, and $N$ from the set $\{12, 20\}$.

\mypara{Experimental Results.}
From the experiments in \cref{fig:fig2-syn}, we observe that, with the dimensions $d$, $M$, and $N$ fixed, an increase in $k$ results in a diminishing discrepancy in convergence speeds between \alg and FedRep. This trend demonstrates \alg's superior performance in under-parameterized settings. Furthermore, keeping $d$, $k$, and $N$ unchanged while increasing the number of clients $M$, we see a reduction in the average error of models generated by \alg. This observation aligns with our theoretical findings presented in \Cref{prop:prop1}.

\mypara{Varying $\gamma_1$ and $\gamma_2$.}
In \cref{fig:varying-gamma}, we report the results of the following experiments where $d=10$, $k=6$, $M=10$, and $N$ selected from the set $\{8, 9, 10, 11\}$. For comparison, we use three pairs of $\gamma_1$ and $\gamma_2$: $\gamma_1 = 2\gamma_2$, $\gamma_1 = \gamma_2$, and $\gamma_1 = \frac{2}{3}\gamma_2$. We do not set $\gamma_1 > 2\gamma_2$ because in this setting, $\|\mathbf{B}\mathbf{W}\|_F$ usually diverges. From the experimental results, we observe that when $N=8$, $9$, or $10$, $\gamma_1 = \gamma_2$ shows the best performance among the three settings of $\gamma_1$ and $\gamma_2$.

\begin{figure}[H]
 \begin{subfigure}{1\linewidth}
    % \hspace{\linewidth}
     \includegraphics[width=\textwidth]{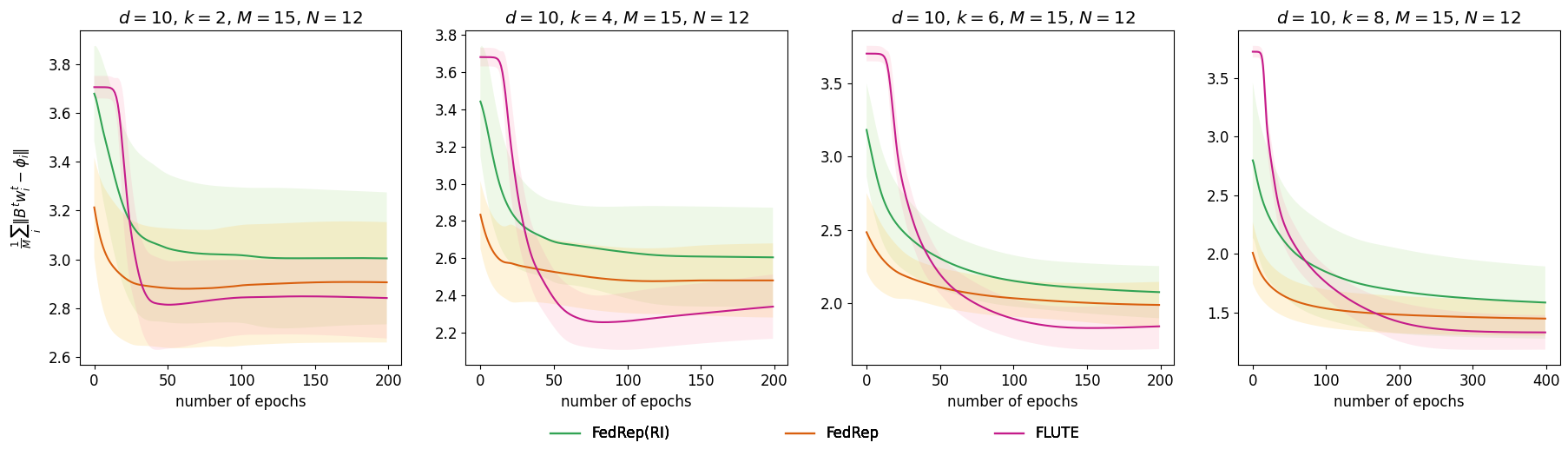}
 \end{subfigure}
 \begin{subfigure}{1\linewidth}
 % \hspace{\linewidth}
     \includegraphics[width=\textwidth]{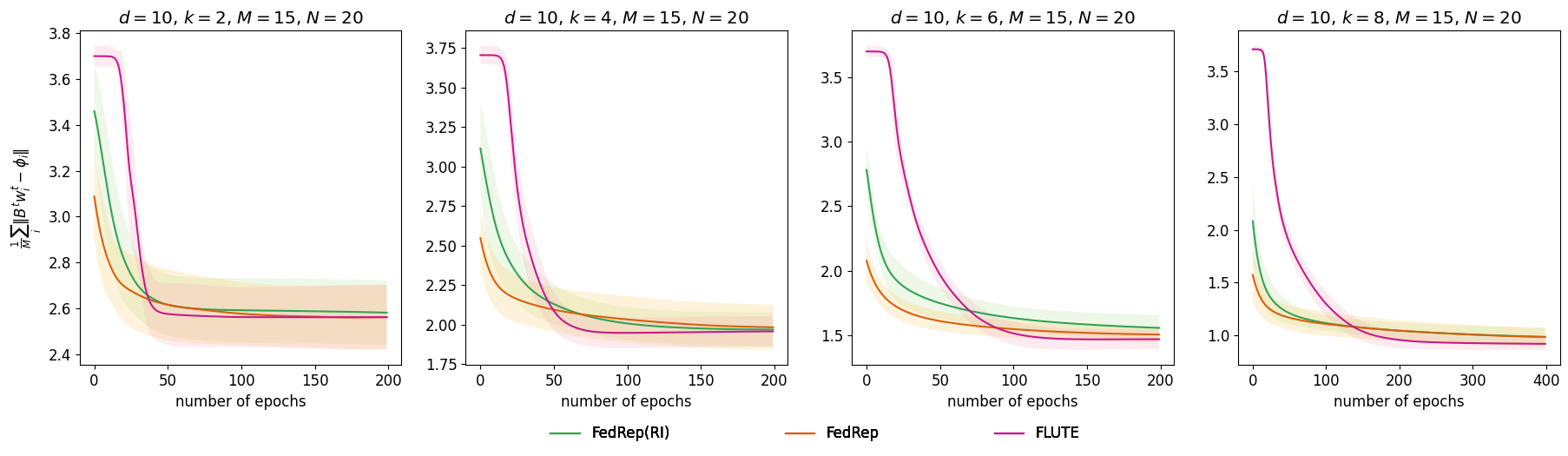}
 \end{subfigure}
 \begin{subfigure}{\linewidth}
 % \hspace{\linewidth}
     \includegraphics[width=\textwidth]{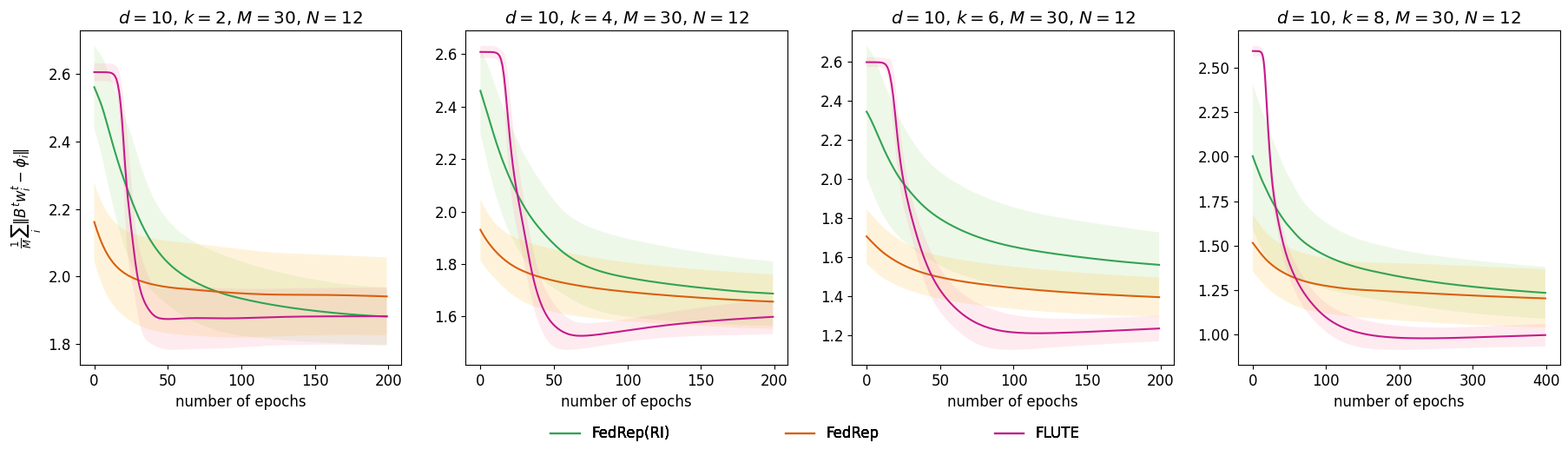}
 \end{subfigure}
 \begin{subfigure}{1\linewidth}
 % \hspace{\linewidth}
     \includegraphics[width=\textwidth]{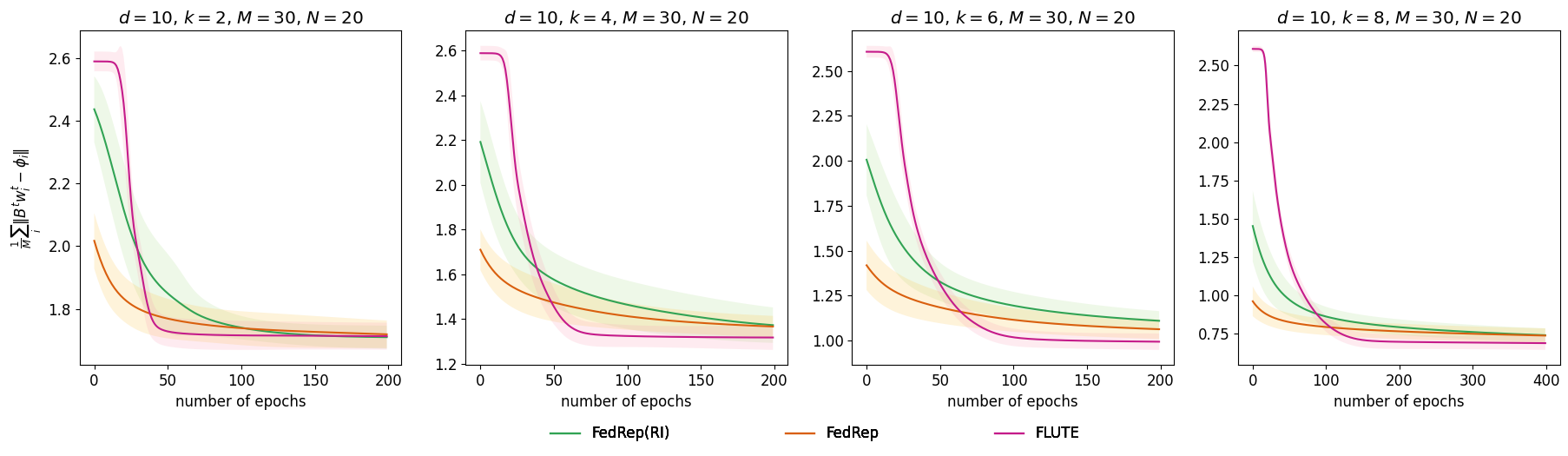}
 \end{subfigure}
 \caption{Experimental results with synthetic datasets.}
 \label{fig:fig2-syn}
\end{figure}

% \begin{figure}[h]
% \begin{subfigure}{\linewidth}
%  \centering
%     \hspace{-0.05\linewidth}
%      \includegraphics[width=0.4\textwidth]{expirments/reviewer2-q4.png}
%  \end{subfigure}
%  \vspace{-0.2in}
%  \caption{Experimental results with varying \(\gamma_1\), \(\gamma_2\).}
%  \label{fig:varying-gamma}
%      \vspace{-0.2in}
% \end{figure}

\begin{figure}[H]
 \begin{subfigure}{1\linewidth}
 % \hspace{\linewidth}
     \includegraphics[width=\textwidth]{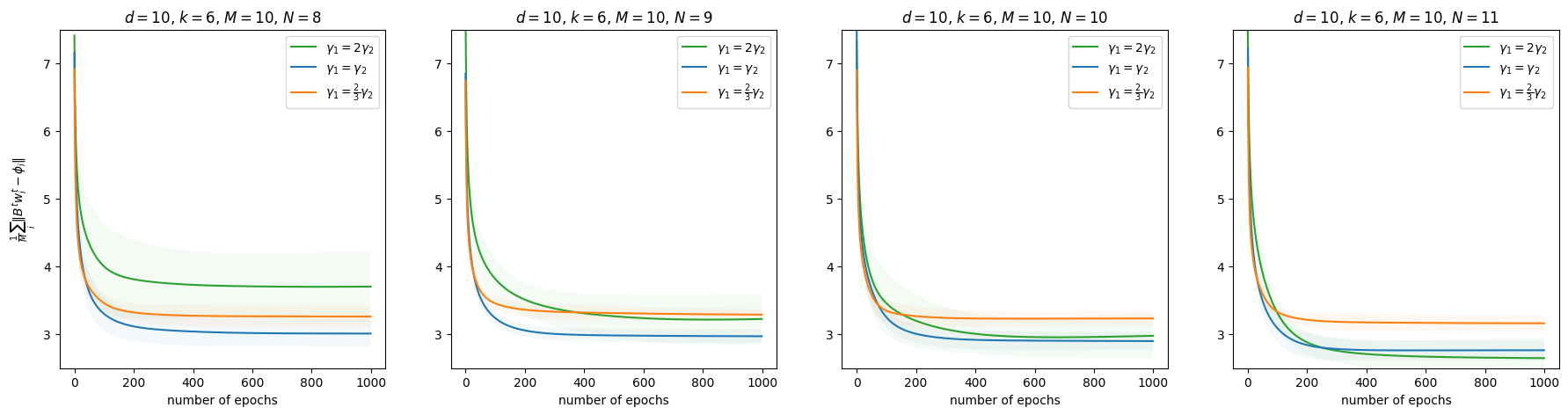}
 \end{subfigure}
 \caption{Experimental results with synthetic datasets.}
 \label{fig:varying-gamma}
\end{figure}

%----------------------------------------
% \newpage
\subsection{Real-world Datasets}
\mypara{Implementation Details.}
For our experiments on the CIFAR-10 dataset, we employ a 5-layer CNN architecture. It begins with a convolutional layer \texttt{Conv2d(3, 64, 5)}, followed by a pooling layer \texttt{MaxPool2d(2, 2)}. The second convolutional layer is \texttt{Conv2d(64, 64, 5)}, which precedes three fully connected layers: \texttt{Linear(64*5*5, 120)}, \texttt{Linear(120, 64)}, and \texttt{Linear(64, 10)}. In contrast, for the CIFAR-100 dataset, we also use a 5-layer CNN, but with some modifications to accommodate the higher complexity of the dataset. The initial layer is \texttt{Conv2d(3, 64, 5)}, followed by pooling and dropout layers: \texttt{MaxPool2d(2, 2)} and \texttt{nn.Dropout(0.6)}. The subsequent convolutional layer is \texttt{Conv2d(64, 128, 5)}. This is succeeded by three fully connected layers: \texttt{Linear(128*5*5, 256)}, \texttt{Linear(256, 128)}, and \texttt{Linear(128, 100)}.

% \mypara{Implementation Details.}
% For our experiments on the CIFAR-10 dataset, we employ a 5-layer CNN architecture. It begins with a convolutional layer \texttt{Conv2d(3, 64, 5)}, followed by a pooling layer \texttt{MaxPool2d(2, 2)}. The second convolutional layer is \texttt{Conv2d(64, 64, 5)}, which precedes three fully connected layers: \texttt{Linear(64*5*5, 120)}, \texttt{Linear(120, 64)}, and \texttt{Linear(64, 10)}.
% In contrast, for the CIFAR-100 dataset, we also use a 5-layer CNN, but with some modifications to accommodate the higher complexity of the dataset. The initial layer is \texttt{Conv2d(3, 64, 5)}, followed by pooling and dropout layers: \texttt{MaxPool2d(2, 2)} and \texttt{nn.Dropout(0.6)}. The subsequent convolutional layer is \texttt{Conv2d(64, 128, 5)}. This is succeeded by three fully connected layers: \texttt{Linear(128*5*5, 256)}, \texttt{Linear(256, 128)}, and \texttt{Linear(128, 100)}.

\mypara{Experimental Results.}
In this section, we plot 
\cref{fig:real1}  to \cref{fig:real8} 
% \cref{fig:real1}, \cref{fig:real2}, \cref{fig:real3}, \cref{fig:real4}, \cref{fig:real5}, \cref{fig:real6}, \cref{fig:real7} and \cref{fig:real8} 
to illustrate the detailed convergence behavior of the test accuracy of the trained models reported in \cref{tab:table-cifar} as a function of the training epochs. We augment the test accuracy results by introducing two different metrics.
The first one is \textit{Global NC2}, which is measured by 
\[ \frac{1}{M} \sum_{i \in [M]} \left\| \frac{\Hv^\top_i \Hv_i}{\|\Hv^\top_i \Hv_i\|_F} - \frac{1}{\sqrt{m-1}} \mathbf{u}_i \mathbf{u}_i^\top \odot \left( \Iv_{m} - \frac{1}{m} \mathbf{1}_m \mathbf{1}^\top_{m} \right) \right\|_F. \]
The second one is \textit{Averaged Local NC2}, referred to as 
\[ \frac{1}{M} \sum_{i \in [M]} \left\| \frac{\Hv^\top_i \Hv_i}{\|\Hv^\top_i \Hv_i\|_F} - \frac{1}{\sqrt{m'-1}} \mathbf{u}_i \mathbf{u}_i^\top \odot \left( \Iv_{m} - \frac{1}{m'} \mathbf{1}_m \mathbf{1}^\top_{m} \right) \right\|_F. \] 

These two metrics are inspired by \Cref{obs:global-nc} and \Cref{obs:local-nc}, respectively. \textit{Global NC2} aims to measure the similarity between the learned models and the optimal under-parameterized global model. In contrast, \textit{Averaged Local NC2} assesses the similarity between the learned models and the optimal local models. Note that these two metrics are positively correlated, meaning that when one is small, the other is usually small as well. In some results, such as those shown in \Cref{fig:real1} and \Cref{fig:real3}, the gaps between FedRep* and \algstar in terms of \textit{Averaged Local NC2} are significantly larger than those in terms of \textit{Global NC2}, suggesting that the models learned by \algstar are closer to the global optimizer than those learned by FedRep*.

% Our code will be available at \url{https://github.com/RenpuLiu/FLUTE}. 
% \congc{No need to say this. You should upload the code with the supp PDF.}

% \subsubsection{Additional Experiment Results.}

\begin{figure}[H]
 \begin{subfigure}{0.32\linewidth}
     \includegraphics[width=\textwidth]{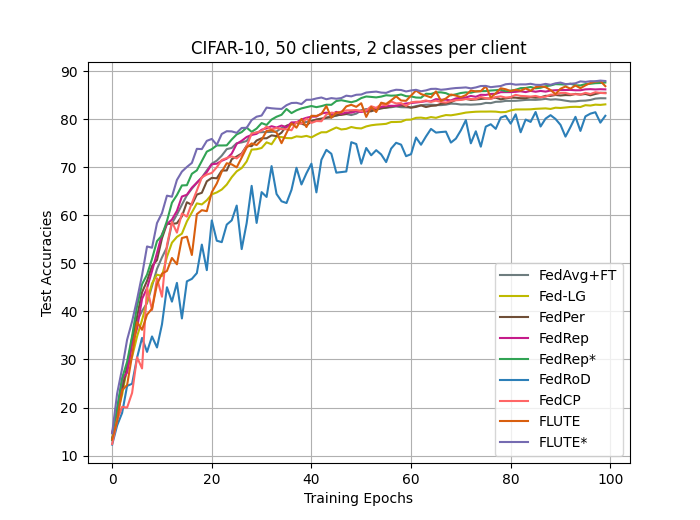}
 \end{subfigure}
\hfill
 \begin{subfigure}{0.32\linewidth}
     \includegraphics[width=\textwidth]{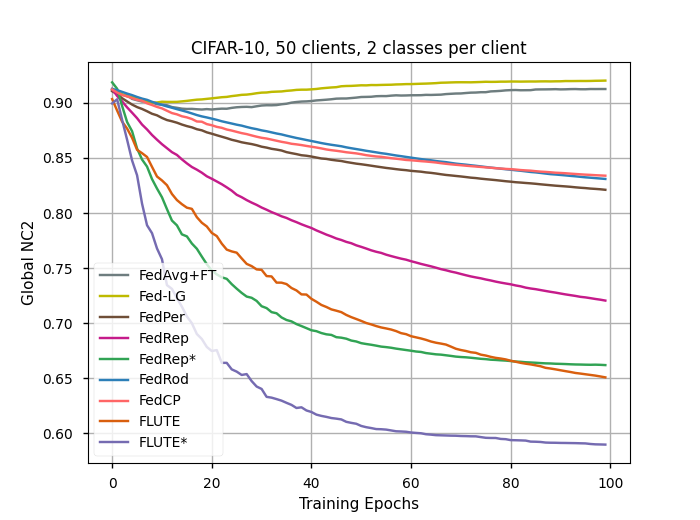}
 \end{subfigure}
 \hfill
  \begin{subfigure}{0.32\linewidth}
     \includegraphics[width=\textwidth]{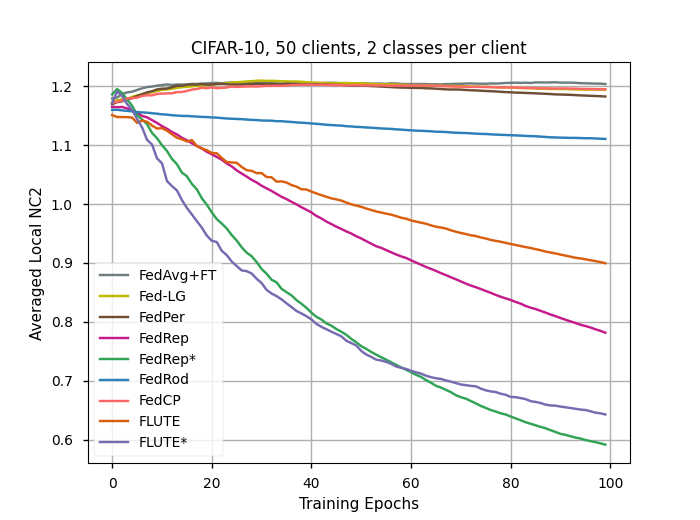}
 \end{subfigure}
 \caption{\small Experimental results for CIFAR10 when $M=50, m'=2$.}
\label{fig:real1}
\end{figure}
%-------------------------------
\begin{figure}[H]
 \begin{subfigure}{0.32\linewidth}
     \includegraphics[width=\textwidth]{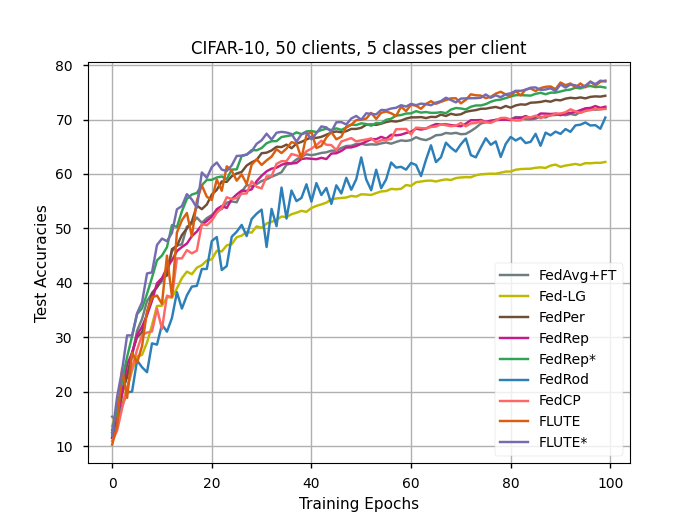}
 \end{subfigure}
\hfill
 \begin{subfigure}{0.32\linewidth}
     \includegraphics[width=\textwidth]{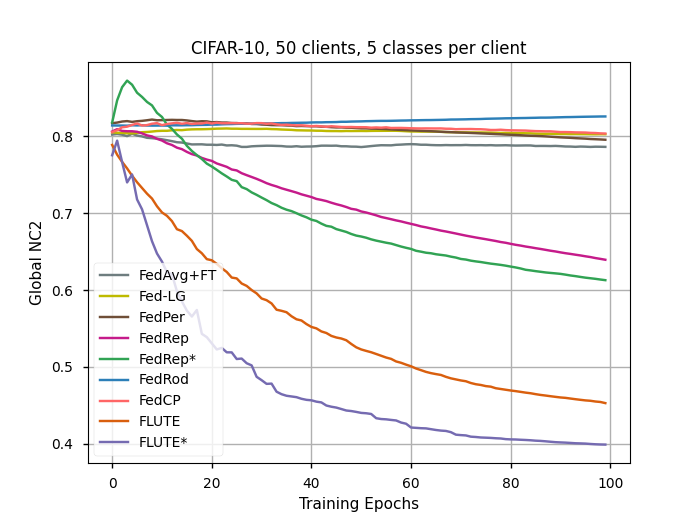}
 \end{subfigure}
 \hfill
  \begin{subfigure}{0.32\linewidth}
     \includegraphics[width=\textwidth]{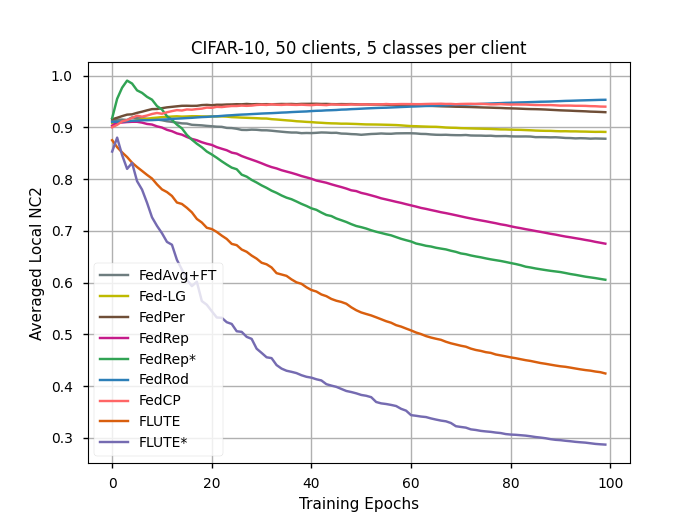}
 \end{subfigure}
 \caption{\small Experimental results for CIFAR10 when $M=50, m'=5$.}
\label{fig:real2}
\end{figure}
%-------------------------------
\begin{figure}[H]
 \begin{subfigure}{0.32\linewidth}
     \includegraphics[width=\textwidth]{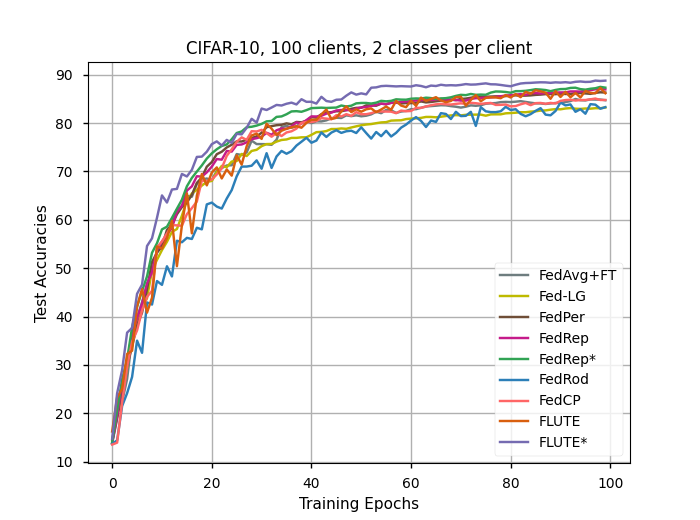}
 \end{subfigure}
\hfill
 \begin{subfigure}{0.32\linewidth}
     \includegraphics[width=\textwidth]{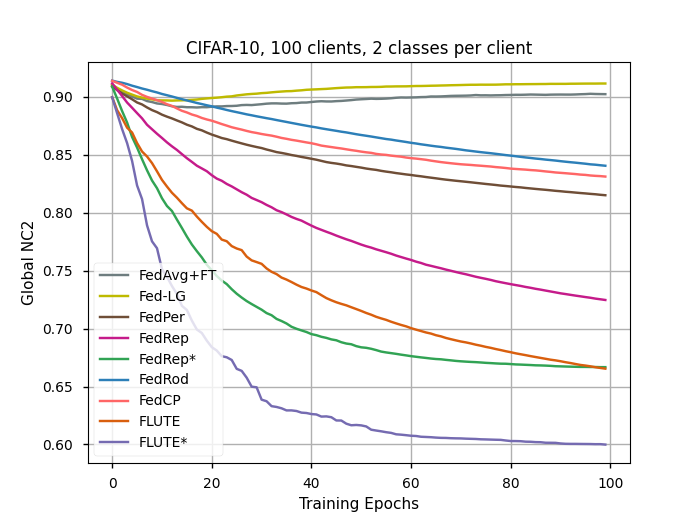}
 \end{subfigure}
 \hfill
  \begin{subfigure}{0.32\linewidth}
     \includegraphics[width=\textwidth]{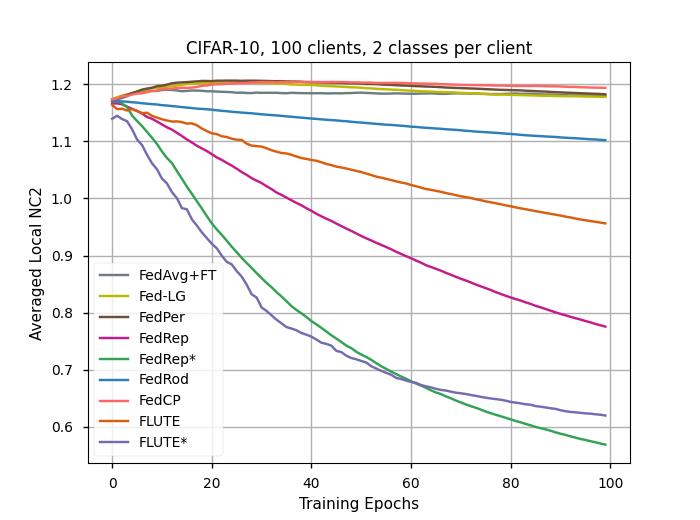}
 \end{subfigure}
 \caption{\small Experimental results for CIFAR10 when $M=100, m'=2$.}
\label{fig:real3}
\end{figure}
%-------------------------------
\begin{figure}[H]
 \begin{subfigure}{0.32\linewidth}
     \includegraphics[width=\textwidth]{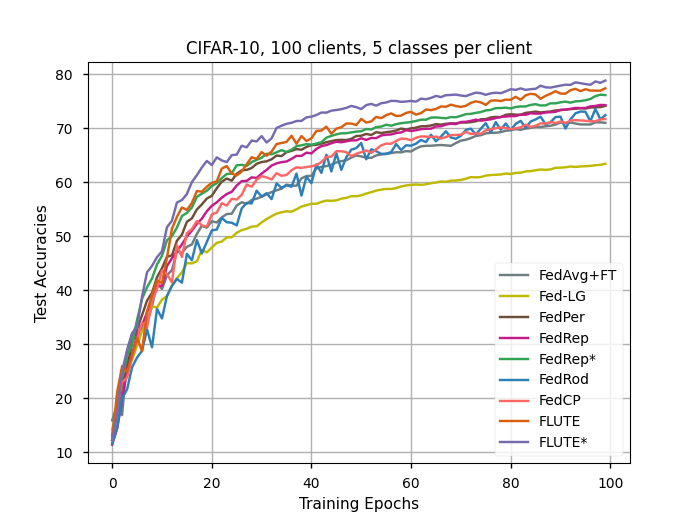}
 \end{subfigure}
\hfill
 \begin{subfigure}{0.32\linewidth}
     \includegraphics[width=\textwidth]{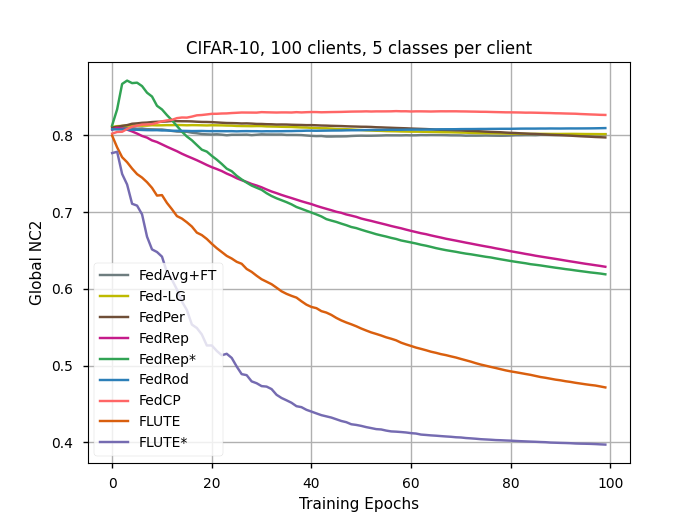}
 \end{subfigure}
 \hfill
  \begin{subfigure}{0.32\linewidth}
     \includegraphics[width=\textwidth]{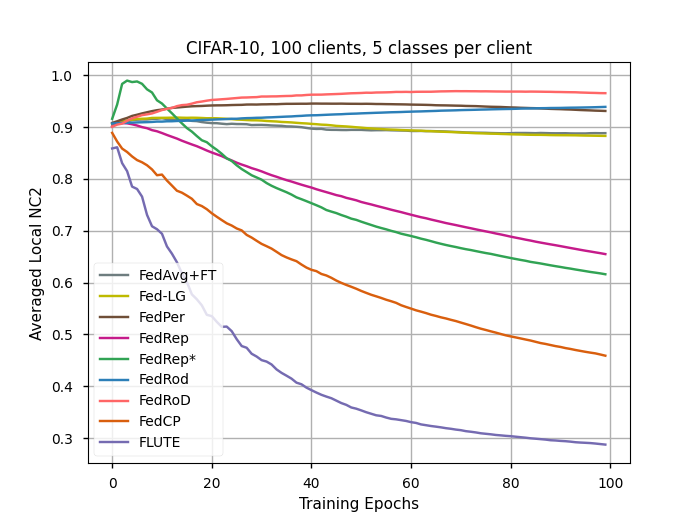}
 \end{subfigure}
 \caption{\small Experimental results for CIFAR10 when $M=50, m'=5$.}
\label{fig:real4}
\end{figure}
%-------------------------------
\begin{figure}[H]
 \begin{subfigure}{0.32\linewidth}
     \includegraphics[width=\textwidth]{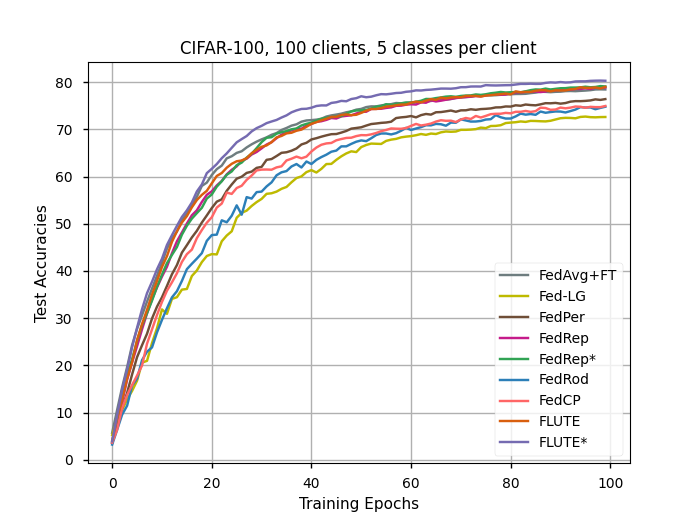}
 \end{subfigure}
\hfill
 \begin{subfigure}{0.32\linewidth}
     \includegraphics[width=\textwidth]{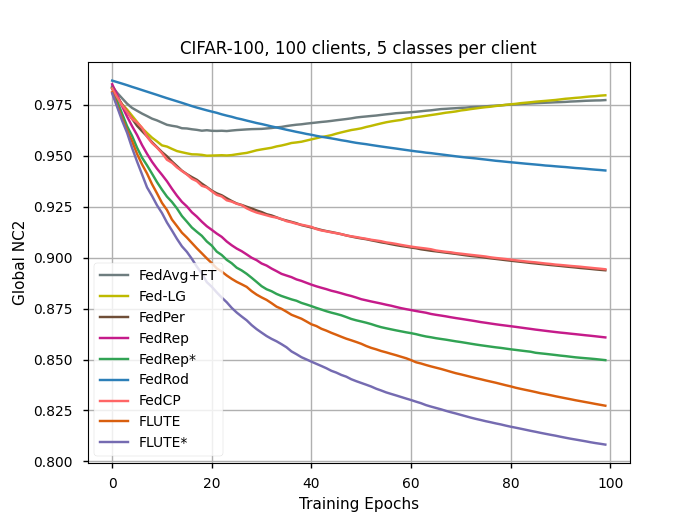}
 \end{subfigure}
 \hfill
 \begin{subfigure}{0.32\linewidth}
     \includegraphics[width=\textwidth]{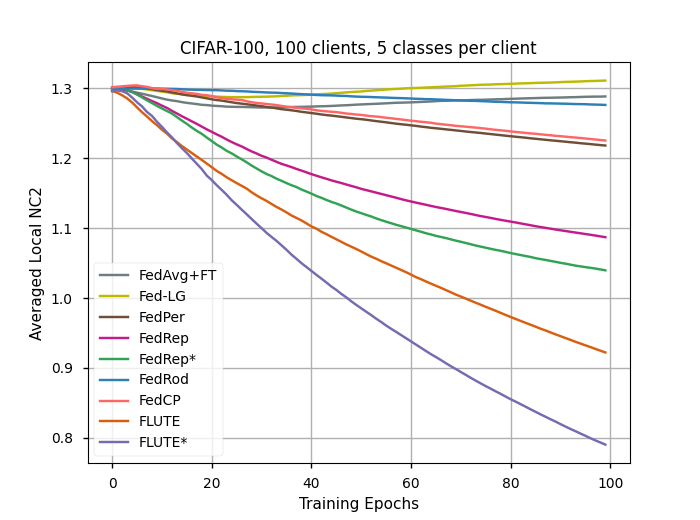}
 \end{subfigure}
 \caption{\small Experimental results for CIFAR100 when $M=100, m'=5$.}
\label{fig:real5}
\end{figure}
%-------------------------------
\begin{figure}[H]
 \begin{subfigure}{0.32\linewidth}
     \includegraphics[width=\textwidth]{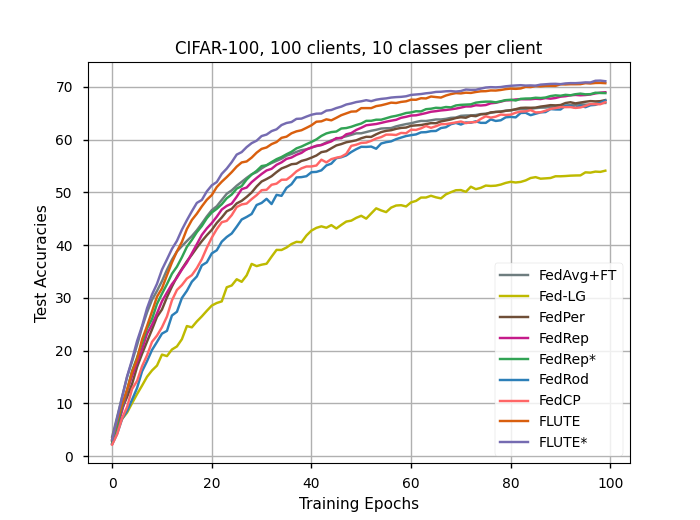}
 \end{subfigure}
\hfill
 \begin{subfigure}{0.32\linewidth}
     \includegraphics[width=\textwidth]{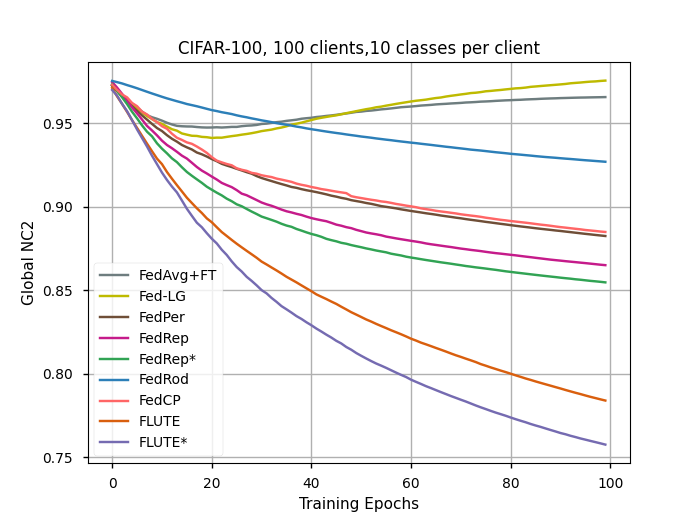}
 \end{subfigure}
 \hfill
 \begin{subfigure}{0.32\linewidth}
     \includegraphics[width=\textwidth]{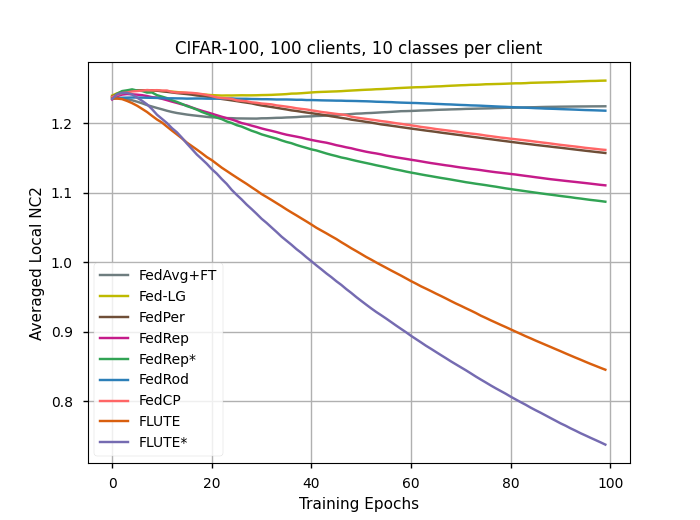}
 \end{subfigure}
 \caption{\small Experimental results for CIFAR100 when $M=100, m'=10$.}
\label{fig:real6}
\end{figure}
%-------------------------------
\begin{figure}[H]
 \begin{subfigure}{0.32\linewidth}
     \includegraphics[width=\textwidth]{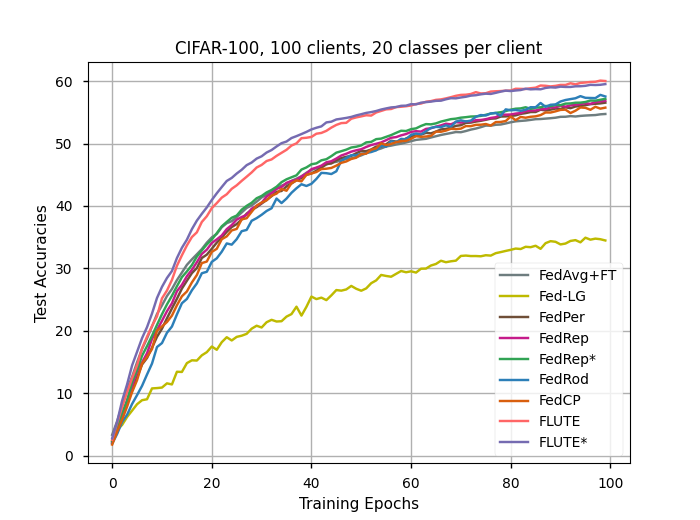}
 \end{subfigure}
\hfill
 \begin{subfigure}{0.32\linewidth}
     \includegraphics[width=\textwidth]{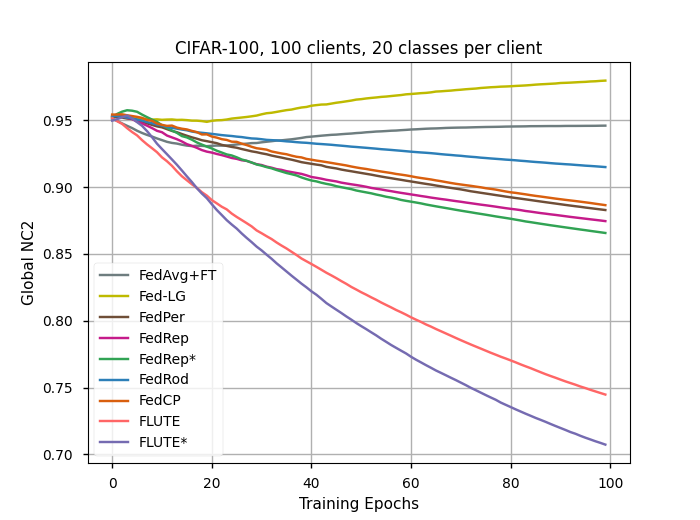}
 \end{subfigure}
 \hfill
 \begin{subfigure}{0.32\linewidth}
     \includegraphics[width=\textwidth]{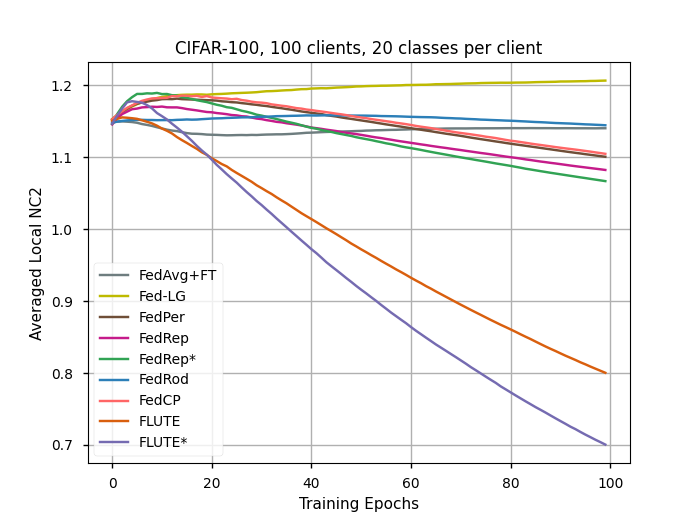}
 \end{subfigure}
 \caption{\small Experimental results for CIFAR100 when $M=100, m'=20$.}
\label{fig:real7}
\end{figure}
%-------------------------------
\begin{figure}[H]
 \begin{subfigure}{0.32\linewidth}
     \includegraphics[width=\textwidth]{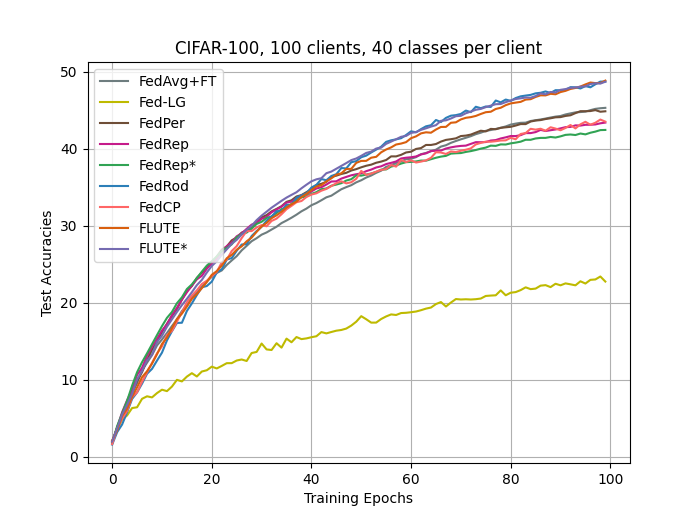}
 \end{subfigure}
\hfill
 \begin{subfigure}{0.32\linewidth}
     \includegraphics[width=\textwidth]{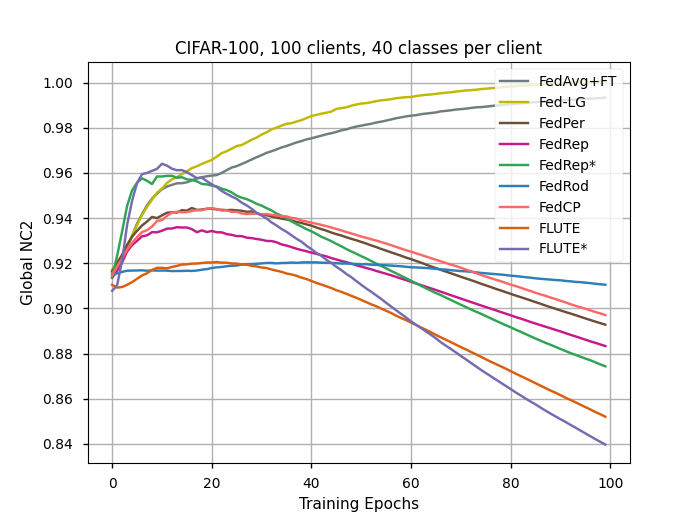}
 \end{subfigure}
 \hfill
 \begin{subfigure}{0.32\linewidth}
     \includegraphics[width=\textwidth]{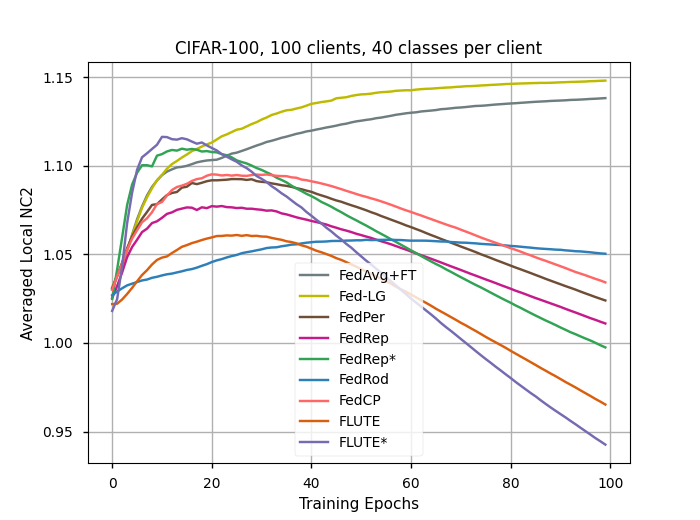}
 \end{subfigure}
 \caption{\small Experimental results for CIFAR100 when $M=100, m'=40$.}
\label{fig:real8}
\end{figure}

%\vfill

\end{document}